\documentclass{article}

\usepackage{microtype}
\usepackage{graphicx}
\usepackage{booktabs} 

\usepackage{hyperref}

 \usepackage[accepted]{icml2023}

\usepackage{amsmath}
\usepackage{amssymb}
\usepackage{mathtools}
\usepackage{amsthm}

\usepackage[capitalize,noabbrev]{cleveref}

\theoremstyle{plain}
\newtheorem{theorem}{Theorem}[section]
\newtheorem{proposition}[theorem]{Proposition}

\newtheorem{corollary}[theorem]{Corollary}
\theoremstyle{definition}

\theoremstyle{remark}
\newtheorem{remark}[theorem]{Remark}

\usepackage{caption}
\usepackage{subcaption}

\usepackage{lipsum}
\graphicspath{ {./images/} }
\usepackage{mathrsfs}
\usepackage{amsfonts,amsbsy}
\usepackage{algorithm}
\usepackage{comment}
\newtheorem{fact}{Fact}

\newcommand{\abs}[1]{\left\lvert #1 \right\rvert}
\newcommand{\norm}[1]{\left\lVert #1 \right\rVert}
\newcommand{\brk}[1]{\left[ #1 \right]}
\newcommand{\cbrk}[1]{\left\{ #1 \right\}}
\newcommand{\prt}[1]{\left( #1 \right)}

\newcommand{\ceil}[1]{\left\lceil #1 \right\rceil}

\newcommand{\cO}{\mathcal{O}}

\newcommand{\sD}{\mathscr{D}}
\newcommand{\bP}{\mathbb{P}}
\newcommand{\bE}{\mathbb{E}}
\newcommand{\bR}{\mathbb{R}}

\newcommand{\rT}{\boldsymbol{\mathrm{T}}}
\newcommand{\bT}{\boldsymbol{\mathrm{T}}}
\newcommand{\rP}{\mathrm{P}}
\newcommand{\rO}{\mathrm{O}}
\newcommand{\rN}{\mathrm{N}}

\newcommand{\re}[1]{\text{Regret}}
\usepackage{pifont}
\newcommand{\cmark}{\ding{51}}%
\newcommand{\xmark}{\ding{55}}
\usepackage[textsize=tiny]{todonotes}

\icmltitlerunning{A Distribution Optimization Framework for Confidence Bounds of Risk Measures}

\begin{document}

\twocolumn[
\icmltitle{A Distribution Optimization Framework for Confidence Bounds of Risk Measures}



\icmlsetsymbol{equal}{*}

\begin{icmlauthorlist}
\icmlauthor{Hao Liang}{cuhksz,sribd}
\icmlauthor{Zhi-quan Luo}{cuhksz,sribd}

\end{icmlauthorlist}

\icmlaffiliation{cuhksz}{School of Science and Engineering, The Chinese University of Hong Kong, Shenzhen}
\icmlaffiliation{sribd}{Shenzhen Research Institute of Big Data}
\icmlcorrespondingauthor{Hao Liang}{haoliang1@link.cuhk.edu.cn}

\icmlkeywords{Machine Learning, ICML}

\vskip 0.3in
]



\printAffiliationsAndNotice{}  

\begin{abstract}
We present a distribution optimization framework that significantly improves confidence bounds for various risk measures compared to previous methods. Our framework encompasses popular risk measures such as the entropic risk measure, conditional value at risk (CVaR), spectral risk measure, distortion risk measure, equivalent certainty, and rank-dependent expected utility, which are well established in risk-sensitive decision-making literature. To achieve this, we introduce two estimation schemes based on concentration bounds derived from the empirical distribution, specifically using either the Wasserstein distance or the supremum distance. Unlike traditional approaches that add or subtract a confidence radius from the empirical risk measures, our proposed schemes evaluate a specific transformation of the empirical distribution based on the distance. Consequently, our confidence bounds consistently yield tighter results compared to previous methods. We further verify the efficacy of the proposed framework by providing tighter problem-dependent regret bound for the CVaR bandit.

\end{abstract}

\section{Introduction}

The conventional machine learning literature primarily relies on the expected value or mean of a random variable as the performance metric for a given algorithm. However, in certain critical applications such as finance or medical treatment, the decision-maker's focus extends beyond the expected value and emphasizes other characteristics of the distribution. For instance, a risk-averse portfolio manager may place greater importance on tail behavior than expected value. To capture this risk-aware perspective, the decision-maker selects a risk measure (RM) as an alternative to the expected value, effectively representing their specific attitude towards risk. 

In practice, however, it is often infeasible to directly evaluate the risk measure of the unknown underlying distribution. Instead, we must rely on constructing the point estimator based on finite samples. Consequently, the confidence interval that quantifies the coverage of the true risk measure becomes crucial in the risk-sensitive setting, as it certifies a trustworthy range for the decision maker. 

In this paper, we aim to derive confidence bounds for several classes of risk measures: the Conditional Value at Risk (CVaR), the spectral risk measure (SRM), the distortion risk measure (DRM), the entropic risk measure (ERM), the certainty equivalent (CE), and the rank-dependent expected utility (RDEU). In safety-critical applications, such as medical treatment, CVaR is widely used, which represents the expected value within a fraction of the worst outcomes. Despite its practical utility, CVaR exhibits limitations in terms of expressing various risk preferences, as it assigns equal weight to all losses beyond a certain threshold. To address this, the SRM offers a notable generalization by incorporating a non-constant weighting function, enhancing flexibility in risk assessment. DRM came from insurance problems and later applied to investment risks. It encompasses CVaR as a special case and has gained attention in various fields. ERM is a well-known risk measure in mathematical finance and Markovian decision processes. Furthermore, CE serves as a generalization of ERM by replacing the exponential utility function with a more flexible function. This adaptation enhances the model's capability to capture a broader range of risk preferences.  RDEU contributes to understanding decision-making under uncertainty and has been widely applied in diverse domains such as finance, psychology, and health economics \footnote{For more descriptions about these risk measures, please refer to Section \ref{sec:pre} and Appendix \ref{app:rm}.}. 

In the existing literature, the confidence interval is commonly obtained through the concentration inequality, which bounds the deviation between the point estimator and the true risk with high probability. This deviation, referred to as the confidence radius, depends on the sample size and confidence level. Conventionally, the upper or lower confidence bound is determined by adding or subtracting the confidence radius from the point estimator. In this paper, we present two innovative approaches that construct confidence bounds for risk measures without relying on concentration inequalities. Our main contribution is summarized as follows.

\textbf{(1)} We propose a unified framework to obtain refined confidence bounds for several classes of risk measures, specifically for bounded distributions. We recast the problem of  determining  the confidence bound for risk measures based on finite samples as a constrained optimization problem. In particular, we optimize the value of risk measure over a confidence ball of distributions centered around the empirical distribution function (EDF). Furthermore, we obtain the closed-form solution that can be viewed as a transformation of the EDF. We set  the confidence bound as the optimal solution's risk measure value. Notably, the computational overhead increases only marginally.  

\textbf{(2)} We introduce a new baseline approach that leverages  the \emph{local} Lipschitz constant of a risk measure over the confidence ball, which may be of independent interest. In contrast, the previous bounds rely on the \emph{global} Lipschitz constant over the entire space of bounded distributions. In addition, we suggest a systematic way to compute the local Lipschitz constant and show that our bounds  outperform  the new baseline approach in certain scenarios.

\textbf{(3)} As a minor contribution, we propose a meta-algorithm that handles generic risk measures. Specifically, the meta-algorithm specializes in the \texttt{CVaR-UCB} algorithm \cite{tamkin2019distributionally} for CVaR bandit problems. Interestingly, \citet{tamkin2019distributionally} empirically observes that \texttt{CVaR-UCB} outperforms the global Lipschitz constant-based algorithm \texttt{U-UCB} \cite{cassel2018general}  with an order of magnitude improvement. Still, they only provide a regret bound that matches that of \texttt{U-UCB}. We fill this gap by providing an improved  regret upper bound, quantifying the magnitude of improvement.

\subsection{Related Work}
\paragraph{Confidence bounds of risk measures} The concentration of CVaR has been extensively explored in the literature, cf. \citet{brown2007large,wang2010deviation,thomas2019concentration,kolla2019concentration,prashanth2020concentration,la2022wasserstein}. The first three references primarily focus on the bounded distributions, while the remaining references consider unbounded distributions, including the sub-Gaussian, sub-exponential and  heavy tail distributions. \citet{pandey2019estimation,la2022wasserstein}  provide tail bounds for bounded, sub-Gaussian, or sub-exponential distributions. The concentration bounds for DRM, CE, and RDEU are presented in \citet{la2022wasserstein}.
\paragraph{Lipschitz constant-based methods} \citet{kock2021functional,la2022wasserstein} relate the estimation error to the Wasserstein distance between the true and empirical distributions and then use concentration bounds for the latter. \citet{la2022wasserstein} establishes concentration bounds for empirical estimates for a broad class of risk measures, including CVaR, SRM, DRM, RDEU, etc. They derive the concentration bounds via the global Lipschitz constant of the risk measure over the Wasserstein distance for bounded, sub-Gaussian, and sub-exponential distributions. Our bounds only apply to bounded distributions, but we demonstrate that our bounds are tighter than their results whenever they are valid. The computation of the global Lipschitz constant can be challenging, particularly for highly nonlinear risk measures.  In many cases, one may only obtain its upper bound as a surrogate, which further loosens the resulting bounds. In contrast, our framework does not require knowledge of the Lipschitz constant. \citet{kock2021functional} obtain the concentration bounds for general functionals using the supremum distance instead of the Wasserstein distance. While their work primarily focuses on inequality, poverty, and welfare measures, their methodology can be extended to encompass the risk measures mentioned above.  The resulting bounds apply to bounded distributions and are looser than ours. In addition, \citet{liang2023regret}  focuses on risk-sensitive reinforcement learning with dynamic risk measures and leverages the Lipschitz property of risk measures to derive regret upper bounds.  By quantifying the Lipschitz constants, \citet{liang2023regret} provide regret bounds that depend on these constants.  
\vspace{-1ex}
\paragraph{Off-policy risk evaluation} \citet{chandak2021universal,huang2021off} study the off-policy evaluation of functionals of reward or return distribution in bandit or RL setting. \citet{chandak2021universal} formulates the problem of interval estimation for various functionals as a constrained optimization problem over a confidence band, which bears similarity to Formulation \ref{eqt:opt_inf} in our paper. Meanwhile, our work differs from \citet{chandak2021universal} in two aspects. First, \citet{chandak2021universal} focuses on various functionals and derives the optimal solution for different functionals by a \emph{case-by-case geometric} analysis. In particular, their method applies to the mean, variance, quantiles, inter-quantile range, CVaR, and entropy. In contrast, our framework focuses on general risk measures, including but not limited to ERM, CVaR, SRM, DRM, CE, and RDEU. We leverage the intrinsic property of risk measures, namely monotonicity, to derive closed-form optimal solutions that are common across different risk measures. In particular, our derivation  for confidence bounds of CVaR differs from that in \citet{chandak2021universal}. Notably, our work is complementary to  \citet{chandak2021universal} in terms of the applicability of functionals. Our framework can handle arbitrary risk measures using a common optimal solution, while \citet{chandak2021universal} provides confidence bounds for CVaR and other functionals that are not risk measures, where the optimal solution depends on the specific functional. \citet{huang2021off} deal with the off-policy evaluation of Lipschitz risk measures based on their global Lipschitz constant with respect to the supremum distance. 

The rest of the paper is organized as follows. We introduce some basic concepts and notations in section \ref{sec:pre}. We present our new framework under the Wasserstein distance and the supremum distance in section \ref{sec:frame}, and suggest the closed-form solution in Section \ref{sec:sol}. We then provide a new baseline method, which bridges our framework and the previous global Lipschitz constant-based method in Section \ref{sec:imp}. We validate the proposed framework by applying it to the risk-sensitive bandit problems in Section \ref{sec:app}, and provide numerical experiments in Section \ref{sec:exp}. Finally, we provide the concluding remarks in Section \ref{sec:con}.
\section{Preliminaries}
\label{sec:pre}
We introduce some notations here. Let $a<b$ be two real numbers. We denote by $\mathscr{D}([a,b])$ and $\mathscr{D}$ the space of all cumulative distribution functions (CDFs) supported on $[a,b]$ and the space of all CDFs on reals respectively. For a CDF $F\in\sD$, let $X_1, X_2, \cdots,  X_n$ be $n$ i.i.d. samples from $F$. We denote by $F_n$ the empirical distribution function corresponding to these samples:
\[  F_n(\cdot)\triangleq \frac{1}{n}\sum_{i=1}^n \mathbb{I}\{X_i\leq \cdot\} = \frac{1}{n}\sum_{i=1}^n \delta_{X_i}, \]
where $\mathbb{I}$ is the indicator function and $\delta$ is the Dirac measure. We denote by $F^{-1}:(0,1]\mapsto \bR$ is the inverse distribution function (IDF) of $F$, i.e., the quantile function  $F^{-1}(y)\triangleq \inf\{x\in\bR:F(x)\ge y\}$.
\paragraph{Supremum distance}
For two CDFs  $F,G \in \mathscr{D}$, the supremum distance between them is defined as
\[ \norm{F-G}_{\infty}\triangleq \sup_{x\in\bR}\abs{F(x)-G(x)}. \]
The DKW inequality \cite{dvoretzky1956asymptotic,massart1990tight} bounds the deviation of the empirical distribution from the true distribution in terms of the supremum distance with high probability.
\begin{fact}[Two-sided DKW inequality]
\label{fct:dkw}
Let $\delta\in(0,1]$, then the following holds with probability at least $1-\delta$
\begin{equation}
    \label{eqt:con_inf}
    \norm{F-F_n}_{\infty} \leq c_n^{\infty}\triangleq\sqrt{\frac{\log(2/\delta)}{2n}},
\end{equation}
where $c_n^{\infty}$ is the concentration radius.
\end{fact}
The DKW inequality holds for any distribution, including discrete and unbounded distributions.

\paragraph{Wasserstein distance}
For CDFs $F,G \in \mathscr{D}$, the Wasserstein distance between them is defined as
 \begin{align*}
     W_1(F,G)&\triangleq\int_{-\infty}^{\infty}\abs{F(x)-G(x)}dx.
 \end{align*}
$W_1(F,G)$ can be expressed as the $\ell_1$ norm between $F$ and $G$. Therefore we also write $W_1(F,G)=\norm{F-G}_1$. \citet{fournier2015rate} establishes the concentration bounds on the Wasserstein distance between the EDF and the underlying one without explicit constants. \citet{la2022wasserstein} gives the concentration results for sub-Gaussian distributions with explicit constants. As a  corollary, Fact \ref{fct:was_con} provides the concentration bound for bounded distributions.
\begin{fact}
\label{fct:was_con}
Let $F\in\sD([a,b])$. With probability at least $1-\delta$, for every $n\ge\log(1/\delta)$ 
\begin{equation}
\label{eqt:con_was}
\begin{aligned}
        &\norm{F-F_n}_{1} \leq c^1_n \triangleq \frac{256(b-a)}{\sqrt{n}}+8(b-a)\sqrt{\frac{e\log(1/\delta)}{n}}
\end{aligned}
\end{equation}
where $c_n^1$ is the \emph{concentration radius}.
\end{fact}

\paragraph{Risk measure}
\label{subsec:rm}
In this paper, we interpret the random variable as a loss instead of a reward. For two random variables $X\sim F$ and $Y\sim G$, we say that $Y$ dominates $X$  if $\forall x\in\mathbb{R}, F(x)\ge G(x)$, and we write  $Y\succeq X$.  A risk measure $\boldsymbol{\mathrm{T}}$ is defined as a functional mapping from a set of r.v.s $\mathscr{X}$ to the reals that satisfy  the following conditions \cite{follmer2010convex,weber2006distribution}
\begin{itemize}
    \item  Monotonicity:  $X \preceq Y \Rightarrow \bT(X) \leq \bT(Y)$
    \item  Translation-invariance: $\boldsymbol{\mathrm{T}}(X+c)=\boldsymbol{\mathrm{T}}(X)+c, c\in\mathbb{R}$
\end{itemize}
A risk measure $\mathrm{T}$ is said to be \emph{distribution-invariant} if $\boldsymbol{\mathrm{T}}(X)=\boldsymbol{\mathrm{T}}(Y)$ when $X$ and $Y$ follow the same distribution \cite{acerbi2002spectral,weber2006distribution}. In this paper, we only consider distribution-invariant risk measures. We write $\bT(F)=\bT(X)$ for simplicity.  
We remark that there are other functionals mapping a r.v. to a real number, e.g., the inequality measures \cite{kock2021functional} that do not satisfy the monotonicity. In this paper, we derive the confidence bounds for several classes of risk measures. It turns out that the monotonicity of risk measures plays an essential role in our optimization framework.

Table \ref{tab:rm} summarizes the relevant risk measures considered in this paper. These risk measures are grouped into classes, namely SRM, DRM, CE, and RDEU. CVaR and ERM  belong to the SRM and CE classes, respectively, based on specific choices of the weighting function $\phi$ and the utility function $u$.  The specific conditions related to the definitions of these risk measures are listed below. Please refer to Appendix \ref{app:rm} for detailed descriptions.
\begin{itemize}
	\item SRM: $\phi:[0,1]\rightarrow[0,\infty)$ is increasing and satisfying $\int_0^1\phi(y)dy=1$.
	\item DRM: $g:[0,1]\rightarrow[0,1]$ is a continuous, concave and increasing function with $g(0)=0$ and $g(1)=1$.
	\item CE: $u$ is a continuous, convex, and strictly increasing function.
	\item RDEU: $w:[0,1]\rightarrow[0,1]$ is an increasing weight function with $w(0)=0$ and $w(1)=1$; $v:\mathbb{R}\rightarrow\mathbb{R}$ is an (unbounded) increasing differentiable function with $u(0)=0$.
    \item CVaR: an instance of SRM with $\phi(y)=\frac{1}{\alpha}\mathbb{I}\{y\ge1-\alpha\}$
    \item ERM: an instance of CE with $u(x)=\exp(\beta x)$.
\end{itemize}
\begin{table}[t]
	\caption{List of risk measures}
	\label{tab:rm}
	\centering
	\begin{tabular}{ lccr  }
		\toprule
		RM &Notation &Definition   \\ 
		\midrule
		SRM & $\boldsymbol{M_{\phi}}(F)$   &$\int_0^1 \phi(y)F^{-1}(y)dy$ \\ [0.5ex] 
		DRM & $\boldsymbol{\rho_g}(F)$  &   $\int_0^{\infty}g(1-F(x))dx$ \\ [0.5ex] 
        CE & $\boldsymbol{E_u}(F)$ & $ u^{-1}\prt{\int_{\bR}u(x)dF(x) } $ \\ [0.5ex] 
        RDEU & $\boldsymbol{V}(F)$ & $ \int_{a}^{b}v(x)dw(F(x)) $ \\ [0.5ex]
        CVaR & $\boldsymbol{C_{\alpha}}(F)$ & $\inf_{\nu\in\bR}\cbrk{ \nu + \frac{1}{1-\alpha} \bE_{X\sim F}[(X-\nu)^+]}$ \\ [0.5ex] 
        ERM & $\boldsymbol{U_{\beta}}(F)$ & $ \frac{1}{\beta}\log\prt{\int_{\bR}\exp(\beta x)dF(x) } $  \\
		\bottomrule
	\end{tabular}
\end{table}
It is more convenient to represent  some risk measures using IDF, e.g., SRM $M_{\phi}(F)=\int_0^1 \phi(y)F^{-1}(y)dy$. For this reason, we overload notation and write $\bT(F^{-1})=\bT(F)$ whenever convenient for some $\bT$.

\section{Distribution Optimization Framework}
\label{sec:frame}
\subsection{Global Lipschitz Constant-based Approach}
Fact \ref{fct:was_con} and Fact \ref{fct:dkw} present the concentration bound of the empirical distribution in terms of the Wasserstein distance and the supremum distance, respectively. They can be written in a unified way: with probability at least $1-\delta$, we have
\begin{equation}
\label{eqt:con}
   \norm{F-F_n}_p \leq c_n^p, 
\end{equation}
where $p=1$ indicates the Wasserstein distance and $p=\infty$ indicates the supremum distance. To relate the concentration bound of EDF to that of risk measure,
\citet{kock2021functional,bhat2019concentration} use the Lipschitz property of the risk measure, i.e., for any two CDFs  $F,G \in \sD([a,b])$, there exists $L_p(\bT)>0$  such that the risk measure $\bT$ satisfies
 \begin{equation}
 \label{eqt:lip}
     |\rT(F)-\rT(G)| \leq L_p(\rT) \norm{F-G}_p.
 \end{equation}
$L_p(\rT)$ is called the \emph{global Lipschitz constant} (GLC) of $\rT$ w.r.t. $\norm{\cdot}_p$ since the inequality holds for all possible pairs of CDFs. Combining Equation \ref{eqt:con} and Equation \ref{eqt:lip}, \citet{kock2021functional,bhat2019concentration} establish the concentration bounds of a class of Lipschitz functionals 
\[  \rT(F_n)-L_p(\rT)c_n^p\leq\rT(F)\leq \rT(F_n)+L_p(\rT)c_n^p.\]
The quality of the above bounds relies on the tightness of $L_p(\rT)$, so the finest bounds one can get fall back on identifying the tightest GLC
\[  L_p(\rT) \triangleq \sup_{G,G^{\prime}\in  \sD([a,b])} \frac{\rT(G)-\rT(G^{\prime})}{\norm{G-G'}_p},   \]
where we overload the notation of $L_p(\rT)$. The GLC-based approach suffers from several limitations. The GLC may not be easy to compute, especially for some highly nonlinear risk measures. In most cases, one may only obtain its upper bound as a surrogate. Meanwhile, the concentration bounds are far from optimal. The confidence bounds are set to be the product of the GLC and the confidence radius. However, the GLC is loose since it is evaluated over the whole space of bounded distributions. 
\subsection{Local Lipschitz Constant-based Approach}
Before introducing our framework as a remedy, we propose a \emph{new baseline} approach that improves the previous bounds. Observe that Equation \ref{eqt:con} together with the boundedness of $F$ can be written as the norm ball constraint 
\[ B_p(F_n,c_n^p)\triangleq \left\{F\mid \norm{F-F_n}_p \leq c_n^p,F\in\sD([a,b])\right\}.  \]
Define the \emph{local} Lipschitz constant (LLC) over $B_p(F_n,c_n^p)$
\begin{align*}
    L_p(\rT;F_n,c_n^p)&\triangleq \sup_{G,G^{\prime}\in B_p(F_n,c_n^p)} \frac{\rT(G)-\rT(G^{\prime})}{\norm{G-G'}_p}\\
    &\leq \sup_{G,G^{\prime}\in  \sD([a,b])} \frac{\rT(G)-\rT(G^{\prime})}{\norm{G-G'}_p}= L_p(\rT).
\end{align*}
For simplicity, we drop $\rT$ from the Lipschitz constants.
We thus obtain the tighter upper/lower confidence bound (UCB/LCB)
\[ \rT(F_n)+(-)L_p(F_n,c_n^p)c_n^p \leq(\ge) \rT(F_n)+L_p c_n^p.\]
As  sample size increases, the confidence radius $c_n^p$ shrinks, leading to smaller LLC and sharper bounds. In contrast, the previous bounds do not adapt to the sample size.  
\subsection{Distribution Optimization Framework}
We now propose our unified framework to derive confidence bounds for a broad range of risk measures. The idea is quite simple and intuitive. Given a risk measure, we maximize/minimize the risk measure value over the confidence ball and set the  maximal/minimal value as the UCB/LCB. By recasting the problem of finding the confidence bounds to a constrained optimization problem, we obtain the \emph{optimal} bounds from Equation \ref{eqt:con}. Different choices of distances lead to two types of frameworks:
\begin{equation}
\label{eqt:opt_was}
\begin{array}{rrclcl}
\displaystyle \max_{G\in\sD([a,b])} & \multicolumn{3}{l}{\rT(G)}\\
\textrm{s.t.} & \norm{G-F_n}_1\leq c_n^1
\end{array}
\end{equation}
and
\begin{equation}
\label{eqt:opt_inf}
\begin{array}{rrclcl}
\displaystyle \max_{G\in\sD([a,b])} & \multicolumn{3}{l}{\rT(G)}\\
\textrm{s.t.} & \norm{G-F_n}_{\infty}\leq c_n^{\infty}
\end{array}
\end{equation}
We can obtain the LCB by reverting the maximization formulation to a minimization formulation. Denote by $\overline{F^p_n}$ ($\underline{F^p_n}$) the optimal solution, then the UCB and LCB are set to be $\rT\prt{\overline{F^p_n}}$ and $\rT\prt{\underline{F^p_n}}$. In the sequel, we may drop $p$
when the statement holds for either $p=1$ or $p=\infty$.

To demonstrate the optimality of our framework, observe that
$\overline{F_n} \in B(F_n,c_n)$, therefore
\begin{equation}
\label{eqt:comp}
    \rT\prt{\overline{F}_n}\leq \rT(F_n)+L(F_n,c_n)c_n \leq  \rT(F_n)+L c_n.
\end{equation}
$ \rT\prt{\overline{F_n}}$ is tighter than the bound derived from the tightest LLC, our new baseline approach, which already improves the previous bounds.

One may wonder whether $\overline{F_n}$ and $\underline{F_n}$ are easy to obtain. Fortunately, we will show that they admit analytic form for almost all risk measures introduced in Section \ref{subsec:rm} in the next section. Moreover, we will use Equation \ref{eqt:comp} to quantify the tightness of our confidence bounds  in Section \ref{sec:imp}. For ease of notation, we will omit  $\sD([a,b])$.

\section{Closed-form Solution}
\label{sec:sol}
The following theorems present the closed-form solutions to Formulation \ref{eqt:opt_was}-\ref{eqt:opt_inf}. 
The proofs are deferred to Appendix \ref{app:thm}.
\begin{theorem}
\label{thm:opt_inf}
For any  risk measure satisfying the monotonicity, the optimal solution to Formulation \ref{eqt:opt_inf} is given by
\begin{equation}
\label{eqt:sol_inf}
\begin{aligned}
     \overline{F^{\infty}_n} = \boldsymbol{\rP_{c_n^{\infty}}^{\infty}} F_n, \
    \underline{F^{\infty}_n} = \boldsymbol{\rN_{c_n^{\infty}}^{\infty}} F_n,
\end{aligned}
\end{equation}
where $\boldsymbol{\mathrm{P}_c^{\infty}}/\boldsymbol{\mathrm{N}_c^{\infty}}:\sD([a,b])\rightarrow\sD([a,b])$  is the positive/negative operator with coefficient $c>0$ for the supremum distance, which is defined as follows
\begin{align*}
    \prt{\boldsymbol{\mathrm{P}_{c}^{\infty}} F}(x) &\triangleq \max\cbrk{F(x) - c\mathbb{I}\{x< b\},0 },\\
    \prt{\boldsymbol{\mathrm{N}_{c}^{\infty}} F}(x) &\triangleq \min\cbrk{F(x) + c\mathbb{I}\{x\ge a\},1}.
\end{align*}
\end{theorem}
The supremum ball $B_{\infty}(F_n,c_n^{\infty})$ consists of the CDFs within the area sandwiched by $\boldsymbol{\mathrm{P}_{c_n^{\infty}}^{\infty}} F_n$ and $\boldsymbol{\mathrm{N}_{c_n^{\infty}}^{\infty}} F_n$ (see Figure 1). Since any risk measure $\rT$ is monotonic, and
\[ \boldsymbol{\mathrm{P}_{c_n^{\infty}}^{\infty}} F_n(x) \leq G(x) \leq \boldsymbol{\mathrm{N}_{c_n^{\infty}}^{\infty}} F_n, \forall x \in \bR, \forall G \in B_{\infty}(F_n,c_n^{\infty}) \]
then $\boldsymbol{\mathrm{P}_{c_n^{\infty}}^{\infty}} F_n$ and $\boldsymbol{\mathrm{N}_{c_n^{\infty}}^{\infty}} F_n$ are the maximizer and the minimizer respectively. Another interpretation is that $\boldsymbol{\mathrm{P}_{c_n^{\infty}}^{\infty}}$ transports the leftmost atoms of $F_n$  with total mass of $c_n^{\infty}$ to the maximally possible atom $b$, while $\boldsymbol{\mathrm{P}_{c_n^{\infty}}^{\infty}}$ transports the rightmost atoms of $F_n$ with total mass of $c_n^{\infty}$ to the minimally possible atom $a$. Although we can  explicitly represent the optimal solutions in the PMF form, it is more convenient to work with the CDF form.  Please refer to Appendix \ref{app:alg_op} for more details. 
\begin{figure}[t]
\label{fig:opt_inf}
\begin{center}
	\centerline{\includegraphics[width=0.49 \textwidth]{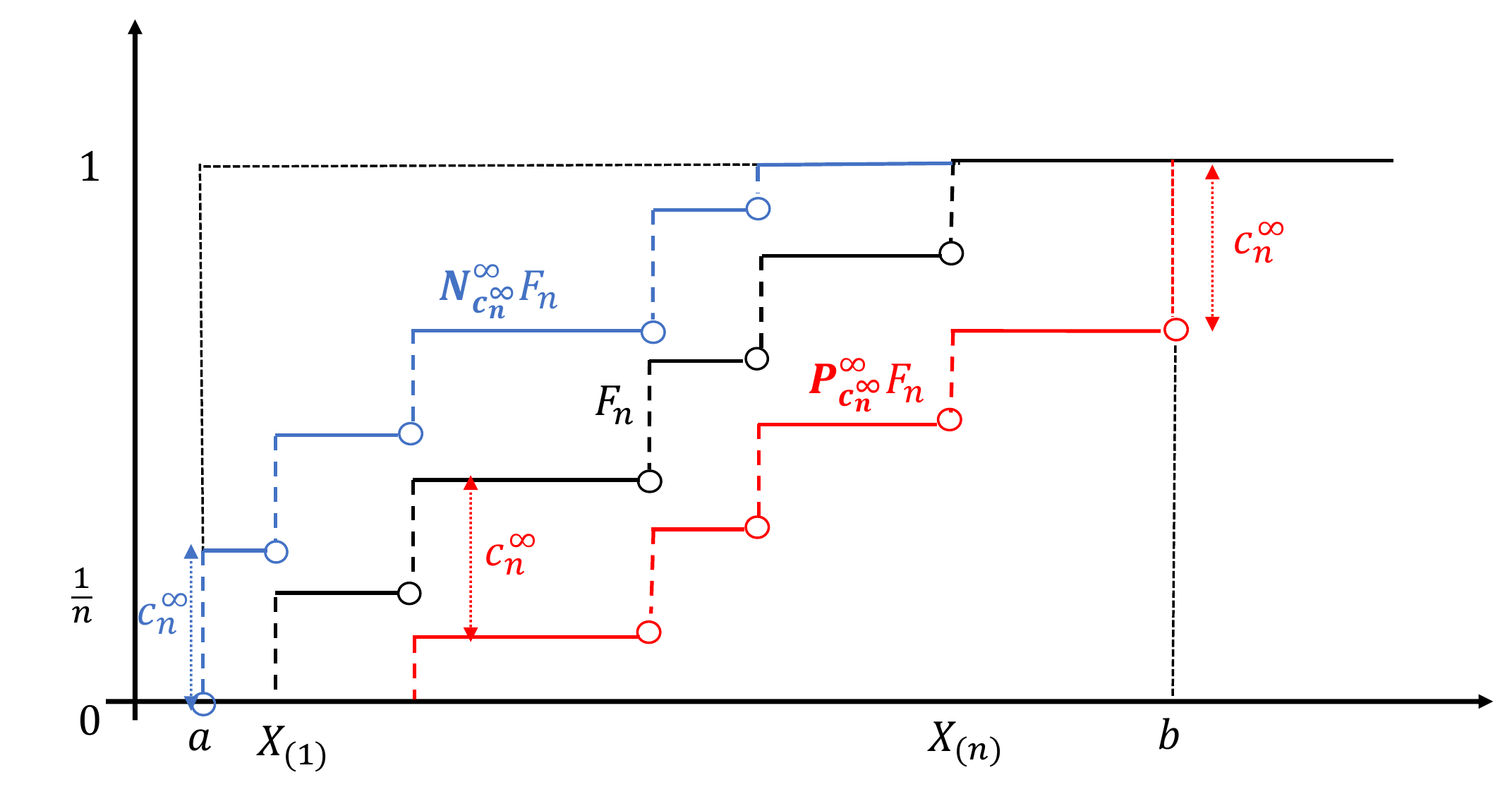}}
	\caption{$F_n$ (black), $\boldsymbol{\rP^{\infty}_{c^{\infty}_n}} F_n$ (blue) and $\boldsymbol{\rN^{\infty}_{c^{\infty}_n}} F_n$ (red).}
\end{center}
\vspace{-4ex}
\end{figure}
\begin{remark}
The positive operator in Equation \ref{eqt:sol_inf} reduces to the optimistic operator introduced in the CVaR bandit/RL \cite{tamkin2019distributionally,keramati2020being}. However, they only consider the case of CVaR, and we generalize it to  arbitrary risk measures.
\end{remark}
\begin{theorem}
\label{thm:opt_was}
For the risk measures except RDEU  in Section \ref{subsec:rm}, the optimal solution to  Formulation \ref{eqt:opt_was} is given by
\begin{equation}
\label{eqt:sol_was}
\begin{aligned}
    \overline{F^1_n} = \boldsymbol{\mathrm{O}_{c_n^{1}}^{1}}F_n, \  \underline{F^{1}_n} = \boldsymbol{\mathrm{P}_{c_n^{1}}^{1}}F_n,
\end{aligned}
\end{equation}
where $ \boldsymbol{\mathrm{P}_{c}^{1}}/\boldsymbol{\mathrm{N}_{c}^{1}}:\sD([a,b])\rightarrow\sD([a,b])$ is called the positive/negative operator for CDF with coefficient $c$ for the Wasserstein distance, which is defined as follows. 
\end{theorem}
\vspace{-0ex}
\begin{figure}[t]
\label{fig:1}
\begin{center}
	\centerline{\includegraphics[width=0.5 \textwidth]{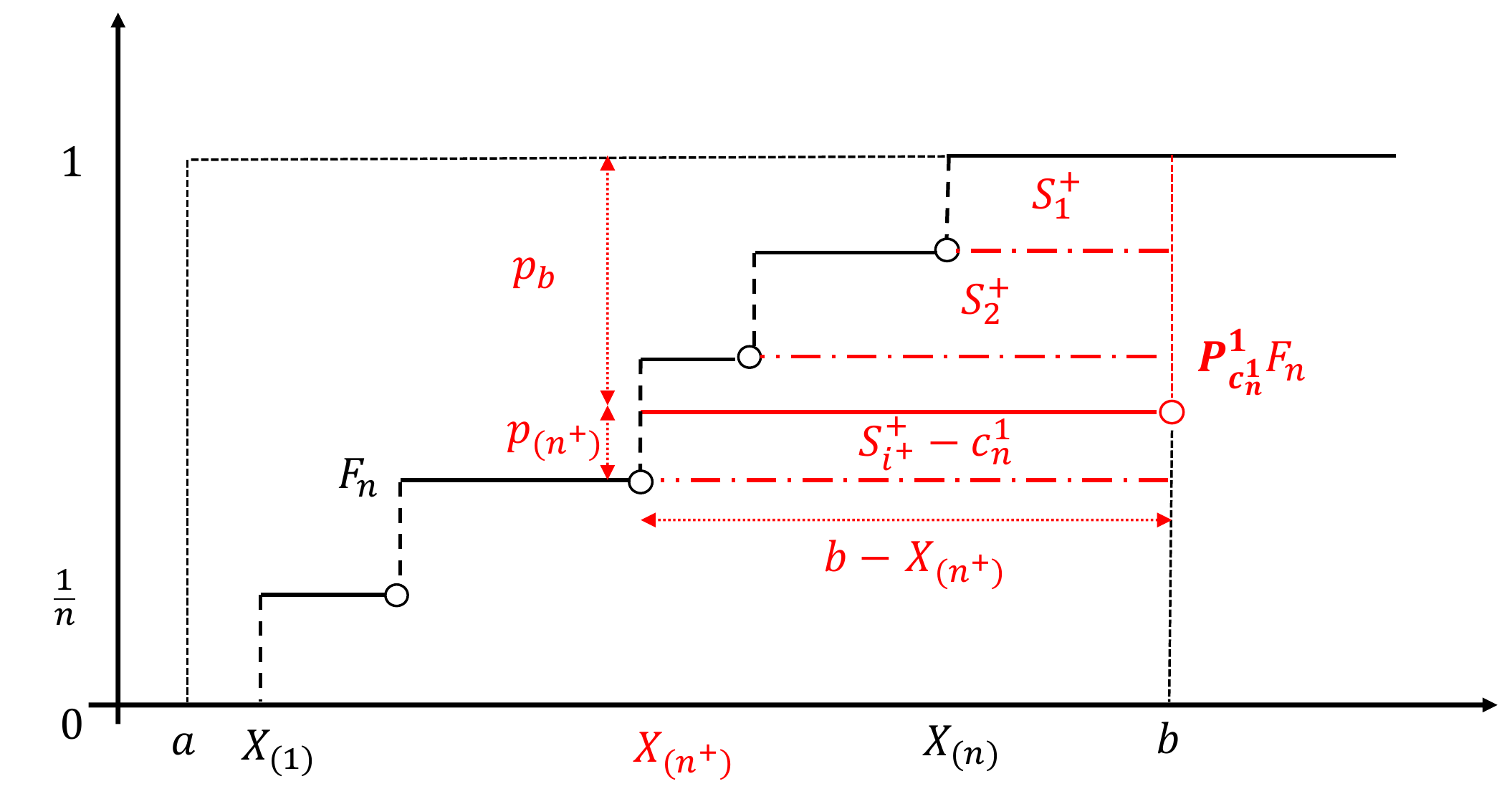}}
    \vspace{-1ex}
	\caption{ $F_n$ (black) and  $\boldsymbol{\rP^1_{c_n^1}} F_n$ (red). $\boldsymbol{\rP^1_{c_n^1}} F_n$ overlaps $F_n$ for $x< X_{(n^+)}$, and  it has only  two jumps at $X_{(n^+)}$ and $b$ for $x\ge X_{(n^+)}$.}
\end{center}
\vspace{-4ex}
\end{figure}
Fix $F_n$ and $c_n^1>0$. Let $X_{(1)}\leq X_{(2)}\cdots \leq X_{(n)}$ be the order statistic of $\{X_i\}$. For $i\in[n]$, we recursively define
\begin{align*}
    S^+_1 \triangleq \frac{1}{n}\prt{b-X_{(n)}}, \ S^+_i \triangleq S^+_{i-1}+\frac{1}{n}(b-X_{(n+1-i)}).
\end{align*}
The geometric interpretation of $S^+_i$ is  the area sandwiched between $F_n$ and the horizontal line $1-\frac{i}{n}$ (see Figure 2). Define $i^+\triangleq\min\{i: S^+_i \ge c_n^1 \}$ as the first index that $S^+_i$ exceeds $c_n^1$. Let $n^+\triangleq n+1-i^+$. 
Then $\boldsymbol{\mathrm{P}_{c_n^1}^1} F_n$ is a categorical distribution 
with atoms $\cbrk{X_{(i)}}_{i\in[n^{+}]}\cup\{b\}$. The probability mass of $X_{(n^+)}$ and $b$ are assigned to be
\[ p_{n^+}\triangleq\frac{1}{b-X_{n^+}}(S^+_{i^+}-c^1_n) , \ p_b \triangleq \frac{i^+}{n}-p_{n^+}, \]
meanwhile the probability mass of the first $n^+-1$ atoms $\cbrk{X_{(i)}}_{i\in[n^{+}-1]}$ remains $\frac{1}{n}$.
To be more precise, $\boldsymbol{\mathrm{P}_{c_n^1}^1} F_n$ is described by the following probability mass function (PMF)
\[  \frac{1}{n}\sum_{i=1}^{n^+-1}\delta_{X_{(i)}} + p_{n^+} \cdot\delta_{X_{(n^+)}} + p_b\cdot b.
\]
The way of transforming $F_n$ into $\boldsymbol{\mathrm{P}_{c_n^{1}}^{1}} F_n$ resembles the well-known \emph{water-filling algorithm} \cite{telatar1999capacity} in wireless commutations in the opposite direction. Imagine that the gravity is reversed to the upward direction, and we  fill the water of amount $c_n^1$ to a tank enclosed by $F_n$ and the vertical line $b$. The water is sequentially filled in the  bins from right to left, in which the $i$-th bin corresponds to the $X_{(n+1-i)}$ until the water is filled up at the $i^+$ bin. By a volume argument, the water level is $\frac{n^+-1}{n}+p_{n^+}$. We then recover the analytic form via the shape of $\boldsymbol{\mathrm{P}_{c_n^{1}}^{1}}F_n$.

Another interpretation is that $\boldsymbol{\mathrm{P}_{c_n^{1}}^{1}}$ replaces the  probability mass $\frac{1}{n}-p_{n^+}$ of $X_{(n^+)}$ and all the atoms to its right $\cbrk{X_{(i)}}_{n^+<i\leq n}$ by the upper bound $b$.
For convenience, we let $X_{(0)}=a$. We recursively define for $i\in[n]$
\begin{align*}
    S^-_1 \triangleq \frac{X_{(n)}-X_{(n-1)}}{n},
    S^-_i \triangleq S^-_{i-1}+\frac{i(X_{(n+1-i)}-X_{(n-i)})}{n}.
\end{align*}
Now $S^-_i$ represents the area sandwiched between $F_n$ and the vertical line $X_{(n-i)}$ (see Figure 3). Define $i^-\triangleq\min\{i: S^-_i \ge c_n^1 \}$ as the first index that $S^-_i$ exceeds $c_n^1$. Let $n^-\triangleq n+1-i^-$. 
Then $\boldsymbol{\rN_{c_n^1}^1} F_n$ is a categorical distribution 
with atoms $\cbrk{X_{(i)}}_{i\in[n^{-}-1]}\cup\{b^-\}$,
where $b^-$ is given by 
\[b^-\triangleq X_{(n^{-}-1)}+\frac{n}{i^-}(S^-_n-c^1_n).\]
$\boldsymbol{\rN_{c_n^1}^1} F_n$ is given by the PMF
\[  \frac{1}{n}\sum_{i=1}^{n^--1}\delta_{X_{(i)}} + \frac{i^-}{n} \cdot b^-.
\]
$\boldsymbol{\rN_{c_n^{1}}^{1}} F_n$  mirrors the water-filling, but we now fill the water rightward until it is filled out with water level $b^-$.  It can also be interpreted as replacing the atoms to the right of $X_{(n^-)}$ with a total mass of $\frac{i^-}{n}$ by a single atom $b^-<X_{(n^-)}$.  
\begin{figure}[t]
\label{fig:neg_opt_was}
\begin{center}
	\centerline{\includegraphics[width=0.5 \textwidth]{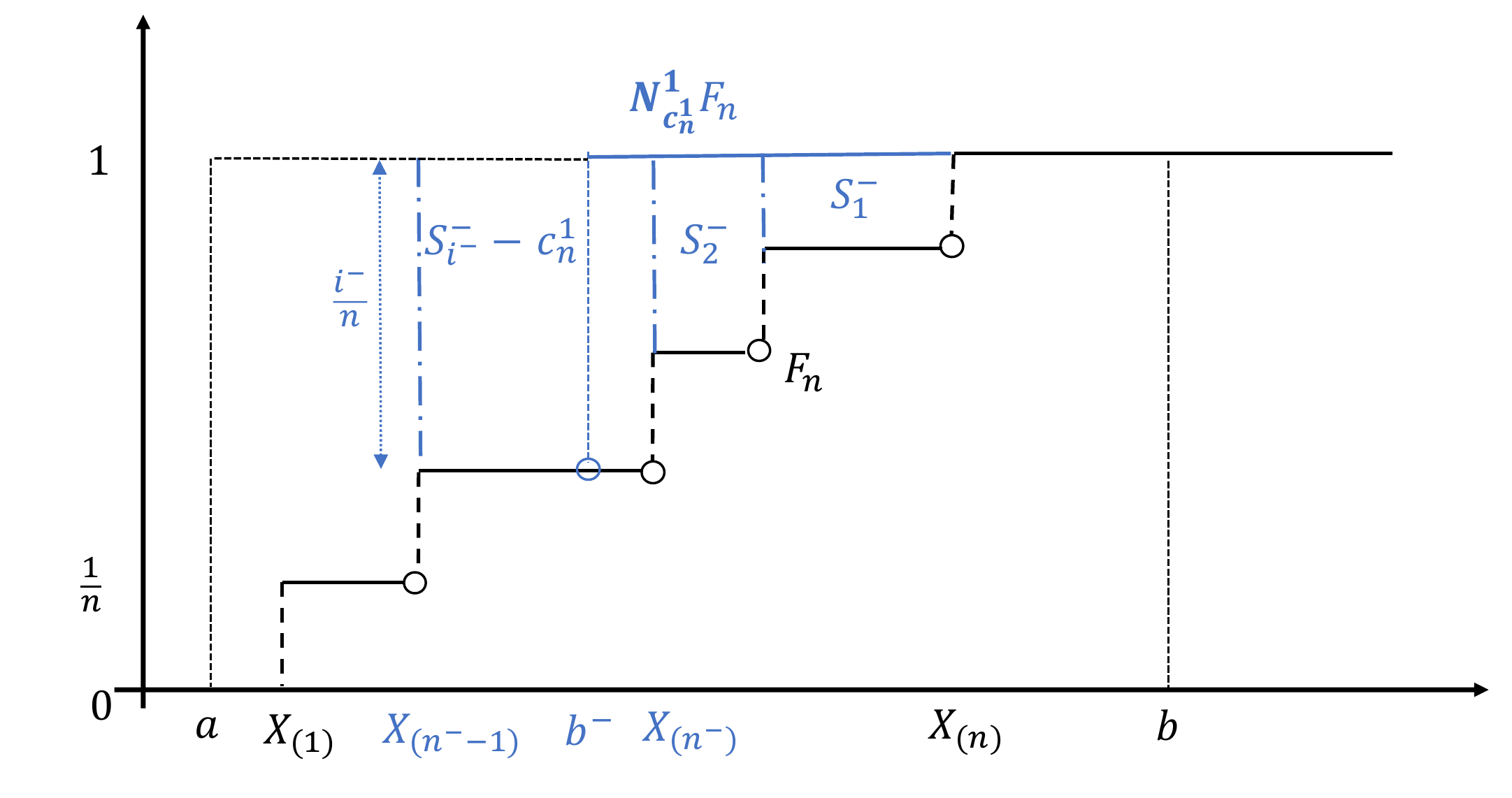}}
    \vspace{-1ex}
	\caption{ $F_n$ (black) and $\boldsymbol{\rN^1_{c^1_n}} F_n$ (blue). $\boldsymbol{\rN^1_{c_n^1}} F_n$ overlaps $F_n$ for $x< b^-$ and has a single jump at $b^-$ with height $\frac{i^-}{n}$.}
\end{center}
\vspace{-4ex}
\end{figure}

\paragraph{Computational issue.} We present the algorithms to actualize Equation \ref{eqt:opt_was}-\ref{eqt:opt_inf} in practice in Appendix \ref{app:alg_op}. We demonstrate that their computational complexity increases slightly more than the LC-based methods. 
\begin{remark}
    For both distances, we only require $F$ to be bounded above by a known constant $b$ to perform $\boldsymbol{\mathrm{P}_{c_n}}F_n$, and require $F$ to be bounded below by a known constant $a$ to perform $\boldsymbol{\rN_{c_n}} F_n$. Thus we only require $F$ to be bounded on one side to obtain the one-sided confidence bound. 
\end{remark}

\section{Improvement of Confidence Bounds}
\label{sec:imp}
\subsection{Derivation of the LLC}
\label{subsec:local_lip}
We present a systematic way of computing the LLC over the confidence ball $B(F_n,c_n)$, which bridges our framework and the GLC-based method. Define $\psi(t;F,G)\triangleq\mathrm{T}((1-t)F+t G)$ for $F, G\in \mathscr{D}([a,b])$ and $t\in[0,1]$. For simplicity, we may drop $F,G$ and write $\psi(t)$ if it is clear from the context.  Note that $\psi(0)=\mathrm{T}(F)$, $\psi(1)=\mathrm{T}(G)$ and $(1-t)F+t G\in \mathscr{D}([a,b])$ for all $t\in[0,1]$. It can be shown that $\boldsymbol{\psi}$ is continuously differentiable under some mild conditions on $\rT$. Observe that
\begin{align*}
    L(\rT;F_n,c_n)&=\sup_{F,G\in B(F_n,c_n) } \frac{\rT(F)-\rT(G)}{\norm{F-G}} \\
    &= \sup_{F,G\in B_p(F_n,c_n) } \frac{\psi(1;F,G)-\psi(0;F,G)}{\norm{F-G}}\\
    & \leq \sup_{F,G\in B(F_n,c_n), t\in[0,1] } \frac{\psi^{\prime}(t;F,G)}{\norm{F-G}} \\
    &\leq \sup_{F,G\in B(F_n,c_n), t\in[0,1] } \boldsymbol{\upsilon}((1-t)F+t G) \\
    &= \sup_{F\in B(F_n,c_n)} \boldsymbol{\upsilon}(F),
\end{align*}
where $\boldsymbol{\upsilon}$ is a functional that satisfies for any $F,G$
\[  \psi'\prt{t;F,G})\leq 	\boldsymbol{\upsilon}((1-t)F+t G)\norm{F-G}.	 \]
Note that $\boldsymbol{\upsilon}$ implicitly depends on $p$ and the risk measure $\rT$. Consequently, we can obtain an upper bound on the LLC by bounding the last term. We can obtain the upper bound on the GLC by removing the ball constraint. For interested readers, please refer to  in Table \ref{tab:func} in Appendix \ref{app:imp} for the functional $\upsilon$ for different risk measures. We will use SRM as an example to illustrate the procedure. 
\subsubsection{An Example: SRM}
Here we consider an alternative form of SRM $ \boldsymbol{M_{\phi}}(F)=\int_a^b\phi(F(x))x dF(x)$.  $\psi$ is continuously differentiable with derivative 
\begin{align*}
\psi^{\prime}(t)&=\frac{d}{d t}\int_a^b\phi((1-t) F+t G)(x))x d((1-t) F+t G)(x)\\
&=-\int_a^b(G-F)(x)\phi((1-t) F+t G)(x))dx \\
&\leq \norm{G-F}_p \norm{\phi(F+t H)}_q.
\end{align*}
We omit the details in the second equality and leave the full derivations to Appendix \ref{app:imp}. Since
\begin{align*}
\frac{\psi^{\prime}(t;F,G)}{\norm{F-G}_p} \leq \norm{\phi((1-t) F+t G)}_q,
\end{align*}
where $\norm{\cdot}_q$ is the dual norm of $\norm{\cdot}_p$, we obtain $\boldsymbol{\upsilon}(F)=\norm{\phi(F)}_q$. 
\begin{table*}
	\caption{Comparison between the LLC and the GLC}	
	\label{tab:lip}
	\centering
	\begin{tabular}{ lccccccr  }
		\toprule
		RM  & Local $(p=1)$  &Global $(p=1)$ &Improvement & Local $(p=\infty)$  &Global $(p=\infty)$ &Improvement\\  
		\midrule
		CVaR   &$\frac{1}{\alpha}$ & $\frac{1}{\alpha}$ &\xmark &$\frac{b-F_n^{-1}((1-\alpha-c)^+)}{\alpha}$ & $\frac{b-a}{\alpha}$ &\cmark\\ [0.5ex] 
        SRM  & $\phi(1)$ & $\phi(1)$ &\xmark & $\norm{\phi(\underline{F^{\infty}_n})}_1$  & $(b-a)\phi(1)$ &\cmark\\ [0.5ex] 
        DRM &$\norm{g^{\prime}}_{\infty}$ &$\norm{g^{\prime}}_{\infty}$ &\xmark &$\norm{g^{\prime}(1-\underline{F^{\infty}_n})}_1$ &$(b-a)\norm{g^{\prime}}_{\infty}$ &\cmark\\ [0.5ex] 
        ERM  &$\frac{\exp(\beta b)}{\int_a^b\exp(\beta x)d \underline{F^1_n}(x))}$ &$\exp(\beta (b-a))$ &\cmark &$\frac{\exp(\beta b)-\exp(\beta a)}{\beta\int_a^b\exp(\beta x)d \underline{F^{\infty}_n}(x))}$ &$\frac{\exp(\beta (b-a))-1}{\beta}$ &\cmark\\ [0.5ex]
        RDEU & N/A\footnote{We cannot obtain the LLC of RDEU for $p=1$, but the GLC is still available.} & $\norm{w^{\prime}}_{\infty}\norm{v^{\prime}}_{\infty}$ & N/A & $\norm{w^{\prime}(\underline{F^{\infty}_n})v^{\prime}}_1$ & $\norm{w^{\prime}}_{\infty}\norm{v^{\prime}}_1$ &\cmark\\
        \bottomrule
	\end{tabular}
\end{table*}
Consider the case $p=\infty$. Since $\underline{F^{\infty}_n} \preceq F$ for $F\in B_{\infty}(F_n,c_n^{\infty})$,
then $ \phi\prt{\underline{F^{\infty}_n}(x)}\ge \phi(F(x)),\ \forall F\in B_{\infty}(F_n,c_n^{\infty}),\ \forall x\in[a,b]$.
It holds that
\begin{align*}
    \max_{F\in B_{\infty}(F_n,c_n^{\infty})}\norm{\phi(F)}_{1}&=\max_{F\in B_{\infty}(F_n,c_n^{\infty})} \int_a^b\phi(F(x))dx
    \\
    =\int_a^b\phi(\underline{F^{\infty}_n}(x))dx 
    &=\norm{\phi\prt{\underline{F^{\infty}_n}}}_{1}.
\end{align*}
In contrast, the GLC can be bounded by choosing $F=\delta_a$
\[ \max_{F\in \mathscr{D}([a,b])} \int_a^b\phi(F(x))dx=(b-a)\norm{\phi}_{\infty}=(b-a)\phi(1). \]

Following such principle, we obtain the LLCs for other risk measures (cf. Table \ref{tab:lip}). 

\subsection{Improvement of Distribution Optimization Framework}
\label{subsec:imp}
Equation \ref{eqt:comp} qualitatively establishes that the confidence bounds derived from our framework are tighter than that based on the LLC
\begin{equation*}
\rT\prt{\overline{F_n}}-\rT(F_n)\leq L(F_n,c_n)c_n.
\end{equation*}
Furthermore, we can quantitatively show the improvement 
\begin{align*}
\rT\prt{\overline{F_n}}-\rT(F_n)&=\boldsymbol{\psi}\prt{1;F_n,\overline{F_n}}-\boldsymbol{\psi}\prt{0;F_n,\overline{F_n}}\\
&\leq \max_{t\in[0,1]}\psi^{\prime}\prt{t;F_n,\overline{F_n}}\\
&\leq \max_{t\in[0,1]}\boldsymbol{\upsilon}\prt{(1-t)F_n+t\cdot\overline{F_n}} c_n,
\end{align*}
where the second to the last inequality follows from the definition of $\boldsymbol{\upsilon}$. The following also holds 
\begin{align*}
\rT(F_n)-\rT\prt{\underline{F_n}}\leq \max_{t\in[0,1]}\boldsymbol{\upsilon}\prt{(1-t) \underline{F_n}+t F_n} c_n.
\end{align*}
Therefore, it is more convenient to compare $\rT\prt{\overline{F_n}}-\rT(F_n)$ and $L(F_n,c_n)c_n$,
which is shown for the supremum distance in Table \ref{tab:impr}. For convenience, we normalize them by $c_n$ and state the results for general CDF $F$. Our UCBs  are strictly and consistently tighter than the LLC-based bounds for the supremum distance. Due to the space limit, the results for the Wasserstein distance are shown in Appendix \ref{app:imp}. 
\begin{table*}
	\caption{Improvement of confidence bounds for supremum distance over LLC}
	\label{tab:impr}
	\centering
	\begin{tabular}{ lcccccr  }
		\toprule
		RM & CVaR &SRM &DRM &ERM &RDEU \\  
		\midrule
         $L_{\infty}(\rT;F,c)$  & $\frac{b-F^{-1}(1-\alpha-c)}{\alpha}$    & $\norm{\phi(\underline{F^{\infty}})}_1$ &   $\norm{g^{\prime}(1-\underline{F^{\infty}})}_1$  & $\frac{\exp(\beta b)-\exp(\beta a)}{\int_a^b\exp(\beta x)d \underline{F^{\infty}}(x)}$    & $\norm{w^{\prime}(\underline{F^{\infty}})v^{\prime}}_1 $\\ [1.ex] 
        $\frac{\rT\prt{\overline{F^{\infty}}}-\rT(F)}{c}$  & $\frac{b-F^{-1}(1-\alpha)}{\alpha}$    & $ \norm{\phi(F)}_1$ &   $\norm{g^{\prime}(1-F)}_1$  & $\frac{\exp(\beta b)-\exp(\beta a)}{\beta \int_a^b\exp(\beta x)d F(x)}$   & $\norm{w^{\prime}(F)v^{\prime}}_1 $ \\ [1.ex] 
        Improvement  &\cmark  &\cmark &\cmark  &\cmark   &\cmark \\ [1.ex] 
    \bottomrule
	\end{tabular}
\end{table*}
\vspace{-4ex}
\subsection{Illustrating example: CVaR}
We use CVaR to illustrate Table \ref{tab:lip} and Table \ref{tab:impr}. Table \ref{tab:lip} compares the LLC with GLC for different risk measures. The second row in Table \ref{tab:lip} shows that  the LLC and GLC of CVaR for the Wasserstein distance are identical. In addition, the GLC of CVaR for supremum distance is $\frac{b-a}{\alpha}$, which is larger than its LLC  $\frac{b-F^{-1}((1-\alpha-c)^+)}{\alpha}$. 

Table \ref{tab:impr} presents the improvement of our bound for the supremum distance. The second column in Table \ref{tab:impr} implies 
\begin{align*}
\frac{\boldsymbol{C_{\alpha}}\prt{\overline{F^{\infty}}}-\boldsymbol{C_{\alpha}}(F)}{c} &= \frac{b-F^{-1}(1-\alpha)}{1-\alpha} \\
<  L_{\infty}(\boldsymbol{C_{\alpha}};F,c) &= \frac{b-F^{-1}(1-\alpha-c)}{\alpha} \\
< L_{\infty}(\boldsymbol{C_{\alpha}})  &=  \frac{b-a}{\alpha}.
\end{align*}
Our upper bound is strictly tighter than the LLC-based and GLC-based bound. The improvement depends on the distribution $F$ and $\alpha$. For small $\alpha$ and distribution $F$ with non-fat upper tail, $b-F^{-1}(1-\alpha)\ll b-a$, leading to a much finer bound. In particular, consider a uniform distribution  
\begin{align*}
\frac{\boldsymbol{C_{\alpha}}\prt{\overline{F^{\infty}}}-\boldsymbol{C_{\alpha}}(F)}{c} &= \frac{\alpha(b-a)}{\alpha} \\
<  L_{\infty}(\boldsymbol{C_{\alpha}};F,c) &= \frac{(\alpha+c)(b-a)}{\alpha} \\
< L_{\infty}(\boldsymbol{C_{\alpha}})  &=  \frac{b-a}{\alpha}.
\end{align*}
Our bound for uniform distribution considerably improves by a factor of $1/\alpha$ and $(\alpha+c)/\alpha$ compared to the GLC-based and LLC-based bound, respectively.

\section{Application to Risk-sensitive Bandits}
\label{sec:app}
Now we consider the risk-sensitive multi-armed bandit (MAB) problems. The quality of each arm is measured by the risk measure value of its loss distribution. The loss distribution of the $i$-th arm is denoted by $F_i$, and the risk value associated with $\rT$ is $\rT(F_i)$. The algorithm interacts with a bandit instance $\nu=(F_i)_{i\in[K]}$ for $N$ rounds
. In each round $t\in[N]$, the algorithm $\pi$ chooses an arm $I_t\in[K]$ and observes the loss $X_t\sim F_{I_t}$. The performance of an algorithm $\pi$ is measured by the cumulative regret 
\[	\text{Regret}(\pi,\nu,N)\triangleq \bE\brk{\sum_{t\in[N]} \rT(F_{I_t})-\min_{i\in[K]}\rT(F_{i}) }. 		\] 
While the UCB-type algorithms \cite{auer2002finite} are widely applied to the risk-neutral MAB problems, they rely on the concentration bound of the mean. We propose a \emph{meta-algorithm} (cf. Algorithm \ref{alg:lcb}) to deal with generic risk measures, where the LCB is derived from our framework. The algorithm maintains the EDF $\hat{F}_{i,t}$ for each arm $i$ 
\begin{equation}
\vspace{-1ex}
\label{eqt:edf_lcb}
\hat{F}_{i,t}\triangleq \frac{1}{s_i(t)}\sum_{t'=1}^{t-1}\mathbb{I}\cbrk{X_{t'}\leq \cdot, I_{t'}=i},
\end{equation}
where $s_i(t)\triangleq\sum_{t'=1}^{t-1}\mathbb{I}\{I_{t'}=i\}$ is the number of times of pulling arm $i$ up to time $t$. For convenience, we assume the first arm is the optimal arm, i.e., $\rT(F_1)< \rT(F_i), i\neq 1$.
\begin{algorithm}[tb]
	\caption{\texttt{Lower Confidence Bound}}
	\label{alg:lcb}
	\begin{algorithmic}[1]
		\STATE{\textbf{Input:} $N$,  $a$}
		\FOR{round $t\in[K]$}
		\STATE{Pull arm $I_t\leftarrow t$}
		\ENDFOR
		\FOR{round  $t = K+1, K+2, \cdots, N$}
		\STATE{Compute $\hat{F}_{i,t}$ via Equation \ref{eqt:edf_lcb}}
		\STATE{Set $c_i(t)\leftarrow \sqrt{\frac{\log(2KN^2)}{s_i(t)}}$ for all $i\in[K]$}
		\STATE{$\underline{F_{i,t}}\leftarrow \boldsymbol{\rN^{\infty}_{c_i(t)}}\hat{F}_{i,t}$}
		\STATE{$I_t\leftarrow \arg\min_{i\in[K]}\boldsymbol{\mathrm{T}}\prt{\underline{F_{i,t}}}$}
		\ENDFOR
	\end{algorithmic}
\end{algorithm}
When specializing the risk measure to the CVaR, we obtain a distribution-dependent regret upper bound.
\begin{proposition}
\label{prop:reg}
The expected regret of  Algorithm \ref{alg:lcb} on a instance $\nu\in \sD([a,\infty])$ with $\rT=C_{\alpha}$ is bounded as
\begin{align*}
&\text{Regret}(\text{LCB},\nu,N) \\
&\leq \frac{4\log(\sqrt{2}N)}{\alpha^2}\sum_{i>1}^K\frac{\prt{b-F_i^{-1}(1-\alpha-2c^*_i)}^2}{\Delta_i}+3\sum_{i=1}^K\Delta_i,
\end{align*}
\vspace{-0ex}
where the sub-optimality gap $\Delta_i\triangleq C_{\alpha}(F_i)-C_{\alpha}(F_1)<b-F_i^{-1}(1-\alpha)$, and $c_i^*$ is a $F_i$-dependent constant that solves the equation $2\frac{b-F_i^{-1}(1-\alpha-2c)}{\alpha}c=\Delta_i$.
\end{proposition}
The proof is deferred to Appendix \ref{app:reg}.
\begin{remark}
The distribution-dependent constant $c_i^*$ always exists and satisfies $c_i^*\in(0,(1-\alpha)/2)$. Let $g(c)\triangleq2\frac{b-F_i^{-1}(1-\alpha-2c)}{\alpha}c$. Observe that $g$ is an increasing function with $g(0)=0$ and $g((1-\alpha)/2)=\frac{1-\alpha}{\alpha}(b-a)>\Delta_i$ for $\alpha\in(0,0.5]$.
In particular, for uniform distribution, we can show that
\begin{align*}
    b-F_i^{-1}(1-\alpha-2c^*_i) = \frac{b-a}{2}\prt{ \sqrt{\alpha^2+\frac{4\alpha\Delta_i}{b-a}}+\alpha}.
\end{align*}
\end{remark}
\begin{remark}
The meta-algorithm for CVaR reduces to the \texttt{CVaR-UCB} algorithm \cite{tamkin2019distributionally}. 
Interestingly, \citet{tamkin2019distributionally} empirically observes that \texttt{CVaR-UCB} outperforms the GLC-based algorithm \texttt{U-UCB} \cite{cassel2018general}  with an order of magnitude improvement, but they only provide a regret bound of
\[ \frac{4\log(\sqrt{2}N)}{\alpha^2}\sum_{i>1}^K\frac{\prt{b-a}^2}{\Delta_i}+3\sum_{i=1}^K\Delta_i,\]
which matches that of \texttt{U-UCB}. We fill this gap by quantifying the improvement in the magnitude
\[ \sum_{i>1}\frac{\prt{b-F_i^{-1}(1-\alpha-2c^*_i)}^2}{\Delta_i^2}/\sum_{i>1}\frac{(b-a)^2
}{\Delta_i^2} < 1.	\]
\end{remark}
\begin{remark}
\citet{baudry2021optimal} introduces a Thompson Sampling algorithm \texttt{B-CVTS} for CVaR bandit with bounded rewards, which is the first  asymptotic optimal CVaR bandit algorithm. Notably, our main contribution in this paper is the framework to improve confidence bounds rather than designing optimal CVaR bandit algorithms. Additionaly, \texttt{CVaR-UCB} has several advantages. \texttt{CVaR-UCB} can perform  incremental updates to compute the CVaR values, but \texttt{B-CVTS} needs to maintain and sample from the posterior of distributions. Thus the time and space complexity of \texttt{CVaR-UCB} is quite low. Moreover, \texttt{CVaR-UCB} can be applied to semi-unbounded distributions, e.g., log-normal distribution, while \texttt{B-CVTS} assumes bounded rewards.
\end{remark}

\vspace{-1ex}
\section{Numerical Experiments}
\label{sec:exp}
To better visualize the benefits of our framework relative to those of LLC and GLC, we conducted a series of empirical comparisons. Details and complete figures are deferred to Appendix \ref{app:num}.
\vspace{-1ex}
\paragraph{Confidence bounds.}  We consider five different beta distributions and two risk measures: CVaR and ERM. 
Due to space limitations, we provide the results for one typical beta distribution and one particular distance in Figure 4. Our bounds are consistently tighter than LC-based ones for various risk measures and varying sample sizes.
\begin{figure}[t]
     \centering
     \begin{subfigure}[b]{0.22\textwidth}
         \centering
         \includegraphics[width=1.1\textwidth]{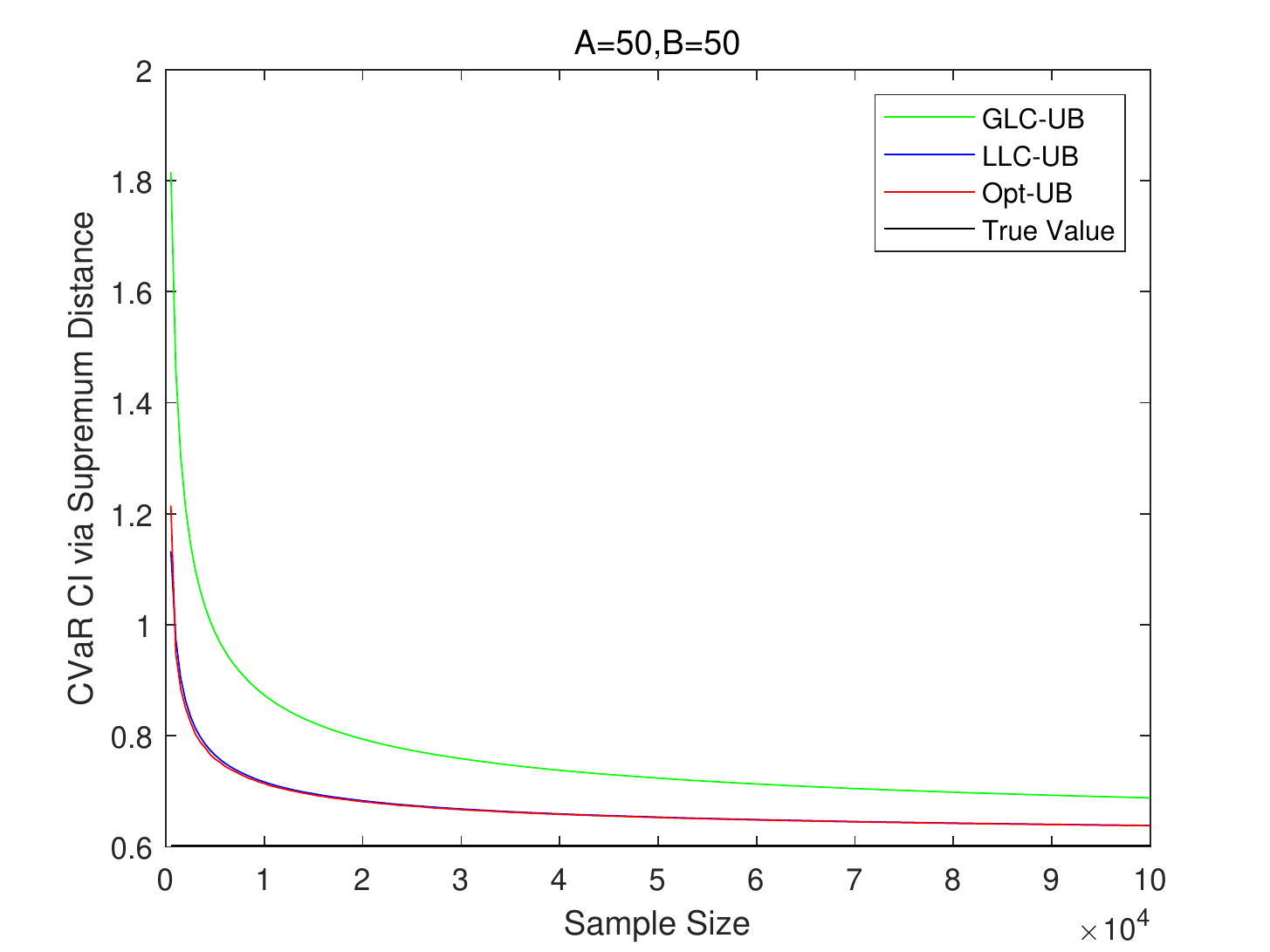}
         \caption{CVaR UCB via supremum distance}
     \end{subfigure}
     \hfill
     \begin{subfigure}[b]{0.22\textwidth}
         \centering
         \includegraphics[width=1.1\textwidth]{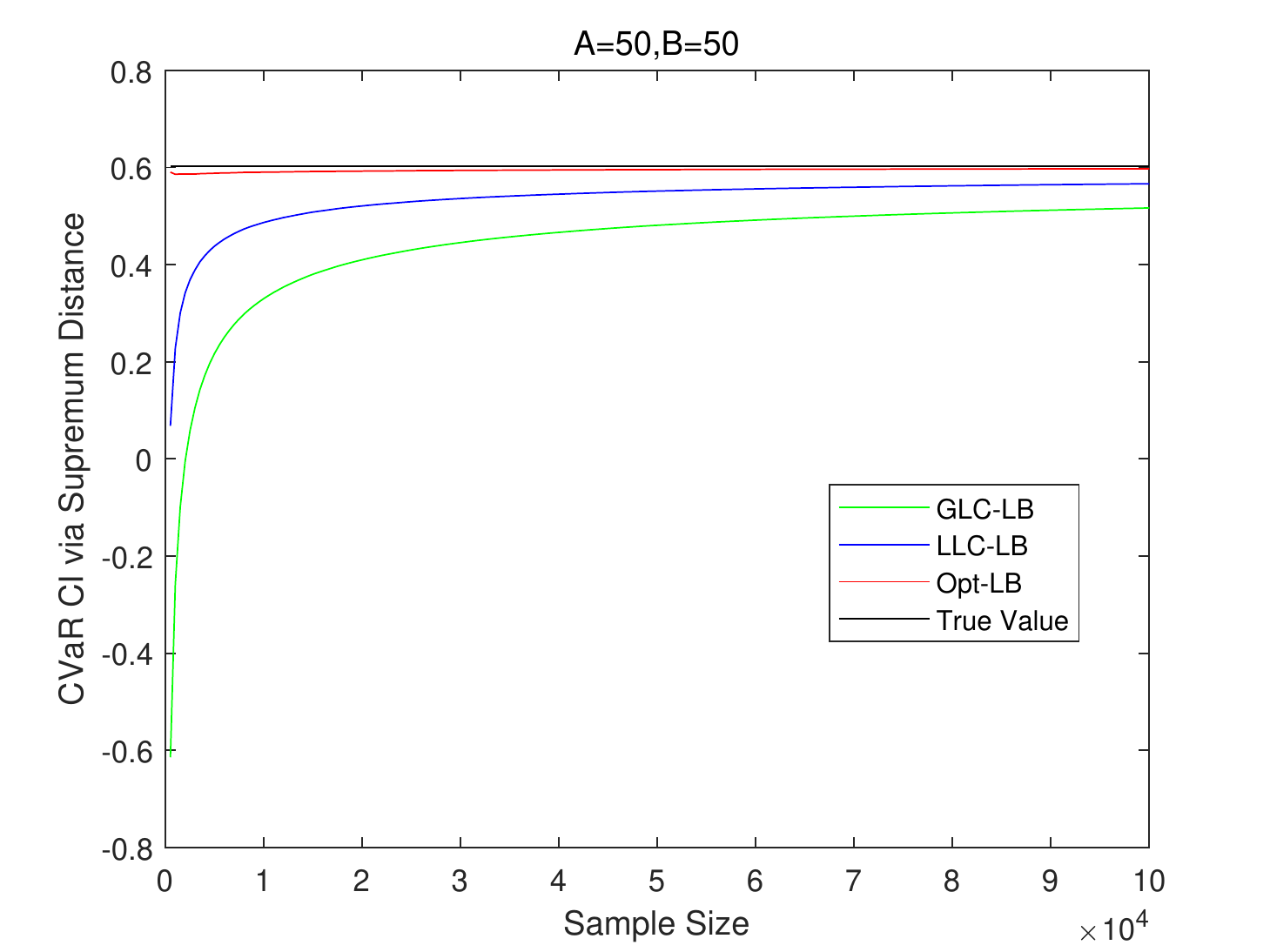}
         \caption{CVaR LCB via supremum distance}
     \end{subfigure}
    \begin{subfigure}[b]{0.22\textwidth}
         \centering
         \includegraphics[width=1.1\textwidth]{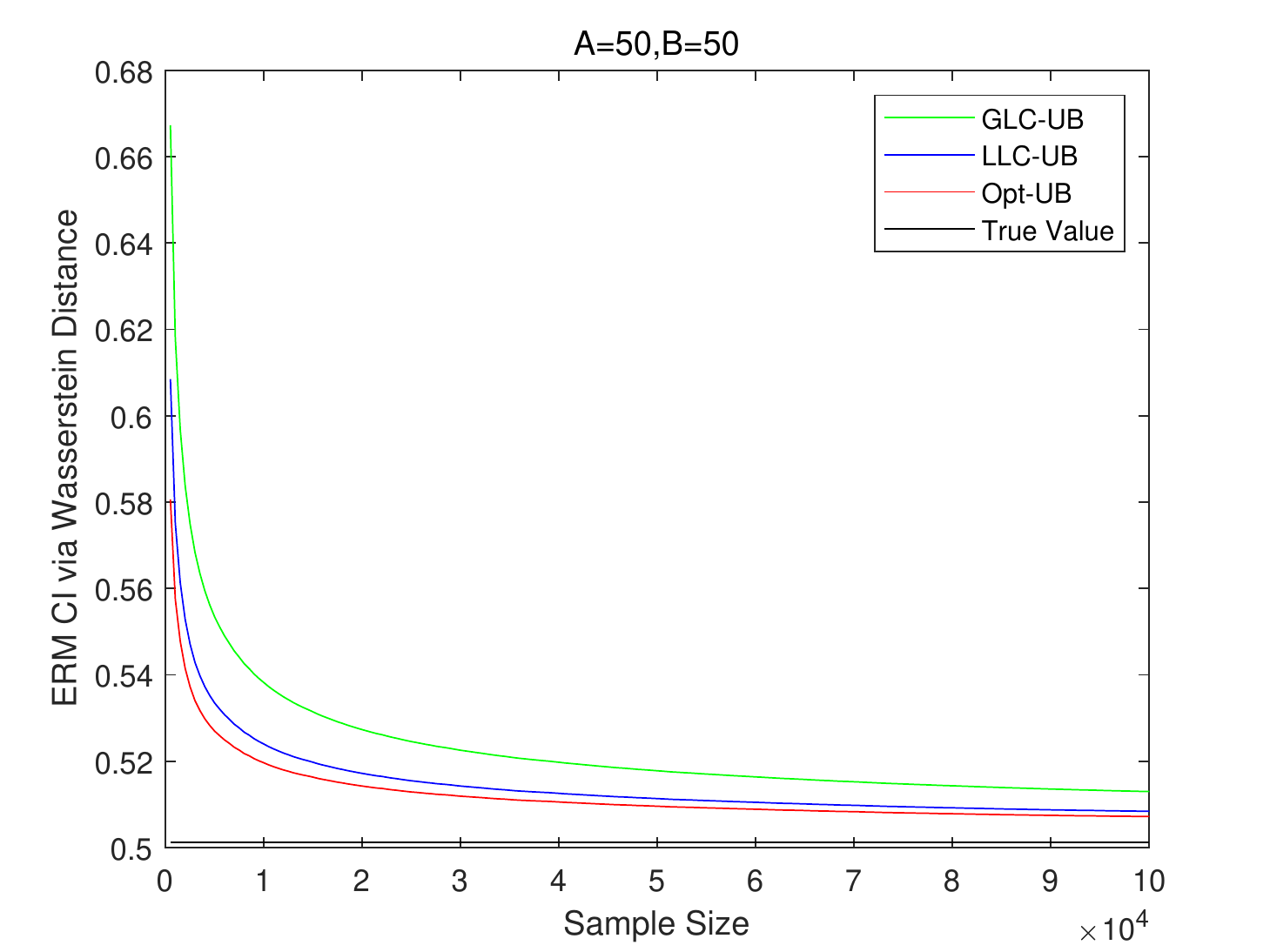}
         \caption{ERM UCB via Wasserstein distance}
     \end{subfigure}
     \hfill
     \begin{subfigure}[b]{0.22\textwidth}
         \centering
         \includegraphics[width=1.1\textwidth]{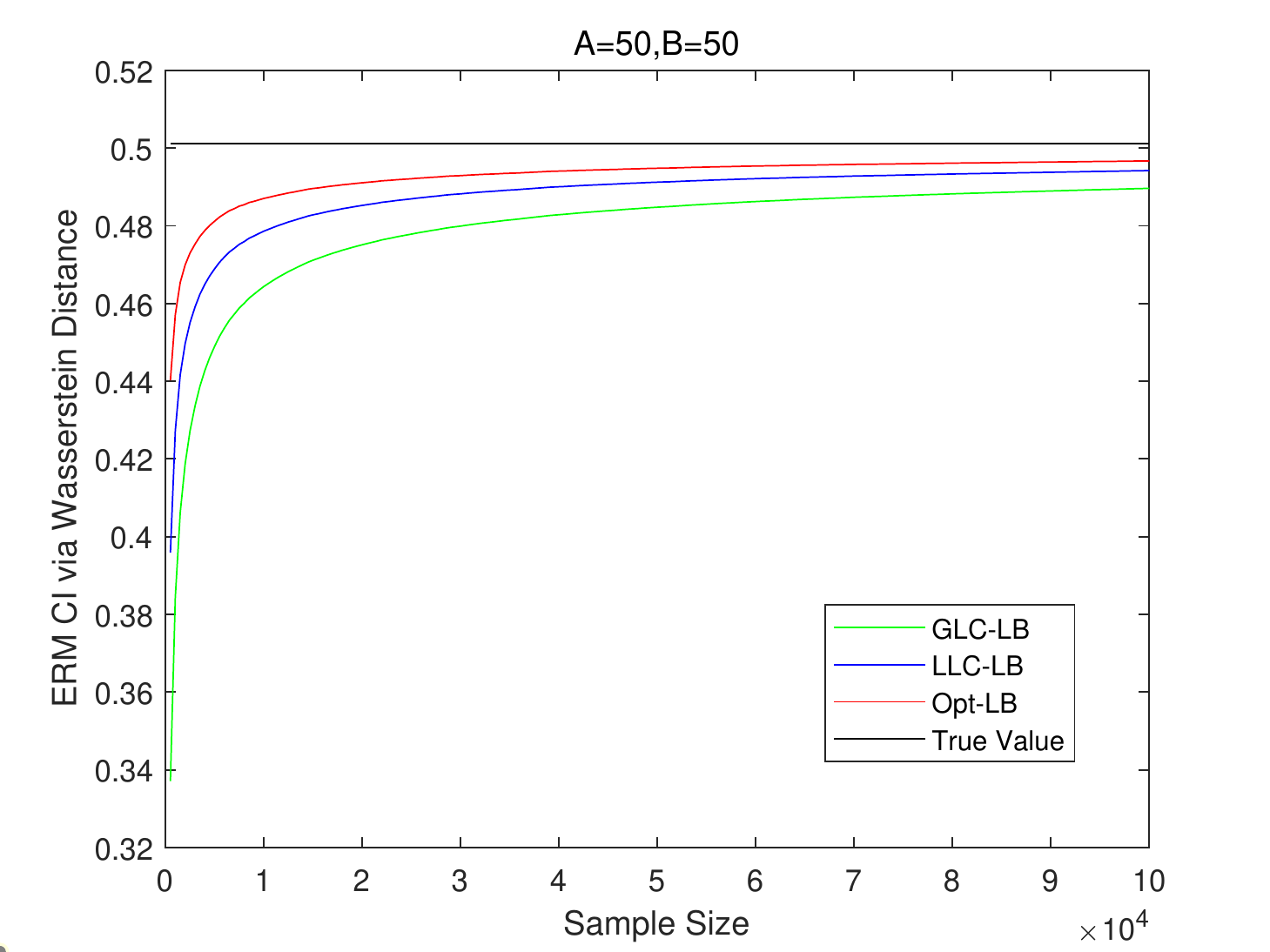}
         \caption{ERM LCB via Wasserstein distance}
     \end{subfigure}
        \caption{Comparisons of CIs for CVaR and ERM with varying sample sizes.}
        \label{fig:ci}
        \vspace{-2ex}
\end{figure}

\vspace{-1ex}
\paragraph{CVaR bandits.}    
We compare  \texttt{CVaR-UCB} with the UCB algorithm using GLC (\texttt{GLC-UCB}) and using LLC (\texttt{LLC-UCB}) in Figure 5. It shows that \texttt{CVaR-UCB} outperforms \texttt{LLC-UCB} and \texttt{GLC-UCB}.
\begin{figure}[ht]
\begin{center}
	\centerline{\includegraphics[width=0.40\textwidth]{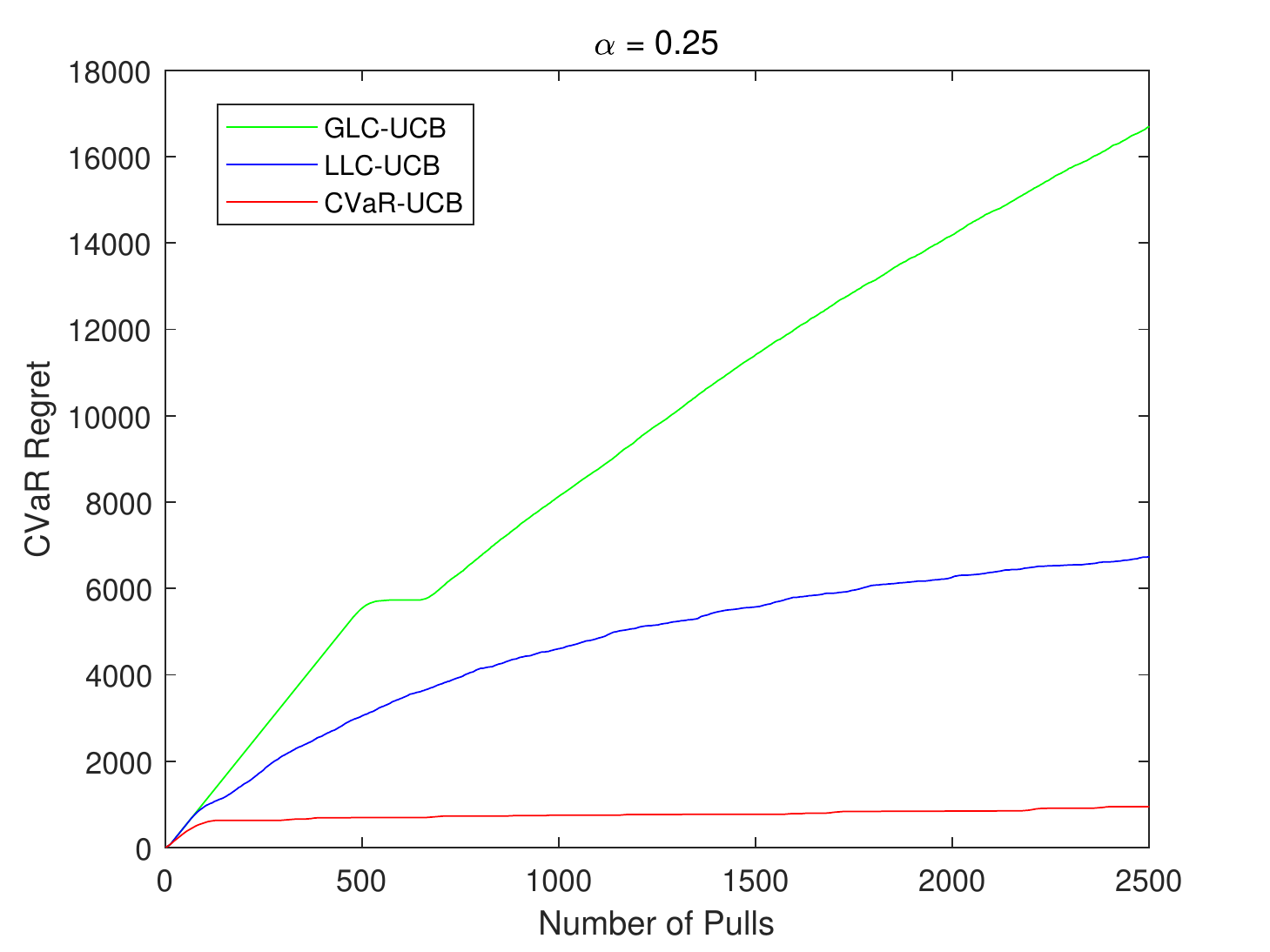}}
	\caption{Cumulative CVaR-regret of \texttt{CVaR-UCB} (red), \texttt{LLC-UCB} (blue), and \texttt{GLC-UCB} (green). }
\end{center}
\vspace{-6ex}
\end{figure}

\vspace{-1ex}
\section{Conclusion}
\label{sec:con}
We propose a distribution optimization framework to obtain improved confidence bounds of several risk measures. By viewing the solutions as certain transformations of the EDF, we design efficient algorithms to compute the confidence bounds. The tightness of our bounds is further illustrated via comparisons with the new baseline method.

The major limitation is that our framework only deals with bounded distribution in general. However, it is applicable to semi-unbounded distributions for CVaR, SRM, DRM, and RDEU. It would be interesting to study the distribution optimization framework under more general assumptions, e.g., sub-Gaussian or sub-exponential distributions. Another  promising future direction is to generalize the framework to the multivariate setting. 
One may apply the multivariate DKW inequality to the multivariate risk measures. 

\section*{Acknowledgements}
We thank all the anonymous reviewers for their helpful comments and suggestions. The work of Zhi-quan Luo was supported in part by the National Key Research and Development Project under grant 2022YFA1003900 and in part by the Guangdong Provincial Key Laboratory of Big Data Computing.

\newpage
\bibliography{example_paper.bib}
\bibliographystyle{icml2023}

\newpage
\appendix
\onecolumn

\section{Table of Notation}
$$
\begin{array}{ll}
\hline \text { Symbol } & \text { Explanation } \\
\hline \mathscr{D} & \text { The space of all CDFs} \\
\mathscr{D}([a,b]) & \text { The space of all CDFs supported on $[a,b]$} \\
B_p(F,c) & \text { The $\norm{\cdot}_p$ norm ball centered at $F$ with radius $c$ } \\
F_n & \text { The empirical distribution function corresponding to $n$ samples from $F$}  \\
c_n^p & \text { The confidence radius w.r.t. $\norm{\cdot}_p$ for $n$ samples}  \\
\rT & \text { Risk measure } \\
L_p(\rT) & \text { The global Lipschitz constant of $\rT$ w.r.t. $\norm{\cdot}_p$ } \\
L_p(\rT;F,c) & \text { The local Lipschitz constant of $\rT$ w.r.t. $\norm{\cdot}_p$ over $B_p(F_n,c_n^p)$ } \\
\overline{F^p_n} & \text { The maximizer of Formulation \ref{eqt:opt_was} or \ref{eqt:opt_inf}} \\
\underline{F^{p}_n} & \text { The minimizer of Formulation \ref{eqt:opt_was} or \ref{eqt:opt_inf}} \\
\boldsymbol{\mathrm{P}_{c}^{1}} & \text { The positive operator  w.r.t. $\norm{\cdot}_1$ with coefficient $c$ } \\
\boldsymbol{\mathrm{N}_{c}^{1}} & \text { The negative operator  w.r.t. $\norm{\cdot}_1$ with coefficient $c$ } \\
\boldsymbol{\mathrm{P}_{c}^{\infty}} & \text { The positive operator  w.r.t. $\norm{\cdot}_{\infty}$ with coefficient $c$ } \\
\boldsymbol{\mathrm{N}_{c}^{\infty}} & \text { The negative operator  w.r.t. $\norm{\cdot}_{\infty}$ with coefficient $c$ } \\
\nu & \text { Bandit instance } \\
N & \text { Number of total rounds } \\
K & \text { Number of total arms } \\
\pi & \text { Bandit algorithm } \\
s_i(t) & \text { The number of times of pulling arm $i$ up to time $t$ } \\
\hline
\end{array}
$$

\section{Risk Measures}
\label{app:rm}
\paragraph{Conditional Value at Risk (CVaR)}
CVaR \cite{rockafellar2000optimization} is a popular risk measure in  financial portfolio optimization \cite{krokhmal2002portfolio,zhu2009worst}. Formally, the CVaR value at level $\alpha\in(0,1)$ for a distribution $F$ is defined as
\[  \boldsymbol{C_{\alpha}}(F)\triangleq \inf_{\nu\in\bR}\cbrk{ \nu + \frac{1}{1-\alpha} \bE_{X\sim F}[(X-\nu)^+]}.    \]
\citet{acerbi2002coherence} showed that when $F$ is a continuous distribution,  $\boldsymbol{C_{\alpha}}(F)=\bE_{X\sim F}[X|X\ge F^{-1}(1-\alpha)]$.
\paragraph{Spectral risk measure (SRM)}
SRM is a generalization of CVaR that adopts a non-constant weighting function \cite{acerbi2002spectral}. The SRM of $F$ is defined as
\[  \boldsymbol{S_{\phi}}(F)\triangleq \int_0^1 \phi(y)F^{-1}(y)dy, \]
where $\phi:[0,1]\rightarrow[0,\infty)$ is weighting function. $\phi$ is said to be admissible if it is increasing and satisfies that $\int_0^1\phi(y)dy=1$. \citet{acerbi2002spectral} showed that $\boldsymbol{S_{\phi}}(F)$ is a coherent risk measure if $\phi$ is admissible. SRM can be viewed as a weighted average of the quantiles $F^{-1}$, with weight specified by $\phi(y)$. In fact, $\boldsymbol{S_{\phi}}(F)$ specializes in $\boldsymbol{C_{\alpha}}+(F)$ for $\phi(y)=\frac{1}{1-\alpha}\mathbb{I}\{y\ge1-\alpha\}$.
\paragraph{Distortion risk measure (DRM)}
DRM  is originally from the insurance problems and later applied to investment risks \cite{wang1996premium,wang2004cat}. For a distribution $F\in \sD([0,\infty))$, the DRM $\rho_g(F)$ is defined as
\[  \boldsymbol{\rho_g}(F)\triangleq \int_0^{\infty}g(1-F(x))dx,    \]
where $g:[0,1]\rightarrow[0,1]$ is a continuous increasing function with $g(0)=0$ and $g(1)=1$. We refer to $g$ as the distortion function. Similar to SRM, DRM reduces to CVaR by setting $g(y)=\min\prt{\frac{x}{1-\alpha},1}$.
\paragraph{Entropic risk measure (ERM)}
ERM is a well-known risk measure in risk-sensitive decision-making, including mathematical finance \cite{follmer2016stochastic}, Markovian decision processes \cite{bauerle2014more}. The ERM value of $F$ with coefficient $\beta\neq 0$ is defined as
\[  \boldsymbol{U_{\beta}}(F)\triangleq \frac{1}{\beta}\log(\bE_{X\sim F}[\exp(\beta X)])=\frac{1}{\beta}\log\prt{\int_{\bR}\exp(\beta x)dF(x)}.\]
\paragraph{Certainty equivalent}
Certainty equivalent can be viewed as a generalization of  ERM, which replace the exponential utility function with a  more general  function. Let $u$ be a continuous and strictly increasing function such that its inverse $u^{-1}$ exists, then the certainty equivalent $C_u(F)$ of $F$ associated with a function $u$ is given by
\[  \boldsymbol{C_u}(F)\triangleq u^{-1}(\bE_{X\sim F}[u(X)])=u^{-1}\prt{\int_{\bR}u(x)dF(x) }. \]
It is trivial that the certainty equivalent $C_u(F)$ reduces to the ERM when $u(x)=\exp(\beta x)$.

\paragraph{Rank dependent expected utility (RDEU)}
RDEU value \cite{quiggin2012generalized} of $F\in\mathscr{D}([a,b])$ is defined as
\begin{align*}
\boldsymbol{V}(F)\triangleq \int_{a}^{b}v(x)dw(F(x)),
\end{align*}
where $w:[0,1]\rightarrow[0,1]$ is an increasing weight function such that $w(0)=0$ and $w(1)=1$, and $v:\mathbb{R}\rightarrow\mathbb{R}$ be an (unbounded) increasing differentiable function with $v(0)=0$.

\section{Proof of Theorems}
\label{app:thm}
For CDFs $F,G \in \mathscr{D}$, the Wasserstein distance between them can be represented by their IDFs
 \begin{align*}
     W_1(F,G)=\int_{-\infty}^{\infty}\abs{F(x)-G(x)}dx  = \int_{0}^{1}\abs{F^{-1}(y)-G^{-1}(y)}dy,
 \end{align*}
With slight abuse of notation, we write $W_1(F,G)=\norm{F^{-1}-G^{-1}}_1$.
We will prove the theorems for the  more general formulations in the following. 
\begin{equation}
\label{eqt:opt_was_gen}
\begin{array}{rrclcl}
\displaystyle \max_{G\in\sD([a,b])} & \multicolumn{3}{l}{\rT(G)}\\
\textrm{s.t.} & \norm{G-F}_1\leq c
\end{array}
\end{equation}
and
\begin{equation}
\label{eqt:opt_inf_gen}
\begin{array}{rrclcl}
\displaystyle \max_{G\in\sD([a,b])} & \multicolumn{3}{l}{\rT(G)}\\
\textrm{s.t.} & \norm{G-F}_{\infty}\leq c
\end{array}
\end{equation}
Observe that $\norm{G-F}_1=\norm{G^{-1}-F^{-1}}_1$, thus we can recast  Formulation \ref{eqt:opt_was} as
\begin{equation}
\label{eqt:opt_was_inv_gen}
\begin{array}{rrclcl}
\displaystyle \max_{G\in\sD([a,b])} & \multicolumn{3}{l}{\rT(G^{-1})}\\
\textrm{s.t.} & \norm{G^{-1}-F^{-1}}_1\leq c
\end{array}
\end{equation}

\begin{proposition}
\label{prop:opt_was}
For  risk measures ERM and CE  defined in Section \ref{subsec:rm}, the optimal solution to  Formulation \ref{eqt:opt_was_gen} is given by
\begin{equation}
\begin{aligned}
    \overline{F^1} = \boldsymbol{\rP^1_{c}}F, \  \underline{F}^{1} = \boldsymbol{\rN^1_{c}}F,
\end{aligned}
\end{equation}
where $ \boldsymbol{\rP^1_{c}}/\boldsymbol{\rN^1_{c}}:\sD([a,b])\rightarrow\sD([a,b])$ is  the positive/negative operator for CDF for the Wasserstein distance, which defined as follows. 
\end{proposition}
Define $S^+(F,x)\triangleq \int_x^b \prt{F(z)-F(x)}dz$. Its geometric meaning is the area sandwiched between $F$ and a constant $F(x)$ from $x$ (see Figure 4 (a)). Notice that $S^+$ may be discontinuous w.r.t. $x$ since $F$ can be discontinuous w.r.t. $x$ ($S^+(x)<S^+(x^-)$). For $c>0$, we define $g^+(F,c)\triangleq\max\{x\ge a: S^+(F,x)\leq c\}\in[a,b)$. For simplicity, we drop $F$ from the notations if it is clear from the context. Given $F\in\sD([a,b])$ as input, $\mathrm{O}_c^1$ outputs a CDF 
\[      \prt{\boldsymbol{\rP^1_{c}} F}(x)\triangleq  \begin{cases}
F(g^+(c))-\frac{c-S^+(g^+(c))}{b-g^+(c)}, x\in[g^+(c),b), \\
F(x), \text{otherwise}.
\end{cases} 
\]
Analogously, we define $S^-(F,x)\triangleq \int_x^b \prt{1-F(z)}dz$ and $g^-(F,c)\triangleq\max\{x\ge a: S^-(F,x)\leq c\}\in[a,b)$. We omit $F$ for simplicity. Note that $S^-$ is continuous w.r.t. $x$. Hence $g^-(c)=\{x: S^-(x)= c\}$. For $F\in\sD([a,b])$, $\boldsymbol{\rN^1_{c}}$ outputs a CDF 
\[      \prt{\boldsymbol{\rN^1_{c}} F}(x)\triangleq  \begin{cases}
1, x\in[g^-(c),b), \\
F(x), \text{otherwise}.
\end{cases} 
\]
\begin{proposition}
\label{prop:opt_was_inv}
For  risk measures CVaR, SRM and DRM defined in Section \ref{subsec:rm}, the optimal solution to  \ref{eqt:opt_was_inv_gen} is given by
\begin{equation}
\begin{aligned}
    \prt{\overline{F^1}}^{-1} = \boldsymbol{\rP^1_{c}} F^{-1}, \  \prt{\underline{F^{1}}}^{-1} = \boldsymbol{\rN^1_{c}} F^{-1},
\end{aligned}
\end{equation}
where we overload notations to denote by $ \boldsymbol{\rP^1_{c}}/\boldsymbol{\rN^1_{c}}:(\sD([a,b]))^{-1}\rightarrow(\sD([a,b]))^{-1}$ the positive/negative operator for IDF for the Wasserstein distance.
\end{proposition}
We overload notations and define $S^+(F^{-1},y)\triangleq \int_y^1 \prt{b-F^{-1}(z)}dz$.  For $c>0$, we define $g^+(F^{-1},c)\triangleq\min\{y: S^+(F^{-1},y)\ge c\}\in(0,1)$. For simplicity, we drop $F$ from the notations if it is clear from the context. Given $F^{-1}\in\sD([a,b])^{-1}$ as input, $\boldsymbol{\rP^1_{c}}$ outputs a IDF 
\[      \prt{\boldsymbol{\rP^1_{c}} F^{-1}}(y)\triangleq 
\begin{cases}
1, y\in(g^+(c),1], \\
F^{-1}(y), \text{otherwise}.
\end{cases} 
\]
Analogously, we define $S^-(F^{-1},y)\triangleq \int_y^1 \prt{F^{-1}(z)-F^{-1}(y)}dz$ and $g^-(F^{-1},c)\triangleq\min\{y: S^-(F^{-1},y)\ge c\}\in(0,1)$ for $c>0$. Notice that $S^-$ may be discontinuous w.r.t. $y$ ($S^-(y+)<S^-(y)$). Given $F^{-1}\in\sD([a,b])^{-1}$ as input, $\boldsymbol{\rN^1_{c}}$ outputs a IDF 
\[      \prt{\boldsymbol{\rN^1_{c}} F^{-1}}(y)\triangleq \begin{cases}
F^{-1}(g^-(c))+\frac{S^{-1}(g^{-1}(c))-c}{1-g^-(c)}, y\in(g^-(c),1], \\
F^{-1}(y), \text{otherwise}.
\end{cases} 
\]
\begin{proposition}
\label{prop:opt_inf}
For any  risk measure, the optimal solution to Formulation \ref{eqt:opt_inf_gen} is given by
\begin{equation}
\begin{aligned}
     \overline{F^{\infty}} = \boldsymbol{\rP^{\infty}_{c}} F, \
    \underline{F}^{\infty} = \boldsymbol{\rN^{\infty}_{c}} F,
\end{aligned}
\end{equation}
where $\boldsymbol{\rP^{\infty}_{c}}/\boldsymbol{\rN^{\infty}_{c}}:\sD([a,b])\rightarrow\sD([a,b])$  is the positive/negative operator with coefficient $c>0$ for the supremum distance, which is defined as follows
\begin{align*}
    \prt{\boldsymbol{\rP^{\infty}_{c}} F}(x) &\triangleq \max\cbrk{F(x) - c\mathbb{I}\{x\in[a,b),0 },\\
    \prt{\boldsymbol{\rN^{\infty}_{c}} F}(x) &\triangleq \min\cbrk{F(x) + c\mathbb{I}\{x\in[a,b),1}.
\end{align*}
\end{proposition}

Figure 4 
illustrates how the operators defined in Proposition \ref{prop:opt_was}-\ref{prop:opt_inf} transforms a typical continuous CDF $F\in\sD([a,b])$.

\subsection{Proof of Proposition \ref{prop:opt_was}}
We only provide the proof for CE because ERM is a special case of CE by choosing $u(x)=\exp(\beta x)$. For simplicity, we write the optimization problem as 
\begin{equation*}
\begin{array}{rrclcl}
\displaystyle \max_{G\in\sD([a,b])} & \multicolumn{3}{l}{\boldsymbol{C_u}(G)}\\
\textrm{s.t.} & \norm{G-F}_1\leq c
\end{array}
\end{equation*}
\begin{proof}
	We consider the maximization problem first. 
	The objective function can be written as
	\[  \boldsymbol{C_u}(G)=u^{-1}\prt{\int_a^b u(x)dG(x)}=u^{-1}\prt{u(b)-\int_a^b G(x)u^{\prime}(x)dx}. \]
	Since $u^{-1}$ is monotonically increasing, we can reformulate 	the original optimization problem as 
	\begin{equation*}
	\begin{array}{rrclcl}
	\displaystyle \min_{G} & \multicolumn{3}{l}{\int_a^b G(x)u^{\prime}(x)dx}\\
	\textrm{s.t.} &G\in B_1(F,c)
	\end{array}
	\end{equation*}
	The fact that the ball constraint $G\in B_1(F,c)$ is symmetric about $F$ implies  $G^*(x)\leq F(x), \forall x\in[a,b]$. Suppose $G^*(y)>F(y)$ for $y$ in a set that is a union of  disjoint intervals $\cup_i I_i=\cup_i (a_i,b_i)$, and $G^*(y)\leq F(y)$ otherwise. We can choose 
	\[      H(y)=
	\begin{cases}
	\max\{2F(y)-G^*(y),F(a_i)\}, y\in  I_i, \\
	G^*(y), \text{otherwise},
	\end{cases} 
	\]
	which satisfies the ball constraint. However, $C_u(H)>C_u(G^*)$ since $\int_{I_i}H(x)u^{\prime}(x)dx\leq \int_{I_i}F(x)u^{\prime}(x)dx<\int_{I_i}G^{*}(x)\exp(\beta x)dx$, which leads to a contradiction. Hence we have $G^*(x)\leq F(x), \forall x\in[a,b]$. Define 
	\[      \tilde{G}(x)\triangleq \prt{\boldsymbol{\rP^{1}_{c}} F}(x)=
	  \begin{cases}
F(g^+(c))-\frac{c-S^+(g^+(c))}{b-g^+(c)}, x\in[g^+(c),b), \\
F(x), \text{otherwise}.
\end{cases} 
	\]
	Notice that the (new) minimization problem has a linear objective and convex ball constraint. Thus, the optimal solution exists in the boundary of the ball constraint. It suffices to consider the following optimization problem
	\begin{equation*}
	\begin{array}{rrclcl}
	\displaystyle \min_{G} & \multicolumn{3}{l}{\int_a^b G(x)u^{\prime}(x)dx}\\
	\textrm{s.t.} &\norm{F-G}_1=c,\\
		&G \succeq F.
	\end{array}
	\end{equation*}
	It is easy to check that $\tilde{G}$ satisfies the constraints. It remains to show that  $\int_a^b \tilde{G}(x)u^{\prime}(x)dx\leq \int_a^b G(x)u^{\prime}(x)dx$ for any feasible $G$. Consider a feasible $G\neq \tilde{G}$. It is obvious that $G(b^-)\ge \tilde{G}(g^+(c))$, otherwise $\norm{F-G}_1=\int_a^b F(x)-G(x)dx>\int_a^b F(x)-\tilde{G}(x)dx=c$, which contradicts with $G \in B_1(F,c)$. If $G(b^-)=\tilde{G}(g^+(c))$, then we again have $G=\tilde{G}$. Otherwise $\norm{F-G}_1>c$. Therefore it holds that  $G(b^-)>\tilde{G}(g^+(c))$. We have $G(g^+(c))<\tilde{G}(g^+(c))$, since otherwise $\norm{F-G}_1<\norm{F-\tilde{G}}_1=c$. Since $G$ is monotonically increasing and right continuous, there exists $g^{\prime}(c)\in(g^+(c),b)$ such that $G(x)<\tilde{G}(x)=\tilde{G}(g^+(c))$ for $x\in[g^+(c),g^{\prime}(c))$ and $G(x)>\tilde{G}(x)=\tilde{G}(g^+(c))$ for $x\in(g^{\prime}(c),b)$. Moreover, $G(x)\leq\tilde{G}(x)=F(x)$ for $x\in[a,g^+(c))$.	It follows that
	\begin{align*}
	\int_a^b G(x)u^{\prime}(x)dx -  \int_a^b \tilde{G}(x)u^{\prime}(x)dx
	&=\int_{a}^{g^{\prime}(c)}(G(x)-\tilde{G}(x))u^{\prime}(x)dx+\int_{g^{\prime}(c)}^b (G(x)-\tilde{G}(x))u^{\prime}(x)dx \\
	&\ge u^{\prime}(g^{\prime}(c))\int_{g^{\prime}(c)}^b (G(x)-\tilde{G}(x))dx-u^{\prime}(g^{\prime}(c))\int_a^{g^{\prime}(c)} (\tilde{G}(x)-G(x))dx\\
	&=0.
	\end{align*}
	 The last equality follows from that
	$\norm{F-G}_1-\norm{F-\tilde{G}}_1=\int_a^b F(x)-G(x)dx-\int_a^b F(x)-\tilde{G}(x)dx=\int_a^b \tilde{G}(x)-G(x)dx=0$.
\end{proof}
Similarly, we can prove that the optimal solution to the minimization problem is given by
\[      \prt{\boldsymbol{\rN^{1}_{c}} F^{-1}}(y)\triangleq 
\begin{cases}
F^{-1}(g^-(c))+\frac{S^{-1}(g^{-1}(c))-c}{1-g^-(c)}, y\in(g^-(c),1], \\
F^{-1}(y), \text{otherwise}.
\end{cases} 
\]
\subsection{Proof of Proposition \ref{prop:opt_was_inv}}
We only provide proof for SRM and DRM because CVaR is a special case of SRM or DRM.
\subsubsection{Proof for SRM}
\begin{proof}
	The objective function is given by
	\[	\boldsymbol{M_{\phi}}(G)= \int_0^1\phi(y)G^{-1}(y)dy,\]
	which is linear in $G^{-1}$. Meanwhile, the constraint $\norm{G^{-1}-F^{-1}}_1\leq c$ is also a convex ball constraint w.r.t. $G^{-1}$. Furthermore, $\phi(y)$ is an increasing function. Using analogous arguments to the proof of Theorem \ref{thm:opt_was} completes the proof.
\end{proof}
\subsubsection{Proof for DRM}
\begin{proof}
	By a change of variable $y=F(x)$, the DRM can be represented as
	\begin{align*}
	\boldsymbol{\rho_g}(G)=\int_0^1 g(1-y)dG^{-1}(y)&=g(1-y)F^{-1}(y)|_0^1-\int_0^1 G^{-1}(y)dg(1-y)\\
	&=-a+\int_0^1 G^{-1}(y)g^{\prime}(1-y)dy.
	\end{align*}
	Again, the objective function is  linear in $G^{-1}$, and the constraint  is also a convex ball constraint. Besides, $g^{\prime}(1-y)$ is increasing in $y$ since $g$ is concave. Using analogous arguments to the proof of Theorem \ref{thm:opt_was} completes the proof.
\end{proof}

\subsection{Proof of Proposition \ref{prop:opt_inf}}
It is easy to verify that $\boldsymbol{\rP^{\infty}_{c}} F\in B_{\infty}(F,c)$. Consider an arbitrary $F\in\sD([a,b])$ and $c>0$. For $x\in[a,b)$, we have 
\[  \prt{\boldsymbol{\rP^{\infty}_{c}} F}(x) = \max\{F(x)-c,0\}\leq G(x),\ \forall G \in B_{\infty}(F,c).  \]
Besides, $\prt{\boldsymbol{\rP^{\infty}_{c}} F}(b)=G(b)=1$ for any  $G \in B_{\infty}(F,c)$. Therefore $G \preceq \boldsymbol{\rP^{\infty}_{c}} F$ for any $G\in B_{\infty}(F,c)$. The result follows from the monotonicity of $\rT$.

\section{Derivations of Results in Section \ref{sec:imp}}
\label{app:imp}
\subsection{Identification of $\boldsymbol{\nu}$}
\begin{table}[ht]
	\caption{$\boldsymbol{\upsilon}(F)$}	
	\label{tab:func}
	\centering
	\begin{tabular}{ lcr  }
		\toprule
	  RM &$\boldsymbol{\upsilon}(F)$    \\
		\midrule
		CVaR &$\frac{1}{\alpha}\norm{\mathbb{I}\{F(\cdot)\ge1-\alpha\}}_q$  \\ [0.5ex] 
        SRM &$\norm{\phi(F)}_q$  \\ [0.5ex] 
        DRM &$\norm{g^{\prime}(1-F)}_{q}$  \\ [0.5ex] 
        ERM &$\frac{\norm{\exp(\beta\cdot)}_q}{\int_a^bu(x)dF(x)}$  \\ [0.5ex] 
        CE & $\norm{u^{\prime}}_q(u^{-1})^{\prime}(\int_a^bu(x)dF(x))$ \\ [0.5ex] 
        RDEU &$\norm{w^{\prime}(F)v^{\prime}}_q$\\  
		\bottomrule
	\end{tabular}
\end{table}
We list functional $\boldsymbol{\upsilon}$ for different risk measures in Table \ref{tab:func}. The  functional $\boldsymbol{\upsilon}$ is crucial to compute the LLC and to show the tightness of our method. We provide detailed derivations of $\boldsymbol{\upsilon}$ in the following. Since the CVaR (ERM) is a special case of SRM (CE), we omit the derivation of  CVaR and ERM. Recall that  $\boldsymbol{\upsilon}((1-t)F+t G,p)$ is a functional satisfying that for any $F,G$
\[  \psi^{\prime}\prt{t;F,G}\leq 	\boldsymbol{\upsilon}((1-t)F+t G,p)\norm{F-G}_p.	 \]
\subsubsection{SRM}
Consider $\boldsymbol{M_{\phi}}$ in the form of
\[ \boldsymbol{M_{\phi}}(F)=\int_a^b\phi(F(x))x dF(x),\]
where $\phi$ is increasing and integrates to 1. Hence $\psi$ is continuously differentiable with derivative 
\begin{align*}
    \psi^{\prime}(t;F,G)&=\frac{d}{d t}\int_a^b\phi((1-t) F+t G)(x))x d((1-t) F+t G)(x)\\
    &=\frac{d}{d t}\brk{\int_a^b\phi((1-t) F+t G)(x))x dF(x)+ t \int_a^b\phi((1-t) F+t G)(x))x d(G-F)(x)}\\
    &=\underbrace{\int_a^b\phi^{\prime}((1-t) F+t G)(x))(G(x)-F(x))x dF(x)}_{(a)}+
    \underbrace{\int_a^b\phi((1-t) F+t G)(x))x d(G-F)(x)}_{(b)} \\
    &+ \underbrace{t\int_a^b\phi^{\prime}((1-t) F+t G)(x))(G(x)-F(x))x d(G-F)(x)}_{(c)}.
\end{align*} 
Since
\begin{align*}
    (b)&=\phi((1-t) F+t G)(x))x(G-F)(x)|_a^b-\int_a^b(G-F)(x) d\brk{\phi((1-t) F+t G)(x))x}\\
    &= -\int_a^b(G-F)(x) d\brk{\phi((1-t) F+t G)(x))x}
\end{align*}
and 
\begin{align*}
    (c)=\int_a^b(G-F)(x)x d\phi((1-t) F+t G)(x))-\int_a^b\phi^{\prime}((1-t) F+t G)(x))(G(x)-F(x))x dF(x)
\end{align*}
We have 
\begin{align*}
    \psi^{\prime}(t)=(a)+(b)+(c)&=-\int_a^b(G-F)(x)\phi((1-t) F+t G)(x))dx=\langle G-F, -\phi((1-t) F+t G) \rangle\\
    &\leq \norm{G-F}_p \norm{\phi((1-t) F+t G)}_q.
\end{align*}
Hence we can choose $\boldsymbol{\upsilon}(F,p)=\norm{\phi(F)}_q$.
\subsubsection{DRM}
A distortion risk measure associated with  distortion function $g$ for a distribution $F$ is 
\[ \boldsymbol{\rho_g}(F)=\int_a^b g(1-F(x))dx,\]
where $g:[0,1]\rightarrow[0,1]$ is a non-decreasing function with $g(0)=0$ and $g(1)=1$. Thus $g^{\prime}$ is non-negative. $\psi$ is continuously differentiable with derivative 
\begin{align*}
    \psi^{\prime}(t;F,G)&= \frac{d}{d t}\int_a^b g(1-(1-t)F(x)-t G(x))dx = \int_a^b g^{\prime}(1-(1-t)F(x)-t G(x))(F(x)-G(x))dx\\
    &= \langle F-G, g^{\prime}(1-(1-t)F-t G) \rangle\\
    &\leq \norm{F-G}_p \norm{g^{\prime}(1-(1-t)F-t G)}_q.
\end{align*}
Hence we can choose $\boldsymbol{\upsilon}(F,p)=\norm{g^{\prime}(1-F)}_q$.
\subsubsection{CE}
For a CE $C_u$, we define $\boldsymbol{E_{u}}(F)\triangleq\int u(x)dF(x)$. Notice that $\boldsymbol{E_u}$ is a \emph{linear} functional, i.e.,
\[  \boldsymbol{E_u}((1-t)F+t G)=(1-t)\boldsymbol{E_u}(F)+t \boldsymbol{E_u}(G)\]
for any $F,G\in D([a,b])$. The function $\psi$ for $\boldsymbol{E_u}$ satisfies
\begin{align*}
    \psi^{\prime}(t;F,G)&=\frac{d}{dt}\boldsymbol{E_u}((1-t)F+tG)=\frac{d}{dt}(1-t)\boldsymbol{E_u}(F)+t \boldsymbol{E_u}(G)=\boldsymbol{E_u}(G)-\boldsymbol{E_u}(F)\\
    &=\int_a^b u(x)dG(x)-\int_a^b u(x)dF(x)\\
    &=u(x)G(x)|_a^b-\int_a^b G(x)du(x)-u(x)F(x)|_a^b+\int_a^b F(x)du(x)\\
    &=\int_a^b (F-G)(x)du(x)\\
    &=\langle F-G, u^{\prime}\rangle,
\end{align*}
where the last equality follows from that $F(b)=G(b)=1$ and $F(a)=G(a)=0$. The certainty equivalent of a distribution $F$ with utility function $u$ is defined as $C_u(F)=u^{-1}(E_u(F))$. It follows that 
\begin{align*}
    \psi^{\prime}(t;F,G)&= \prt{u^{-1}}^{\prime}(\boldsymbol{E_u}((1-t)F+t G))\cdot \langle F-G, u^{\prime}\rangle \leq \prt{u^{-1}}^{\prime}(\boldsymbol{E_u}((1-t)F+t G))\cdot \norm{ F-G}_p \cdot \norm{u^{\prime}}_q.
\end{align*}
Hence we can choose $\boldsymbol{\upsilon}(F,p)=\prt{u^{-1}}^{\prime}(\boldsymbol{E_u}(F)) \norm{u^{\prime}}_q$.
\subsubsection{RDEU}
Let $w:[0,1]\rightarrow[0,1]$ be an \emph{increasing} weight function such that $w(0)=0$ and $w(1)=1$. Let $v:\mathbb{R}\rightarrow\mathbb{R}$ be an (unbounded) \emph{increasing differentiable} function with $u(0)=0$. The RDEU value of $F\in\mathscr{D}([a,b])$ is given by
\begin{align*}
    \boldsymbol{V}(F)&= \int_{a}^{b}v(x)dw(F(x))=v(x)w(F(x))|_a^b-\int_{a}^{b}w(F(x))dv(x)=v(b)-\int_{a}^{b}w(F(x))v^{\prime}(x)dx.
\end{align*}
We have 
\begin{align*}
    \psi^{\prime}(t;F,G)&=\frac{d}{dt} \brk{v(b)-\int_{a}^{b}w((1-t)F(x)+tG(x))v^{\prime}(x)dx}\\
    &=-\int_{a}^{b}w^{\prime}((1-t)F(x)+tG(x))(G(x)-F(x))v^{\prime}(x)dx\\
    &=\langle F-G, w^{\prime}((1-t)F+t G)v^{\prime} \rangle\\
    &\leq \norm{F-G}_p \norm{w^{\prime}((1-t)F+t G)v^{\prime}}_q
\end{align*}
Hence we can choose $\boldsymbol{\upsilon}(F,p)=\norm{w^{\prime}(F)v^{\prime}}_q$.

\subsection{Derivation of the LLC}
In Section \ref{subsec:local_lip}, we have shown that
\begin{align*}
    L_p(\rT;F_n,c_n^p)&=\sup_{F,G\in B_p(F_n,c_n^p) } \frac{\rT(F)-\rT(G)}{\norm{F-G}_p} \\
    &= \sup_{F,G\in B_p(F_n,c_n^p) } \frac{\psi(1;F,G)-\psi(0;F,G)}{\norm{F-G}_p}\\
    & \leq \sup_{F,G\in B_p(F_n,c_n^p), t\in[0,1] } \frac{\psi^{\prime}(t;F,G)}{\norm{F-G}_p} \\
    &\leq \sup_{F,G\in B_p(F_n,c_n^p), t\in[0,1] } \boldsymbol{\upsilon}((1-t)F+t G,p) \\
    &= \sup_{F\in B_p(F_n,c_n^p)} \boldsymbol{\upsilon}(F,p).
\end{align*}
We have derived the functional $\boldsymbol{\upsilon}$ in the previous subsection, which enables us to obtain an upper bound on $L_p(\rT;F,c)$. We only provide the derivations for DRM, CE, and RDEU here because SRM is presented as an illustrating example in Section \ref{subsec:local_lip}. We first give a useful fact to deal with the Wasserstein distance.
\begin{fact}
\label{fct:apprx_ball}
Fix  real numbers $a<b$ For any $G\in\mathscr{D}([a,b])$ and any $c$, there exists a continuous and monotonically  increasing CDF $F\in\mathscr{D}([a,b])$ such that $\norm{F-G}_1\leq c$.
\end{fact}
\subsubsection{DRM}
For DRM, $\boldsymbol{\upsilon}(F,p)=\norm{g^{\prime}(1-F)}_q$. Since  $g$ is concave,  $g^{\prime}$ is monotonically decreasing.
\paragraph{Case $p=1$.} 
By Fact \ref{fct:apprx_ball}, there exists a continuous CDF $F\in B_1(F_n,r_n^1)$. Such $F$ attains all possible value in $[0,1]$, hence
\[ \max_{F\in B_1(F_n,c_n^1)}\norm{g^{\prime}(1-F)}_{\infty}=\max_{F\in B_1(F_n,c_n^1)}\max_{x\in\mathbb{R}}g^{\prime}(1-F(x))=
\max_{y\in[0,1]}g^{\prime}(y)=\norm{g^{\prime}}_{\infty}.\]

\paragraph{Case $p=\infty$.} Since
\[ \max_{F\in B_{\infty}(F_n,c_n^{\infty})}\norm{g^{\prime}(1-F)}_1=\max_{F\in B_{\infty}(F_n,c_n^{\infty})}\int_a^b g^{\prime}(1-F(x))dx,  \]
it follows that
\[ \max_{F\in B_{\infty}(F_n,c_n^{\infty})}\int_a^b g^{\prime}(1-F(x))dx=\norm{g^{\prime}(1-\boldsymbol{\rN^{\infty}_{c_n^{\infty}}}F_n)}_1 \]

\subsubsection{CE}
For CE, $\boldsymbol{\upsilon}(F,p)=\prt{u^{-1}}^{\prime}(\boldsymbol{E_u}(F)) \norm{u^{\prime}}_q$. Since $u$ is convex, $\prt{u^{-1}}^{\prime}$ is a decreasing function. It follows that
\begin{align*}
    \sup_{F\in B_p(F_n,c_n^p)} \prt{u^{-1}}^{\prime}(\boldsymbol{E_u}(F)) \norm{u^{\prime}}_q = (u^{-1})^{\prime}(\boldsymbol{E_u}(\boldsymbol{\rN_{c_n^p}^p} F_n))\norm{u^{\prime}}_q.
\end{align*}

\subsubsection{RDEU}
For RDEU, $\boldsymbol{\upsilon}(F,p)=\norm{w^{\prime}(F)v^{\prime}}_q$. Recall that $w:[0,1]\rightarrow[0,1]$ is an increasing weight function with $w(0)=0$ and $w(1)=1$, and $v:\mathbb{R}\rightarrow\mathbb{R}$ is an (unbounded) increasing differentiable function with $v(0)=0$. We further assume that $w$ is convex so that $w^{\prime}$ is an increasing function.
\paragraph{Case $p=1$.} 
By Fact \ref{fct:apprx_ball}, there exists a continuous CDF $F\in B_1(F_n,r_n^1)$. Such $F$ attains all possible value in $[0,1]$, hence
\[ \max_{F\in B_1(F_n,c_n^1)}\norm{w^{\prime}(F)v^{\prime}}_{\infty}=\max_{F\in B_1(F_n,c_n^1)}\max_{x\in\mathbb{R}}w^{\prime}(F(x))v^{\prime}(x),\]
which might not admit closed form in general. Meanwhile, the GLC (holds without monotonicity)
\[  L_{1}=\max_{F\in\mathscr{D}([a,b])}\norm{w^{\prime}(F)v^{\prime}}_{\infty}=\norm{w^{\prime}}_{\infty}\norm{v^{\prime}}_{\infty}.\]
\paragraph{Case $p=\infty$.} The following holds
\begin{align*}
    \max_{F\in B_{\infty}(F_n,c_n^{\infty})}\norm{w^{\prime}(F)v^{\prime}}_{1}&=\max_{F\in B_{\infty}(F_n,c_n^{\infty})}\int_a^b w^{\prime}(F(x))v^{\prime}(x)dx=\norm{w^{\prime}(\boldsymbol{\rN_{c_n^{\infty}}^{\infty}}F_n)v^{\prime}}_1.
\end{align*}
The GLC (holds without monotonicity) is given by
\[  L_{\infty}=\max_{F\in\mathscr{D}([a,b])}\norm{w^{\prime}(F)v^{\prime}}_1=\norm{w^{\prime}}_{\infty}\norm{v^{\prime}}_1.\]
The last equality uses that for any $d\in[0,1]$, $F(x)=d$ for any $x\in(a,b)$ if $F=d\psi_a+(1-d)\psi_b$ (scaled Bernoulli distribution).
\subsection{Improved Confidence Bound}
In Section \ref{subsec:imp}, we established that 
\begin{align*}
\rT\prt{\boldsymbol{\boldsymbol{\rP_{c_n^p}^p}} F_n}-\rT(F_n)&=\psi\prt{1;F_n,\boldsymbol{\rP_{c_n^p}^p} F_n}-\psi\prt{0;F_n,\boldsymbol{\rP_{c_n^p}^p} F_n} \leq \max_{t\in[0,1]}\psi^{\prime}\prt{t;F_n,\boldsymbol{\rP_{c_n^p}^p} F_n}\\
&\leq \max_{t\in[0,1]}\boldsymbol{\upsilon}\prt{(1-t)F_n+t \boldsymbol{\rP_{c_n^p}^p} F_n,p} c_n^p
\end{align*}
as well as
\begin{align*}
\rT(F_n)-\rT\prt{\boldsymbol{\rN_{c_n^p}^p} F_n}\leq \max_{t\in[0,1]}\boldsymbol{\upsilon}\prt{(1-t) \boldsymbol{\rN_{c_n^p}^p} F_n+t F_n,p} c_n^p.
\end{align*}
The above inequalities lead to the following bounds. 
\subsubsection{Supremum distance}
\begin{proposition}[CE]
	For $F\in\sD([a,b])$, it holds that (assume $\beta>0$)
	\begin{align*}
	\boldsymbol{C_{u}}(\boldsymbol{\rP_c^{\infty}} F)-\boldsymbol{C_{u}}(F) &\leq \norm{u^{\prime}}_{1}(u^{-1})^{\prime}\prt{\int_a^bu(x)d F(x)}\cdot c \leq \norm{u^{\prime}}_{1}(u^{-1})^{\prime}\prt{\int_a^bu(x)d \boldsymbol{\rN_c^{\infty}}F(x)}\cdot c, \\
	\boldsymbol{C_{u}}(F) - \boldsymbol{C_{u}}(\boldsymbol{\rN_c^{\infty}}F) &\leq \norm{u^{\prime}}_{1}(u^{-1})^{\prime}\prt{\int_a^bu(x)d \boldsymbol{\rN_c^{\infty}}F(x)}\cdot c.
	\end{align*}
\end{proposition}
\begin{corollary}[ERM]
	For $F\in\sD([a,b])$, it holds that (assume $u$ convex)
	\begin{align*}
	\boldsymbol{U_{\beta}}(\boldsymbol{\rP_c^{\infty}}F)-\boldsymbol{U_{\beta}}(F) &\leq \frac{\exp(\beta b)-\exp(\beta a)}{\beta \int_a^b\exp(\beta x)d F(x)}\cdot c\leq \frac{L_{\infty}(\boldsymbol{U_{\beta}})c}{\int_a^b\exp(\beta x)d \boldsymbol{\rN_c^{\infty}} F(x)}, \\
	\boldsymbol{U_{\beta}}(F) - \boldsymbol{U_{\beta}}(\boldsymbol{\rN_c^{\infty}}F) &\leq \frac{L_{\infty}(U_{\beta})c}{\int_a^b\exp(\beta x)d \boldsymbol{\rN_c^{\infty}} F(x)}.
	\end{align*}
\end{corollary}
\begin{proposition}[SRM]
	For $F\in\sD([a,b])$, it holds that (assume $\phi$ increasing)
	\begin{align*}
	\boldsymbol{M_{\phi}}(\boldsymbol{\rP_c^{\infty}}F)-\boldsymbol{M_{\phi}}(F) &\leq \int_a^b\phi(F(x))dx\cdot c\leq\int_a^b\phi(\boldsymbol{\rN_c^{\infty}}F(x))dx\cdot c=L_{\infty}(\boldsymbol{M_{\phi}};F,c)c \\
	\boldsymbol{M_{\phi}}(F) - \boldsymbol{M_{\phi}}(\boldsymbol{\rN_c^{\infty}}F) &\leq\int_a^b\phi(\boldsymbol{\rN_c^{\infty}}F(x))dx\cdot c=L_{\infty}(\boldsymbol{M_{\phi}};F,c)c.
	\end{align*}
\end{proposition}
\begin{proposition}[DRM]
	For $F\in\sD([a,b])$, it holds that (assume $g$ concave)
	\begin{align*}
	\boldsymbol{\rho_{g}}(\boldsymbol{\rP_c^{\infty}}F)-\boldsymbol{\rho_{g}}(F) &\leq \int_a^b g^{\prime}(1-F(x))dx\cdot c\leq\int_a^b g^{\prime}(1-\boldsymbol{\rN_c^{\infty}}F(x))dx\cdot c=L_{\infty}(\boldsymbol{\rho_{g}};F,c)c, \\
	\boldsymbol{\rho_{g}}(F)-\boldsymbol{\rho_{g}}(\boldsymbol{\rN_c^{\infty}}F) &\leq\int_a^b g^{\prime}(1-\boldsymbol{\rN_c^{\infty}}F(x))dx\cdot c=L_{\infty}(\boldsymbol{\rho_{g}};F,c)c.
	\end{align*}
\end{proposition}
\begin{corollary}[CVaR]
\label{cor:cvar_inf}
	For $F\in\sD([a,b])$, it holds that 
	\begin{align*}
	\boldsymbol{C_{\alpha}}(\boldsymbol{\rP_c^{\infty}}F)-\boldsymbol{C_{\alpha}}(F)&\leq \frac{b-F^{-1}(1-\alpha)}{\alpha}c \leq \frac{b-F^{-1}(1-\alpha-c)}{\alpha}c=L_{\infty}(\boldsymbol{C_{\alpha}};F,c)c, \\
	\boldsymbol{C_{\alpha}}(F)-\boldsymbol{C_{\alpha}}(\boldsymbol{\rN_c^{\infty}}F)&\leq \frac{b-F^{-1}(1-\alpha-c)}{\alpha}c=L_{\infty}(\boldsymbol{C_{\alpha}};F,c)c.
	\end{align*}
\end{corollary}
\begin{proposition}[RDEU]
	For $F\in\sD([a,b])$, it holds that (assume $w$ convex)
	\begin{align*}
	\boldsymbol{V}(\boldsymbol{\rP_c^{\infty}}F)-\boldsymbol{V}(F) &\leq \int_a^b w^{\prime}(F(x))v^{\prime}(x)dx\cdot c\leq\int_a^b w^{\prime}(\boldsymbol{\rN_c^{\infty}}F(x))v^{\prime}(x)dx\cdot c=L_{\infty}(\boldsymbol{V};F,c)c, \\
	\boldsymbol{V}(F)-\boldsymbol{V}(\boldsymbol{\rN_c^{\infty}}F) &\leq\int_a^b w^{\prime}(\boldsymbol{\rN_c^{\infty}}F(x))v^{\prime}(x)dx\cdot c.
	\end{align*}
\end{proposition}

\subsubsection{Wasserstein distance}
\begin{proposition}[SRM]
	For $F\in\sD([a,b])$, it holds that 
	\begin{align*}
	\boldsymbol{M_{\phi}}(\boldsymbol{\rP_c^{1}}F)-\boldsymbol{M_{\phi}}(F) &= \int_{g^+(c)}^1 (\boldsymbol{\rP_c^1} F^{-1}(y)-F^{-1}(y))\phi(y)dy\leq\phi(1)c=L_{1}(\boldsymbol{M_{\phi}};F,c)c \\
	\boldsymbol{M_{\phi}}(F) - \boldsymbol{M_{\phi}}(\boldsymbol{\rN_c^{1}}F) &=\int_{g^-(c)}^1 (F^{-1}(y)-\boldsymbol{\rN_c^1} F^{-1}(y))\phi(y)dy\leq\phi(1)c.
	\end{align*}
\end{proposition}
\begin{proposition}[DRM]
	For $F\in\sD([a,b])$, it holds that 
	\begin{align*}
	\boldsymbol{\rho_{g}}(\boldsymbol{\rP_c^{1}}F)-\boldsymbol{\rho_{g}}(F) &= \int_{g^+(c)}^1 (\boldsymbol{\rP_c^1} F^{-1}(y)-F^{-1}(y))g^{\prime}(1-y)dy\leq g^{\prime}(0)c=L_{1}(\boldsymbol{\rho_g};F,c)c \\
	\boldsymbol{\rho_{g}}(F)-\boldsymbol{\rho_{g}}(\boldsymbol{\rN_c^{1}}F) &=\int_{g^-(c)}^1 (F^{-1}(y)-\boldsymbol{\rN_c^1} F^{-1}(y))g^{\prime}(1-y)dy\leq g^{\prime}(0)c.
	\end{align*}
\end{proposition}
\begin{corollary}[CVaR]
	For $F\in\sD([a,b])$, it holds that 
	\begin{align*}
	\boldsymbol{C_{\alpha}}(\boldsymbol{\rP_c^{1}}F)-\boldsymbol{C_{\alpha}}(F)&= \frac{c}{\alpha} = L_{1}(\boldsymbol{C_{\alpha}};F,c)c, \\
	\boldsymbol{C_{\alpha}}(F)-\boldsymbol{C_{\alpha}}(\boldsymbol{\rN_c^{1}}F)&= \frac{c}{\alpha} = L_{1}(\boldsymbol{C_{\alpha}};F,c)c.
	\end{align*}
\end{corollary}
\begin{proposition}[CE]
	For $F\in\sD([a,b])$, it holds that 
	\begin{align*}
	\boldsymbol{C_{u}}(\boldsymbol{\rP_c^{1}}F)-\boldsymbol{C_{u}}(F) &\leq \norm{u^{\prime}}_{\infty}(u^{-1})^{\prime}\prt{\int_a^bu(x)d F(x)}\cdot c \leq \norm{u^{\prime}}_{\infty}(u^{-1})^{\prime}\prt{\int_a^bu(x)d \boldsymbol{\rN_c^{1}}F(x)}\cdot c, \\
	\boldsymbol{C_{u}}(F) - \boldsymbol{C_{u}}(\boldsymbol{\rN_c^{1}}F) &\leq \norm{u^{\prime}}_{\infty}(u^{-1})^{\prime}\prt{\int_a^bu(x)d \boldsymbol{\rN_c^{1}}F(x)}\cdot c.
	\end{align*}
\end{proposition}
\begin{corollary}[ERM]
	For $F\in\sD([a,b])$, it holds that (assume $\beta>0$)
	\begin{align*}
	\boldsymbol{U_{\beta}}(\boldsymbol{\rP_c^{1}}F)-\boldsymbol{U_{\beta}}(F) &\leq \frac{\exp(\beta b)}{ \int_a^b\exp(\beta x)d F(x)}\cdot c\leq \frac{L_{1}(\boldsymbol{U_{\beta}})c}{\int_a^b\exp(\beta x)d \boldsymbol{\rN_c^{1}} F(x)}, \\
	\boldsymbol{U_{\beta}}(F) - \boldsymbol{U_{\beta}}(\boldsymbol{\rN_c^{1}}F) &\leq \frac{L_{1}(\boldsymbol{U_{\beta}})c}{\int_a^b\exp(\beta x)d \boldsymbol{\rN_c^{1}} F(x)}.
	\end{align*}
\end{corollary}

\section{Proof of Proposition \ref{prop:reg}}
\label{app:reg}
\begin{proof}
Observe that the regret decomposes as $\text{Regret}(\text{LCB},\nu,N)=\sum_{i=1}^K \Delta_i \mathbb{E}[s_i(N)]$.  Without loss of generality, we assume the first arm is the optimal arm, i.e., 
\[ \boldsymbol{C_{\alpha}}(F_1)\leq \boldsymbol{C_{\alpha}}(F_i),\ \forall i\in[K].\] We will bound $\mathbb{E}[s_i(N)]$ for each suboptimal arm $i\neq1$. Recall that $\underline{F}_{i,t}=\boldsymbol{\rN^{\infty}_{c_i(t)}}\hat{F}_{i,t}$. Denote by $\hat{F}_i^n$ the empirical CDF corresponding to arm $i$ after observing $n$ samples, then we have $\hat{F}_{i,t}=\hat{F}_i^{s_i(t)}$.
Define the good event for arm $i$ as 
\[ \mathcal{G}_i=\{\boldsymbol{C_{\alpha}}(F_1)>\max_{t\in[N]}\boldsymbol{C_{\alpha}}(\underline{F}_{1,t})\}\cap\{\boldsymbol{C_{\alpha}}(\boldsymbol{\rN^{\infty}_{\sqrt{\frac{\log(2/\delta)}{2u_i}}}}\hat{F}_{i}^{u_i})>\boldsymbol{C_{\alpha}}(F_1)\}. \]
We claim that if $\mathcal{G}_i$ occurs, then $s_i(N)\leq u_i$. The proof follows from a contradiction. Suppose $s_i(N)> u_i$, then there exists some round $t\in[N]$ such that $s_i(t)=u_i$ and $I_t=i$. It follows that
\begin{align*}
    \boldsymbol{C_{\alpha}}(\underline{F}_{i,t})&=\boldsymbol{C_{\alpha}}(\boldsymbol{\rN^{\infty}_{c_i(t)}}\hat{F}_{i,t})\\
    &=\boldsymbol{C_{\alpha}}(\boldsymbol{\rN^{\infty}_{\sqrt{\frac{\log(2/\delta)}{2u_i}}}}\hat{F}_{i}^{u_i})\\
    &>\boldsymbol{C_{\alpha}}(F_1)\\
    &>\boldsymbol{C_{\alpha}}(\underline{F}_{1,t}),
\end{align*}
where the inequalities come from the definition of $\mathcal{G}_i$. Hence $I_t=\arg\min_{i\in[K]}\boldsymbol{C_{\alpha}}(\underline{F}_{i,t})\neq i$, which leads to a contradiction. Using the tower property,
\[ \mathbb{E}[s_i(N)]= \mathbb{E}[s_i(N)|\mathcal{G}_i]\mathbb{P}(\mathcal{G}_i)+\mathbb{E}[s_i(N)|\mathcal{G}_i^c]\mathbb{P}(\mathcal{G}_i^c)\leq  u_i+N\mathbb{P}(\mathcal{G}_i^c).  \]
Next, we show that $\mathbb{P}(\mathcal{G}_i^c)$ is small. Using union bound, we have 
$$\mathbb{P}(\mathcal{G}_i^c)\leq \mathbb{P}\prt{\boldsymbol{C_{\alpha}}(F_1)\leq\max_{t\in[N]}\boldsymbol{C_{\alpha}}(\tilde{F}_{1,t})}+\mathbb{P}\prt{\boldsymbol{C_{\alpha}}\boldsymbol{\prt{\rN^{\infty}_{\sqrt{\frac{\log(2/\delta)}{2u_i}}}}\hat{F}_{i}^{u_i}}\leq \boldsymbol{C_{\alpha}}(F_1)}.$$
By Theorem \ref{thm:opt_inf}, for any $t\in[N]$, if $F_1\in B(\hat{F}_{1,t},c_1(t))$ then  $\underline{F}_{1,t}=\boldsymbol{\rN^{\infty}_{c_1(t)}}\hat{F}_{1,t}\succeq F_1$, and $\boldsymbol{C_{\alpha}}(\underline{F}_{1,t})\leq \boldsymbol{C_{\alpha}}(F_1)$.  Hence the first term on the r.h.s. can be bounded as 
\begin{align*}
    \mathbb{P}\prt{\boldsymbol{C_{\alpha}}(F_1)\leq\max_{t\in[N]}\boldsymbol{C_{\alpha}}(\tilde{F}_{1,t})}&=\mathbb{P}\prt{\exists t\in[N]: \boldsymbol{C_{\alpha}}(F_1)\leq \boldsymbol{C_{\alpha}}(\underline{F}_{1,t})}\\
    &\leq \mathbb{P}\prt{\exists t\in[N]: \norm{F_1-\hat{F}_{1,t}}\ge\sqrt{\frac{\log(2/\delta)}{2s_i(t)}}}\\
    &\leq \mathbb{P}\prt{\cup_{s\in[N]} \bigg\{\norm{F_1-\hat{F}_{1}^s}\ge\sqrt{\frac{\log(2/\delta)}{2s}}\bigg\}}\\
    &\leq N\delta,
\end{align*}
where the last inequality follows from a union bound and the DKW inequality. Denote by $c_i\triangleq\sqrt{\frac{\log(2/\delta)}{2u_i}}$. By Corollary \ref{cor:cvar_inf}, we have that $\boldsymbol{C_{\alpha}}(F)-\boldsymbol{C_{\alpha}}( \boldsymbol{\rN^{\infty}_c} F)\leq\frac{b-F^{-1}(1-\alpha-c)}{\alpha}c$. If the event $\Big\{\norm{\hat{F}_{i}^{u_i}-F_i}_{\infty}< c_i\Big\}$ occurs, then  
\[\boldsymbol{C_{\alpha}}(\hat{F}^{u_i}_i)-\boldsymbol{C_{\alpha}}\prt{\boldsymbol{\rN^{\infty}_{c_i}}\hat{F}_{i}^{u_i}}\leq\frac{b-(\hat{F}^{u_i}_i)^{-1}(1-\alpha-c_i)}{\alpha}c_i\]
and
\[ \boldsymbol{C_{\alpha}}(F_i)-\boldsymbol{C_{\alpha}}(\hat{F}^{u_i}_i)\leq \boldsymbol{C_{\alpha}}(\boldsymbol{\rP^{\infty}_{c_i}}\hat{F}^{u_i}_i)- \boldsymbol{C_{\alpha}}(\hat{F}^{u_i}_i)\leq\frac{b-(\boldsymbol{\rP^{\infty}_{c_i}}\hat{F}^{u_i}_i)^{-1}(1-\alpha-c_i)}{\alpha}c_i.  \]
Combining these inequalities,
\begin{align*}
    \boldsymbol{C_{\alpha}}(F_i)-\boldsymbol{C_{\alpha}}\prt{\boldsymbol{\rN^{\infty}_{c_i}}\hat{F}_{i}^{u_i}}&\leq\prt{\frac{b-(\hat{F}^{u_i}_i)^{-1}(1-\alpha-c_i)}{\alpha}+\frac{b-(\boldsymbol{\rP_{c_i}}\hat{F}^{u_i}_i)^{-1}(1-\alpha-c_i)}{\alpha}}c_i\\
    &\leq2\frac{b-(\hat{F}^{u_i}_i)^{-1}(1-\alpha-c_i)}{\alpha}c_i\\
    &\leq2\frac{b- F_i^{-1}(1-\alpha-2c_i)}{\alpha}c_i
\end{align*}
We choose $u_i$ such that $\Delta_i=C_{\alpha}(F_i)-C_{\alpha}(F_1)\ge2\frac{b-F_i^{-1}(1-\alpha-2c_i)}{\alpha}c_i$, then the second term on the r.h.s. can be bounded as 
\begin{align*}
\mathbb{P}\prt{\boldsymbol{C_{\alpha}}\prt{\boldsymbol{\rN^{\infty}_{c_i}}\hat{F}_{i}^{u_i}}\leq \boldsymbol{C_{\alpha}}(F_1)}
    &=\mathbb{P}\prt{\boldsymbol{C_{\alpha}}(F_i)-\boldsymbol{C_{\alpha}}\prt{\boldsymbol{\rN^{\infty}_{c_i}}\hat{F}_{i}^{u_i}}\ge \Delta_i}\\
    &\leq \mathbb{P}\prt{\boldsymbol{C_{\alpha}}(F_i)-\boldsymbol{C_{\alpha}}\prt{\boldsymbol{\rN^{\infty}_{c_i}}\hat{F}_{i}^{u_i}}\ge 2\frac{b-F_i^{-1}(1-\alpha-2c_i)}{\alpha}c_i}\\
    &\leq \mathbb{P}\prt{\norm{\hat{F}_{i}^{u_i}-F_i}_{\infty}\ge c_i}\\
    &\leq\delta.
\end{align*}
Hence, the probability of $\mathcal{G}_i^c$ is bounded as $\mathbb{P}(\mathcal{G}_i^c)\leq (N+1)\delta $. It follows that
\[   \mathbb{E}[s_i(N)]\leq  u_i+N(N+1)\delta. \]
Define $h_i(c)\triangleq2\frac{b-F_i^{-1}(1-\alpha-2c)}{\alpha}c$.  Let $c_i^{*}\triangleq h_i^{-1}(\Delta_i)$ be the solution to the equation
\[    h_i(c)=2\frac{b-F_i^{-1}(1-\alpha-2c)}{\alpha}c=\Delta_i.      \]
Note that $c_i^*$ is a \textbf{distribution-dependent constant}. We let $u_i:=\ceil{\frac{\log(2/\delta)}{2(c_i^*)^2}}$ and let $\delta=\frac{1}{N^2}$:
\[   \mathbb{E}[s_i(N)]\leq  \ceil{\frac{\log(2N^2)}{2(c_i^*)^2}}+2\leq\frac{\log(\sqrt{2}N)}{(c_i^*)^2}+3. \]
Substituting it into the regret decomposition, we get
\begin{align*}
    \sum_{i=1}^K\Delta_i\mathbb{E}[s_i(N)]&\leq \log(\sqrt{2}N)\sum_{i=2}^K\frac{\Delta_i}{(c_i^*)^2}+3\sum_{i=1}^K\Delta_i\\
&=\frac{4\log(\sqrt{2}N)}{\alpha^2}\sum_{i>1}^K\frac{\prt{b-F_i^{-1}(1-\alpha-2c^*_i)}^2}{\Delta_i}+3\sum_{i=1}^K\Delta_i.
\end{align*}

\end{proof} 

\section{Algorithms}
\label{app:alg_op}
We present several comprehensive algorithms that output the confidence bounds for a given risk measure when given $n$  i.i.d. samples and confidence radius as input. Algorithm \ref{alg:wcub} and Algorithm \ref{alg:wclb} compute the UCB and LCB of a risk measure via  Wasserstein distance respectively.  Algorithm \ref{alg:scub} and Algorithm \ref{alg:sclb} compute the UCB and LCB of a risk measure via supremum distance respectively.
\subsection{Time Complexity}
We start with Algorithm \ref{alg:wcub}. The sorting of $n$ samples incurs $\cO(n\log n)$. The for-loop costs $\cO( n)$ since the cost in each iteration is $\cO(1)$. Therefore the total time complexity is $\cO(n\log n+\log n)$. The time complexity of Algorithm \ref{alg:wclb}-Algorithm \ref{alg:sclb} is $\cO(n\log n+\log n)$.
\subsection{Space Complexity}
Consider Algorithm \ref{alg:wcub}. The space complexity of storing the samples is  $\cO( n)$. In addition, storing $S_i$, $p_{n^{\prime}}$, and $p_b$ costs $\cO( n)$. The total space complexity is $\cO( n)$. It is easy to check that the time complexity of Algorithm \ref{alg:wclb}-Algorithm \ref{alg:sclb} is $\cO(n)$.

\begin{algorithm}[tb]
    \caption{\texttt{Wasserstein upper confidence bound}}
    \label{alg:wcub}
    \begin{algorithmic}[1]
        \STATE{Input: $b$, samples $\boldsymbol{X}=X_1, X_2, \cdots, X_n$, risk measure $\rT$, $c>0$}
        \STATE{Sort the $n$ samples in ascent order $X_{(1)}\leq X_{(2)} \cdots \leq X_{(n)}$}
        \STATE{Initialize $S_1=\frac{b-X_{(n)}}{n}$}
        \FOR{$i = 1:n$}
            \IF{$S_i \leq c$}
                \STATE{$S_{i+1} = S_{i}+\frac{1}{n}(b-X_{(n-i)})$}
            \ELSE
                \STATE{$n^{\prime}=n+1-i$}
                \STATE{\textbf{break}}
            \ENDIF
        \ENDFOR
        \STATE{$p_{n^{\prime}}=\frac{S_i-c}{b-X_{(n^{\prime})}}$, $p_b=\frac{i}{n}-p_{n^{\prime}}$}
        \STATE{$\overline{F^1_n}=\frac{1}{n}\sum_{i}^{n^{\prime}-1}\mathbb{I}\{X_{(i)}\leq\cdot\}+p_{n^{\prime}}\mathbb{I}\{X_{(n^{\prime})}\leq\cdot\}+p_b\mathbb{I}\{b\leq \cdot\}$}
        \STATE{Output: $\rT\prt{\overline{F^1_n}}$}
  \end{algorithmic}
\end{algorithm}

\begin{algorithm}[tb]
    \caption{\texttt{Wasserstein lower confidence bound}}
    \label{alg:wclb}
    \begin{algorithmic}[1]
        \STATE{Input: $a$, samples $\boldsymbol{X}=X_1, X_2, \cdots, X_n$, risk measure $\rT$, $c>0$}
        \STATE{Sort the $n$ samples in ascent order $X_{(1)}\leq X_{(2)} \cdots \leq X_{(n)}$}
        \STATE{Initialize $S_1=\frac{X_{(n)}-X_{(n-1)}}{n}$}
        \FOR{$i = 1:n-1$}
            \IF{$S_i \leq c$}
                \STATE{$S_{i+1} = S_{i}+\frac{i+1}{n}(X_{(n-i)}-X_{(n-1-i)})$}
            \ELSE
                \STATE{$n^{\prime}=n+1-i$}
                \STATE{\textbf{break}}
            \ENDIF
        \ENDFOR
        \STATE{$b^-=X_{(n^{\prime}-1)}+\frac{n(S_i-c)}{i}$}
        \STATE{$\underline{F^1_n}=\frac{1}{n}\sum_{i}^{n^{\prime}-1}\mathbb{I}\{X_{i}\leq\cdot\}+\frac{i}{n}\mathbb{I}\{b^-\leq\cdot\}$}
        \STATE{Output: $\rT\prt{\underline{F^1_n}}$}
  \end{algorithmic}
\end{algorithm}

\begin{algorithm}[tb]
    \caption{\texttt{Supremum upper confidence bound}}
    \label{alg:scub}
    \begin{algorithmic}[1]
        \STATE{Input: $b$, samples $\boldsymbol{X}=X_1, X_2, \cdots, X_n$, risk measure $\rT$, $c>0$}
        \STATE{Sort the $n$ samples in ascent order $X_{(1)}\leq X_{(2)} \cdots \leq X_{(n)}$}
        \STATE{Initialize $i=1$}
        \WHILE{$\frac{i}{n}\leq c$}
        \STATE{$i=i+1$}
        \STATE{$l=i$}
        \ENDWHILE{}
        \STATE{$\overline{F^{\infty}_n}=(\frac{l}{n}-c)\mathbb{I}\{X_{(l)}\leq\cdot\}+\frac{1}{n}\sum_{i=l+1}^{n}\mathbb{I}\{X_{(i)}\leq\cdot\}+c\mathbb{I}\{b\leq \cdot\}$}
        \STATE{Output: $\rT\prt{\overline{F^{\infty}_n}}$}
  \end{algorithmic}
\end{algorithm}

\begin{algorithm}[tb]
    \caption{\texttt{Supremum lower confidence bound}}
    \label{alg:sclb}
    \begin{algorithmic}[1]
        \STATE{Input: $a$, samples $\boldsymbol{X}=X_1, X_2, \cdots, X_n$, risk measure $\rT$, $c>0$}
        \STATE{Sort the $n$ samples in ascent order $X_{(1)}\leq X_{(2)} \cdots \leq X_{(n)}$}
        \STATE{Initialize $i=n$}
        \WHILE{$\frac{i}{n}+c\ge1$}
        \STATE{$i=i-1$}
        \STATE{$l=i$}
        \ENDWHILE{}
        \STATE{$\underline{F^{\infty}_n}=c\mathbb{I}\{a\leq \cdot\}+\frac{1}{n}\sum_{i=1}^{l}\mathbb{I}\{X_{(i)}\leq\cdot\}+(1-\frac{l}{n}-c)\mathbb{I}\{X_{(l+1)}\leq\cdot\}$}
        \STATE{Output: $\rT\prt{\underline{F^{\infty}_n}}$}
  \end{algorithmic}
\end{algorithm}

\section{Experiments}
\label{app:num}  
\subsection{Confidence Bounds}
We consider five beta distributions with different parameters. The specific parameters $(A,B)$ is shown above in each figure. Unless otherwise specified, we always use $N=10^5$ samples, $\alpha=0.05$, $\beta=1$, and $\delta=0.05$. For convenience, we use $c^{1}_n =(b-a)c^{\infty}_n$. We plot the CIs for ERM and CVaR for varying sample size and varying risk parameter in Figure 6-13.
\paragraph{Varying sample size.}
Figure 6-7 and Figure 10-11 show how the UCB and LCB of CVaR and ERM change as the number of samples increase. Our bounds remains tighter than the LC-based bounds.
\paragraph{Varying risk parameter.}
Figure 8-9 show how the bounds of CVaR change as $\alpha$ varies, and Figure 12-13 show how the bounds of ERM change as $\beta$ varies. Smaller $\alpha$ and larger $\beta$ represents higher risk sensitivity and induces higher non-linearity and larger LC. These results demonstrate that our bounds achieve more improvement for more risk-sensitive parameter.
\begin{figure*}[ht]
     \centering
     \begin{subfigure}[b]{0.19\textwidth}
         \centering
         \includegraphics[width=\textwidth]{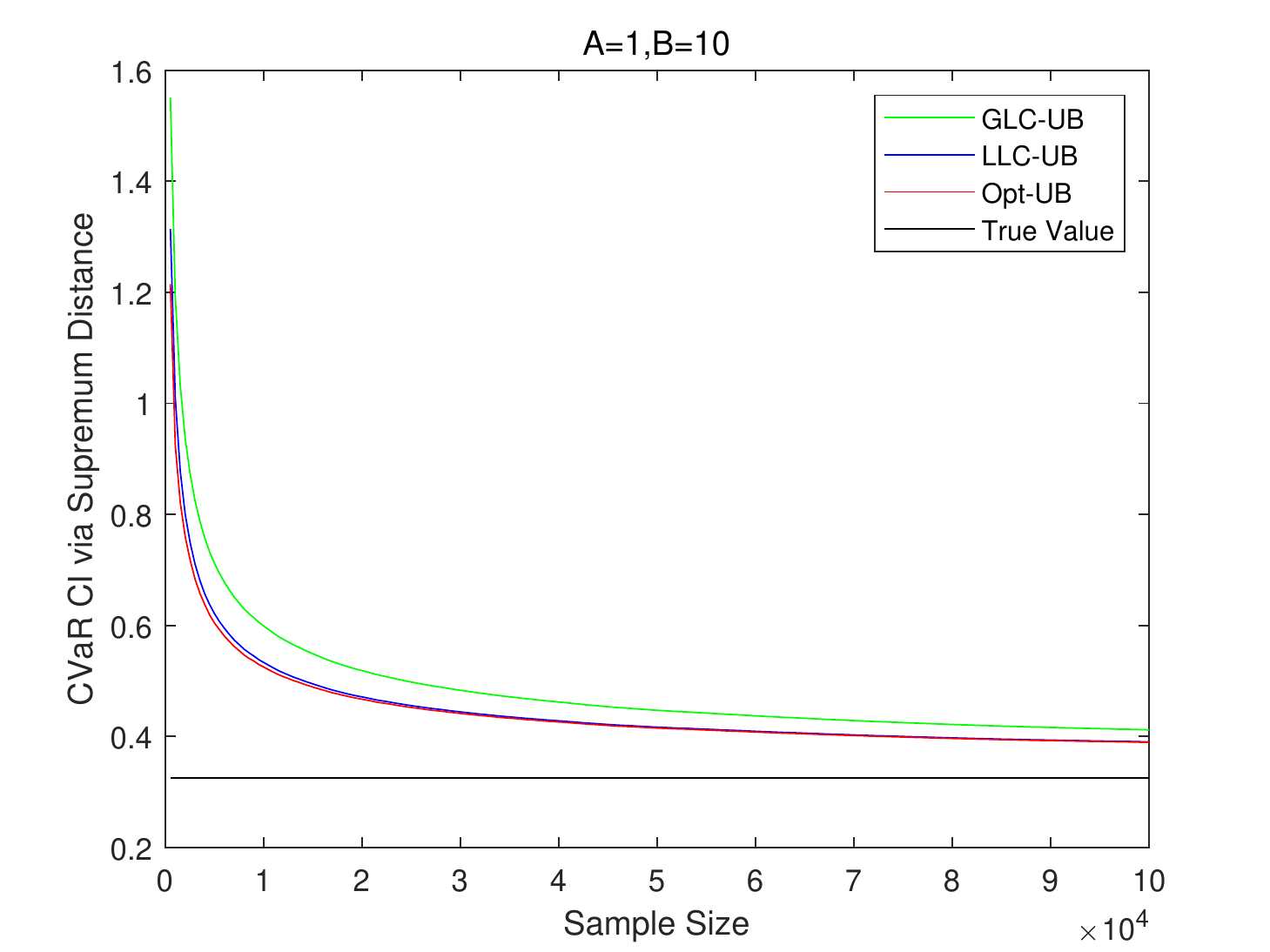}
     \end{subfigure}
     \hfill
     \begin{subfigure}[b]{0.19\textwidth}
         \centering
         \includegraphics[width=\textwidth]{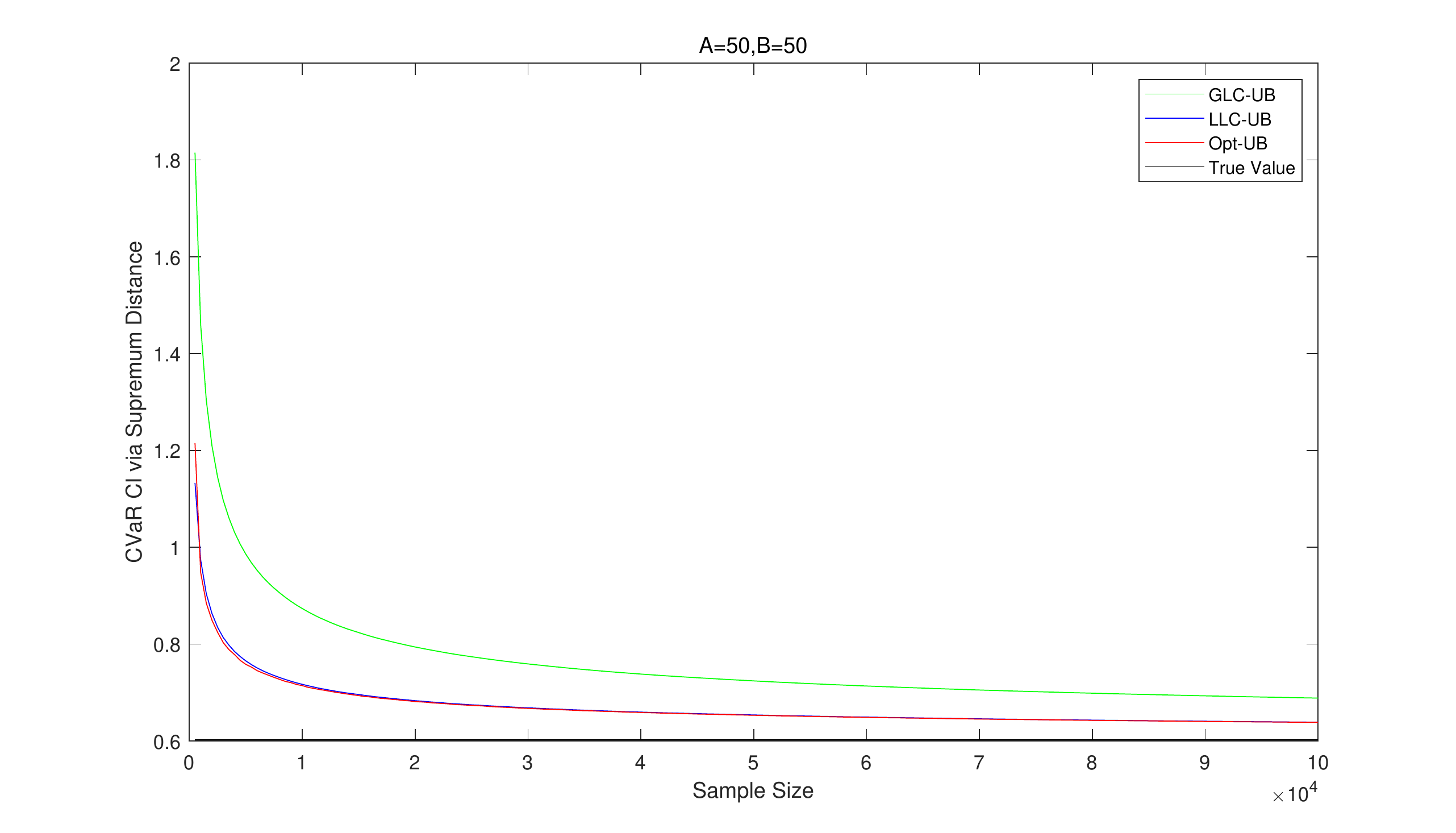}
     \end{subfigure}
     \begin{subfigure}[b]{0.19\textwidth}
         \centering
         \includegraphics[width=\textwidth]{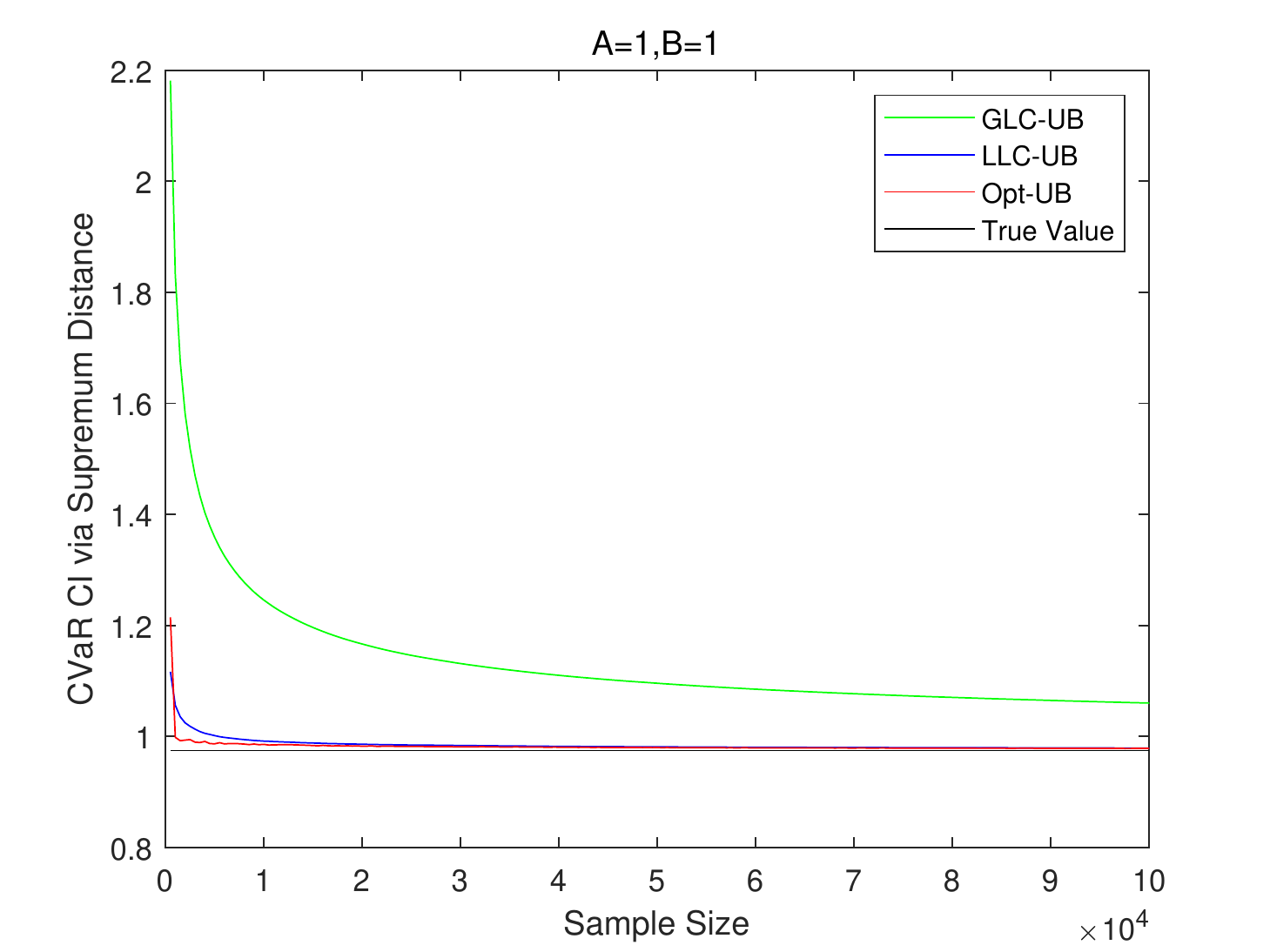}
     \end{subfigure}
     \hfill
     \begin{subfigure}[b]{0.19\textwidth}
         \centering
         \includegraphics[width=\textwidth]{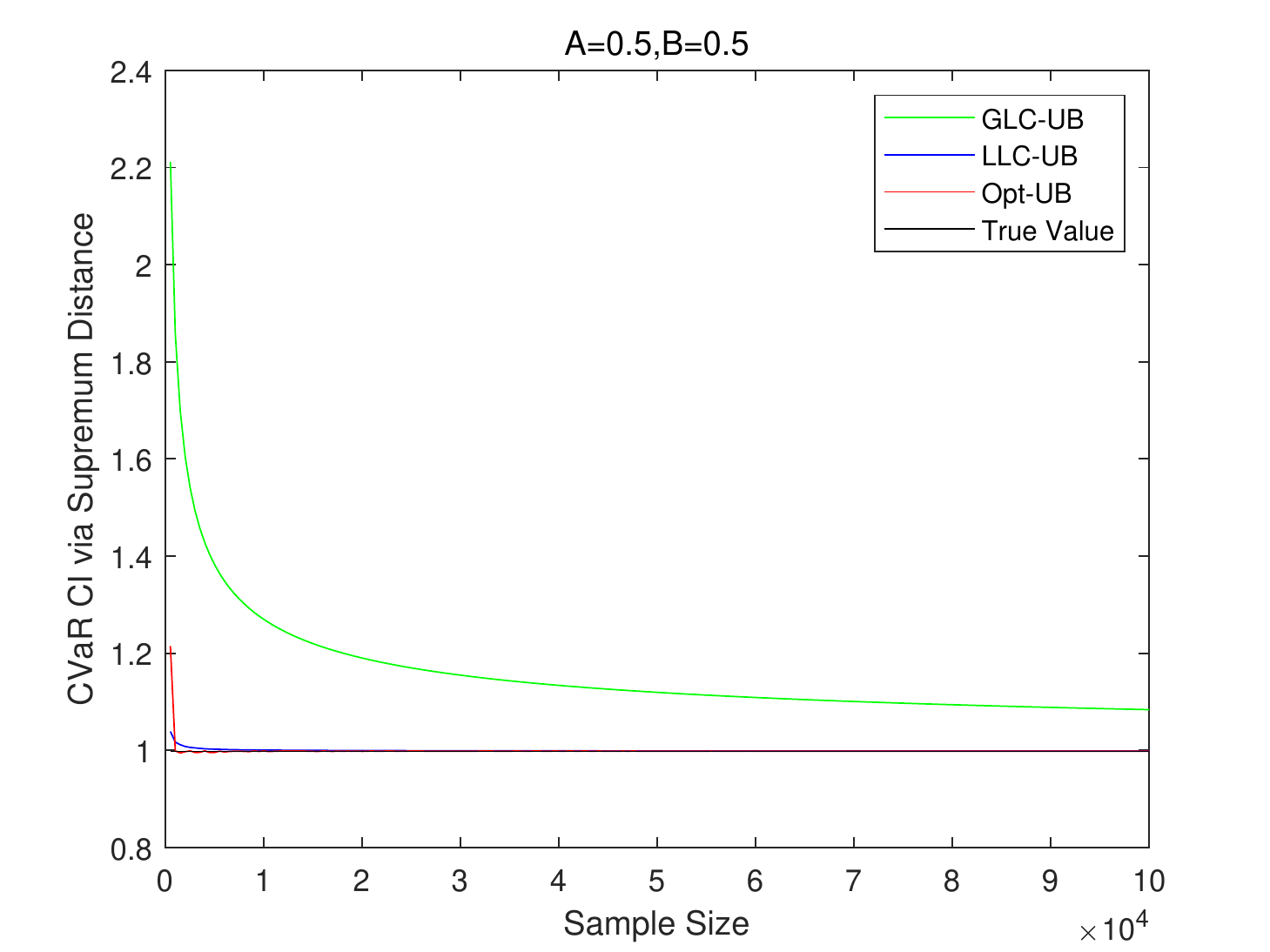}
     \end{subfigure}
     \begin{subfigure}[b]{0.19\textwidth}
         \centering
         \includegraphics[width=\textwidth]{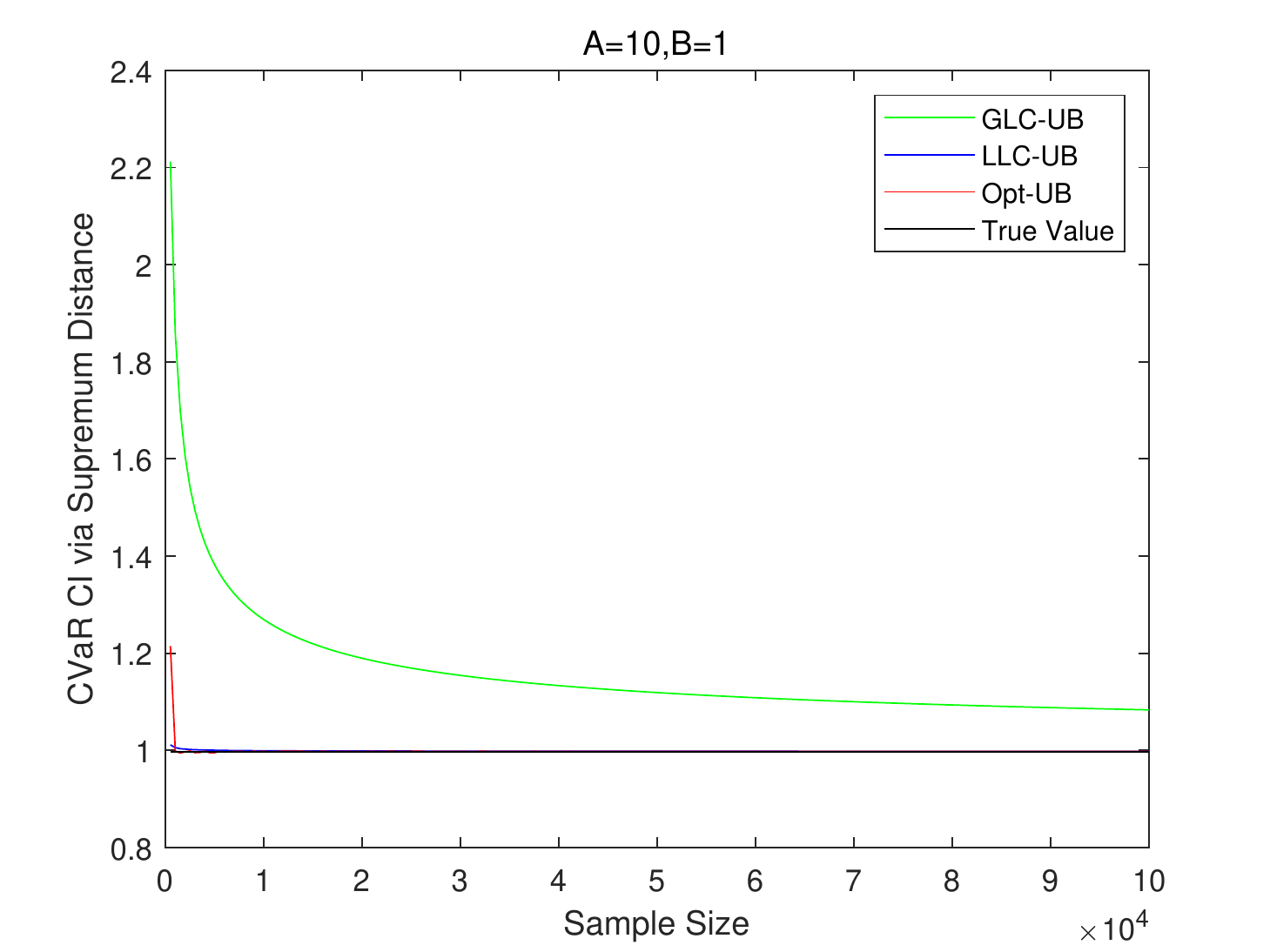}
     \end{subfigure}
        \caption{CVaR UCB with varying sample size}
        \label{fig:cvar_ub_n}
        \vspace{-0ex}
\end{figure*}

\begin{figure*}[ht]
     \centering
     \begin{subfigure}[b]{0.19\textwidth}
         \centering
         \includegraphics[width=\textwidth]{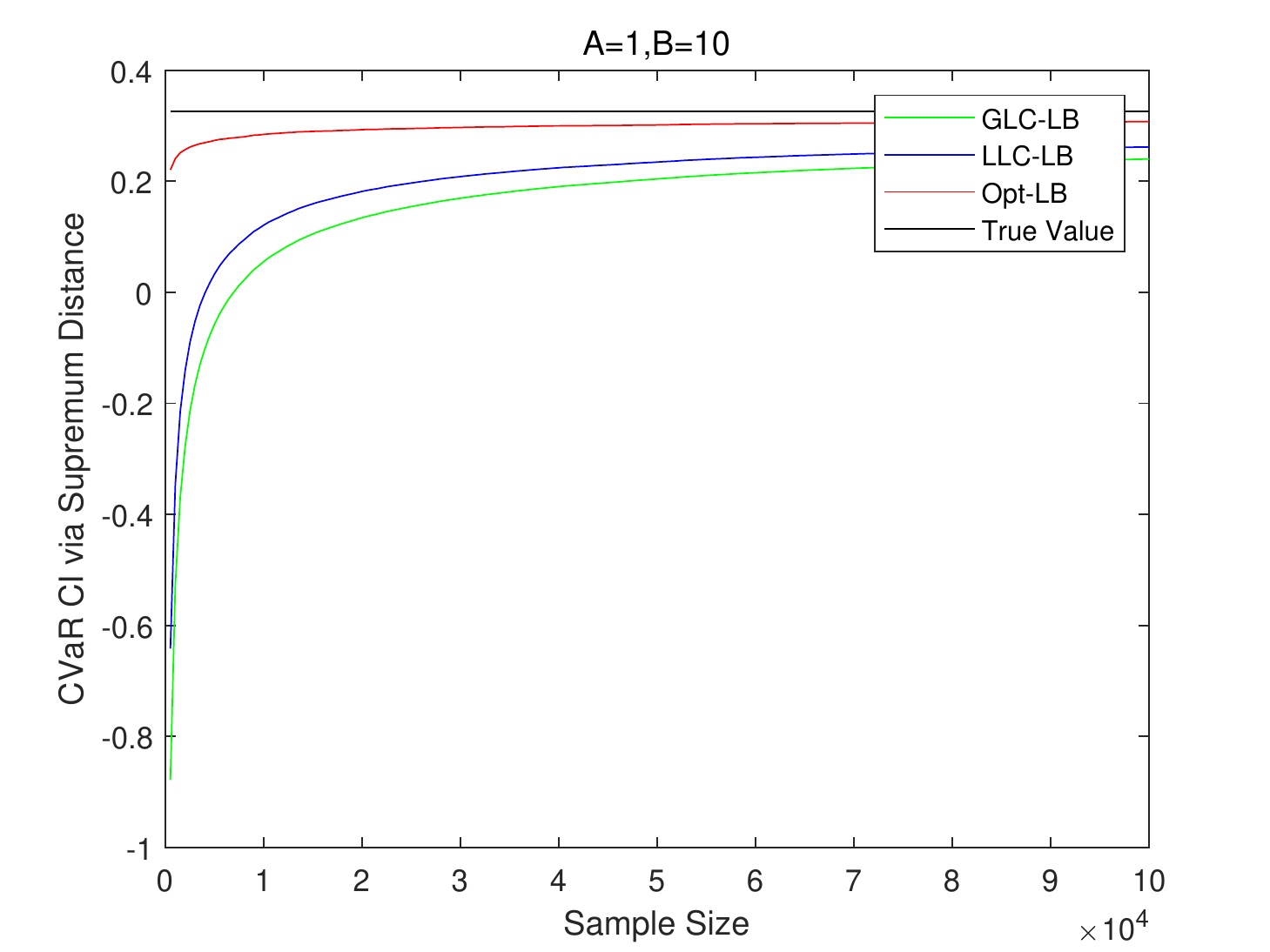}
     \end{subfigure}
     \hfill
     \begin{subfigure}[b]{0.19\textwidth}
         \centering
         \includegraphics[width=\textwidth]{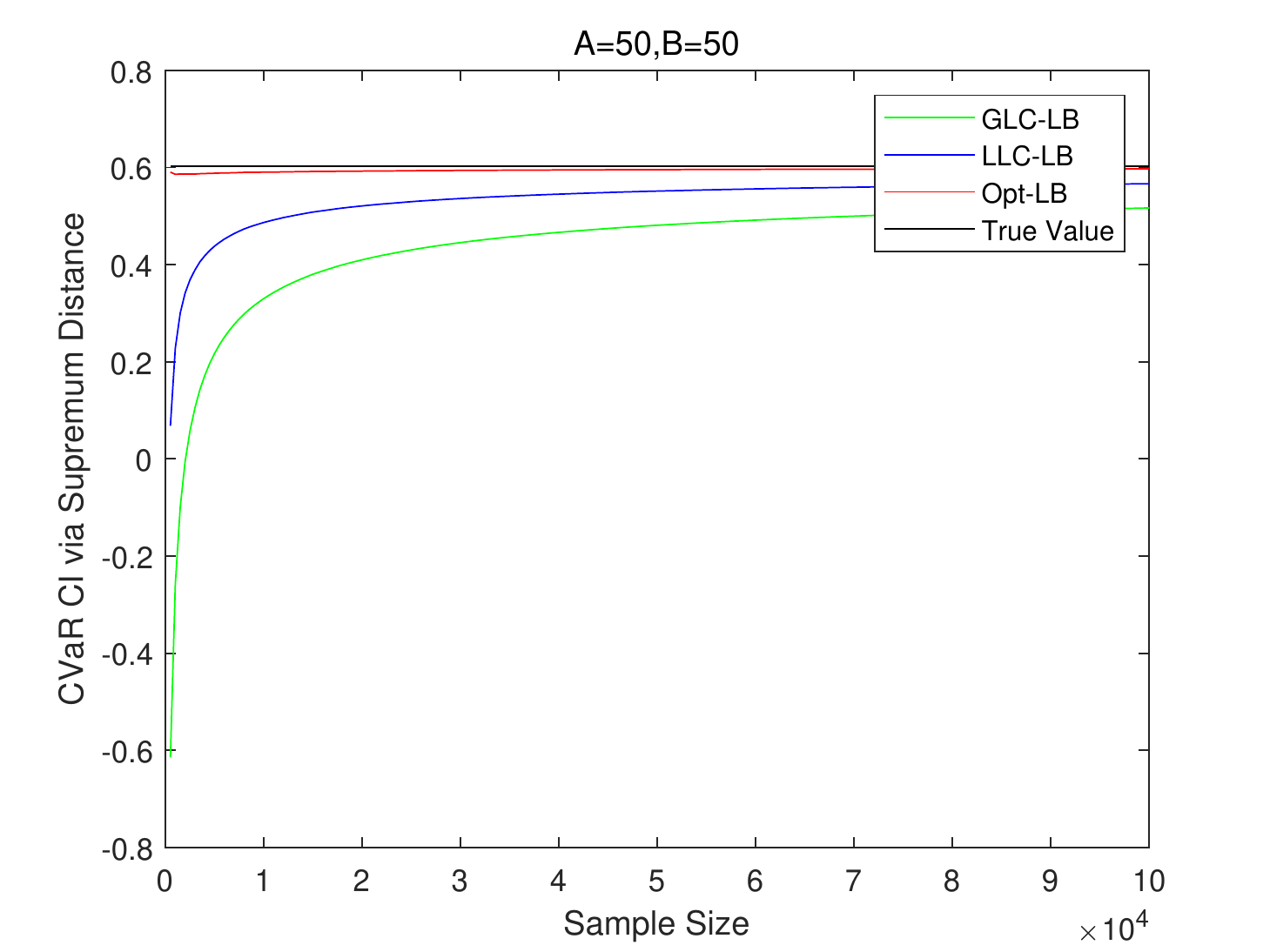}
     \end{subfigure}
     \begin{subfigure}[b]{0.19\textwidth}
         \centering
         \includegraphics[width=\textwidth]{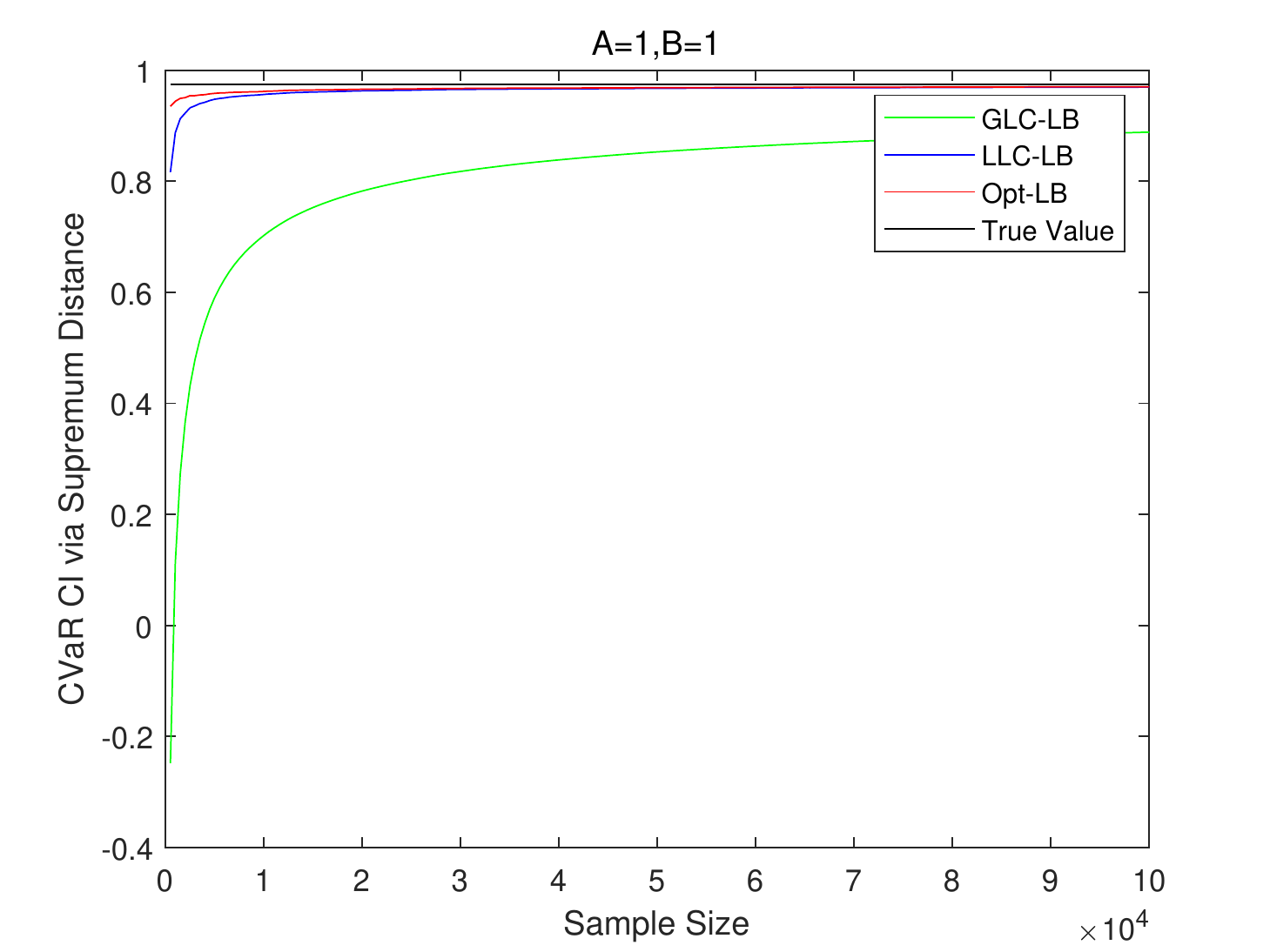}
     \end{subfigure}
     \hfill
     \begin{subfigure}[b]{0.19\textwidth}
         \centering
         \includegraphics[width=\textwidth]{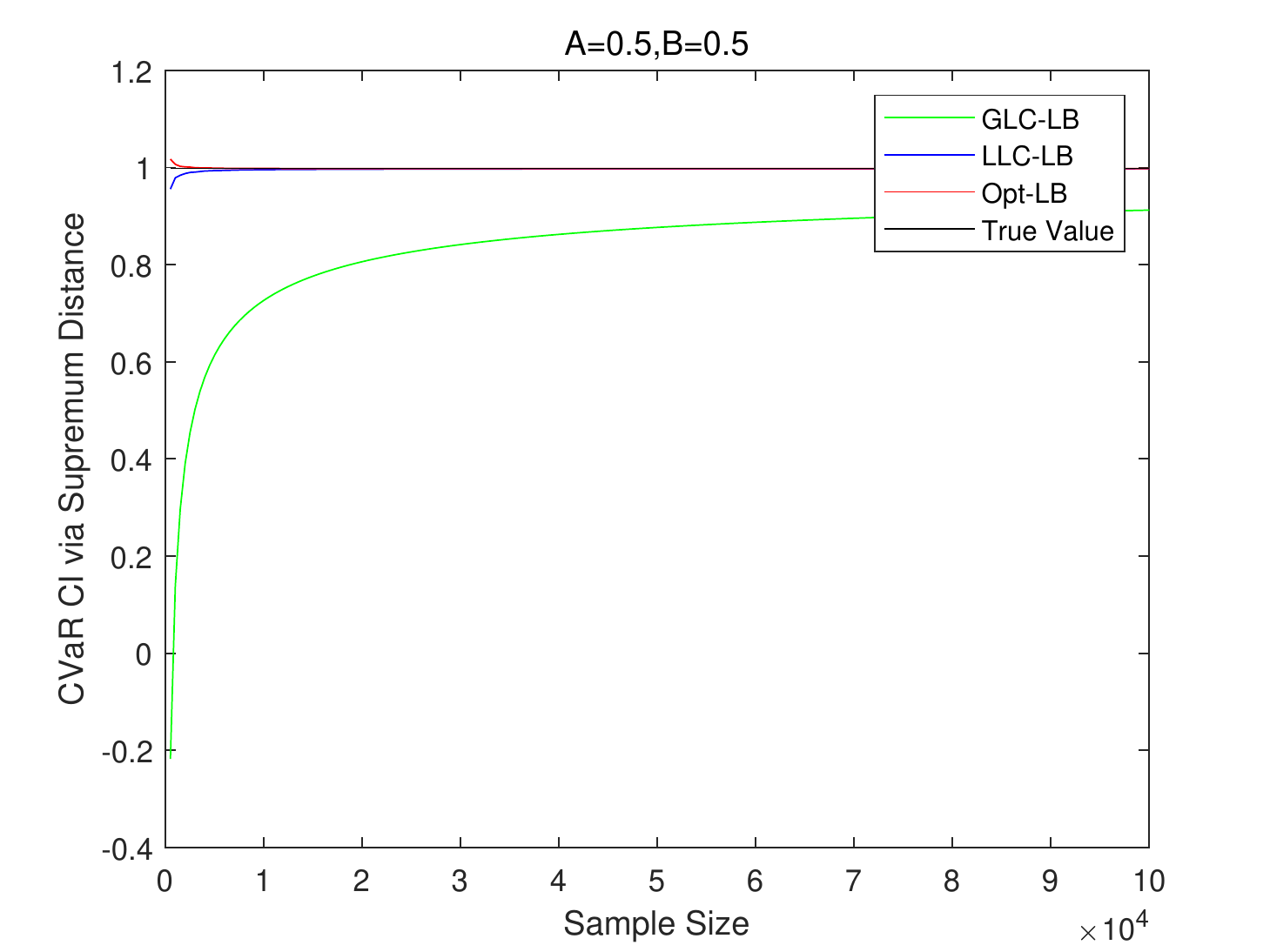}
     \end{subfigure}
     \begin{subfigure}[b]{0.19\textwidth}
         \centering
         \includegraphics[width=\textwidth]{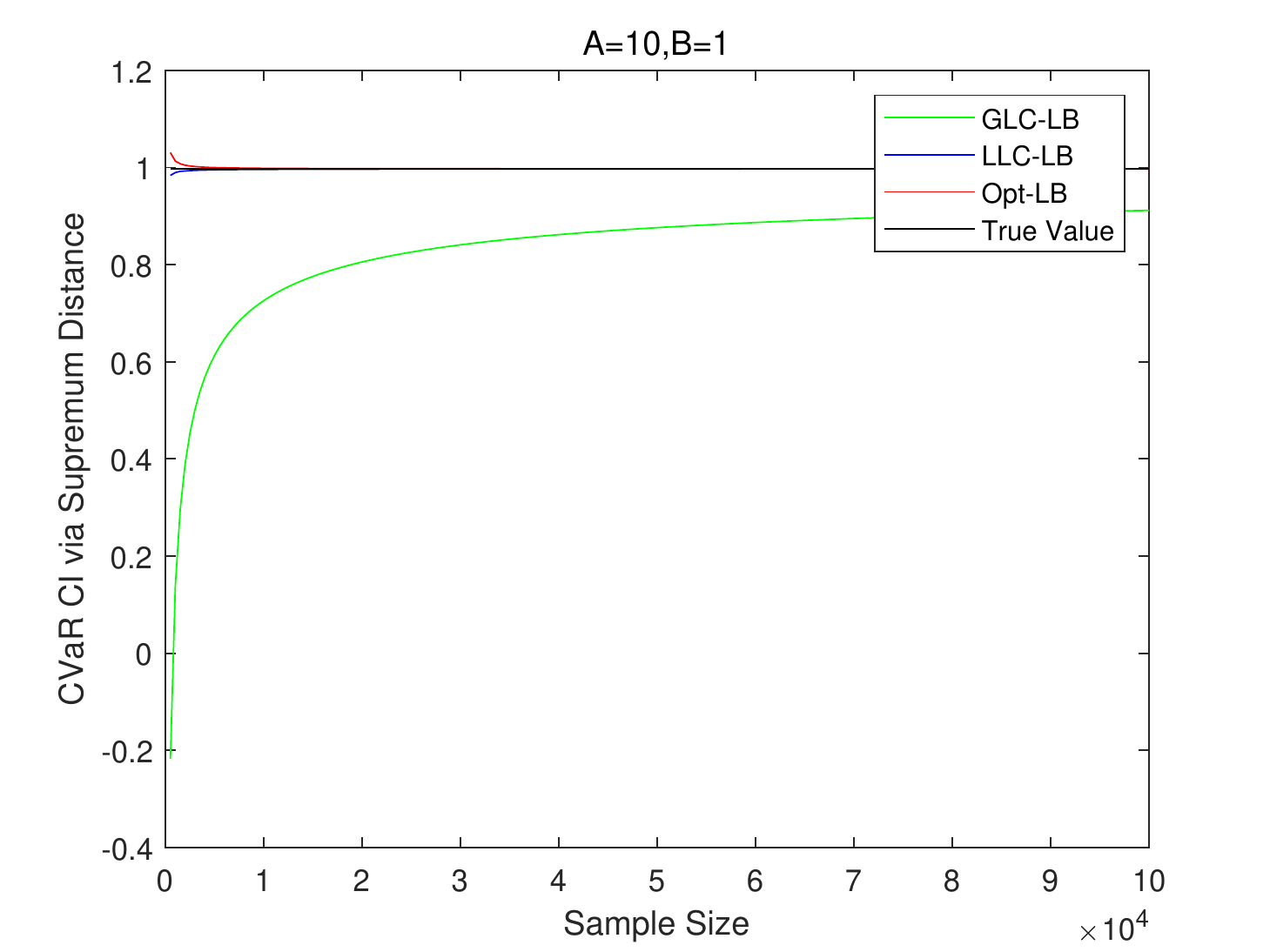}
     \end{subfigure}
        \caption{CVaR LCB with varying sample size}
        \label{fig:cvar_lb_n}
\end{figure*}

\begin{figure*}[ht]
     \centering
     \begin{subfigure}[b]{0.19\textwidth}
         \centering
         \includegraphics[width=\textwidth]{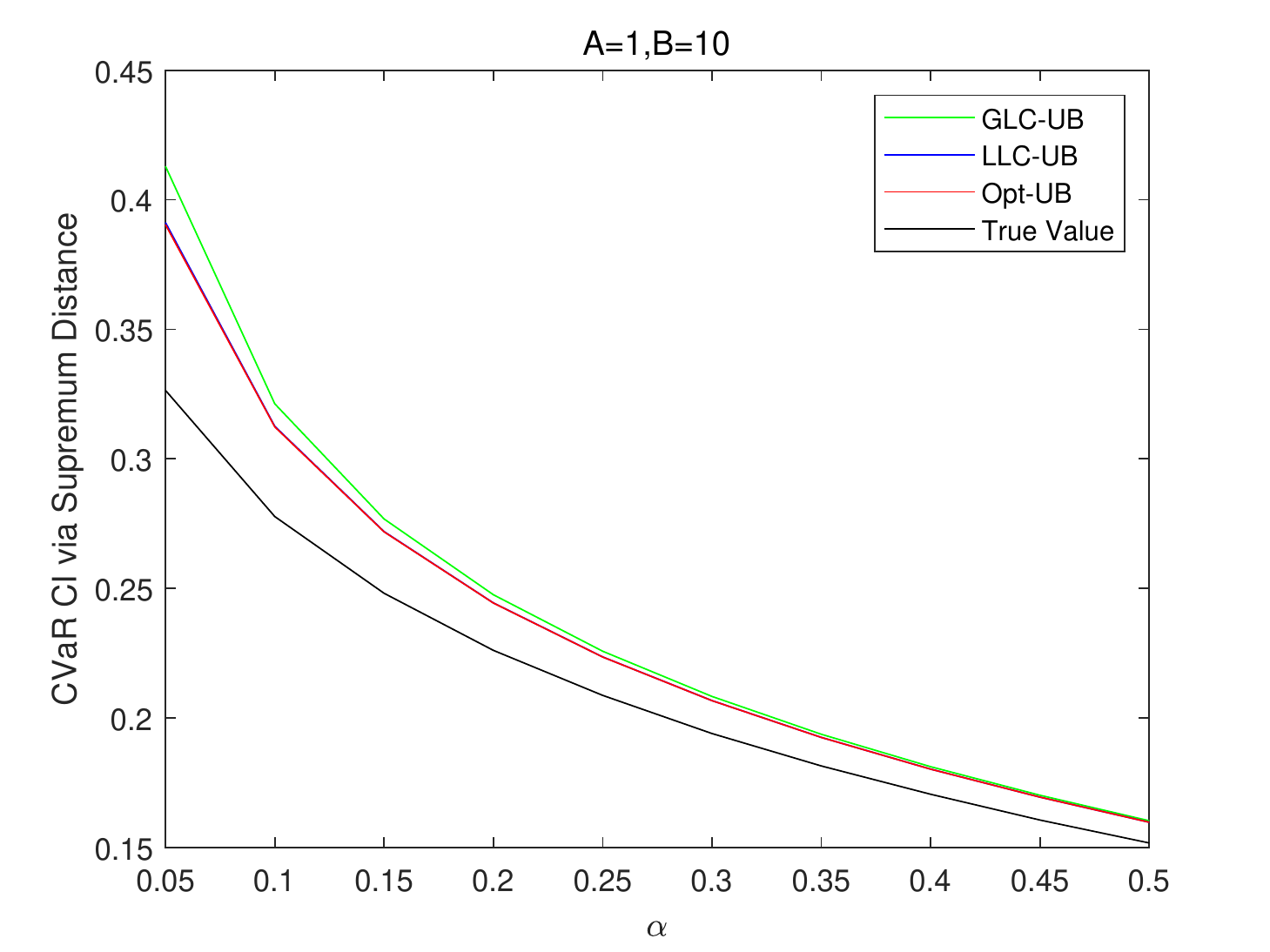}
     \end{subfigure}
     \hfill
     \begin{subfigure}[b]{0.19\textwidth}
         \centering
         \includegraphics[width=\textwidth]{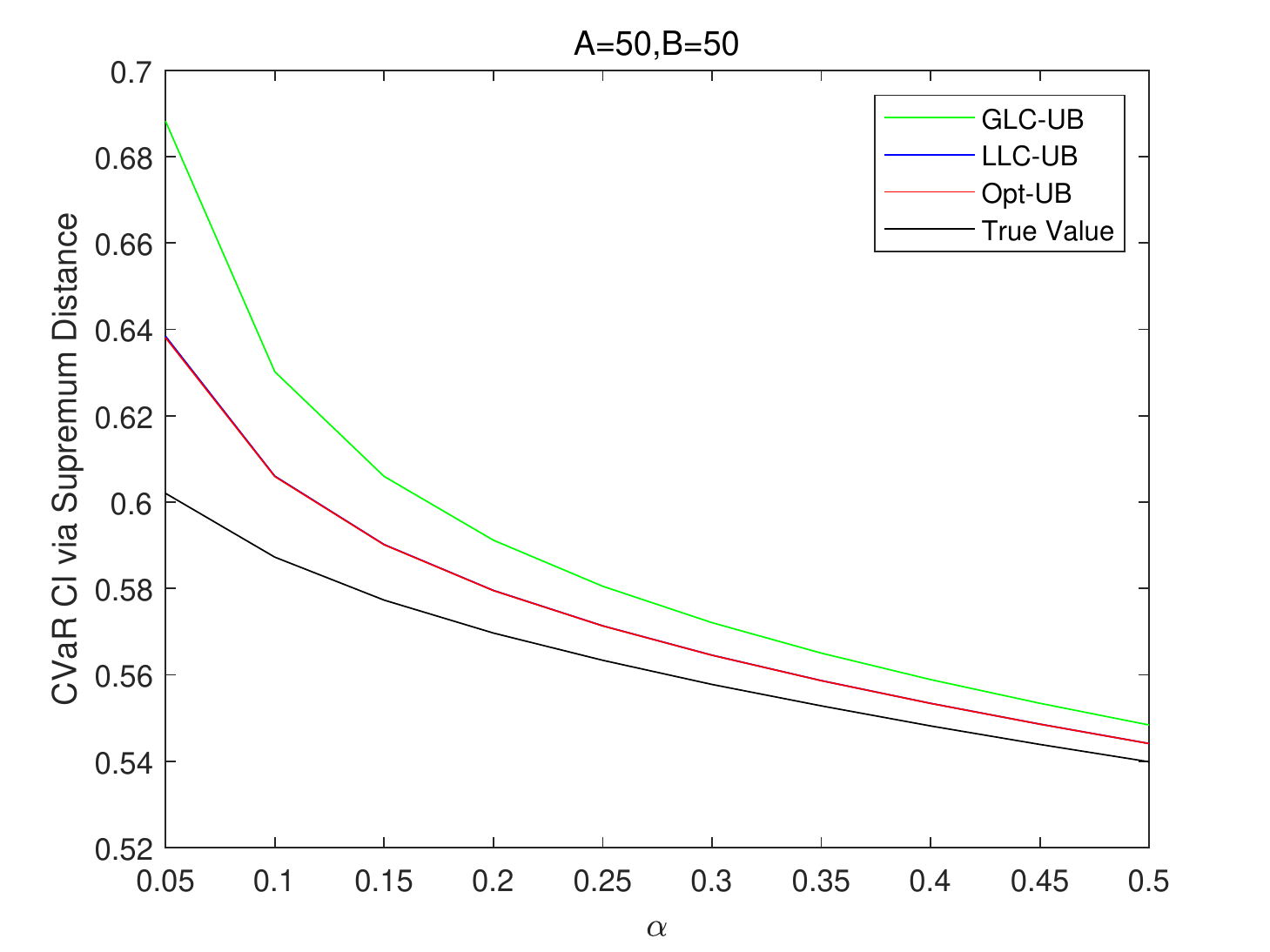}
     \end{subfigure}
     \begin{subfigure}[b]{0.19\textwidth}
         \centering
         \includegraphics[width=\textwidth]{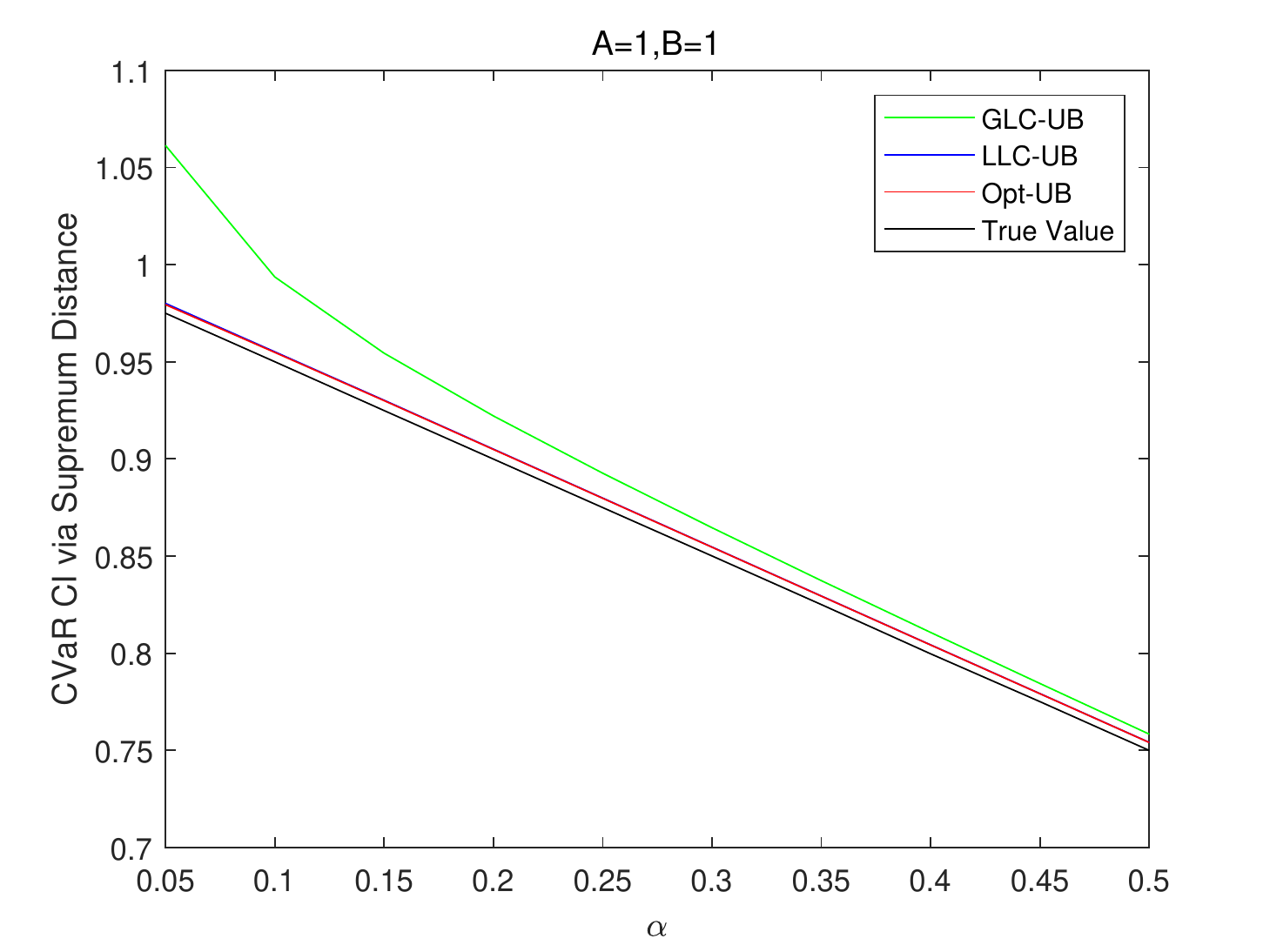}
     \end{subfigure}
     \hfill
     \begin{subfigure}[b]{0.19\textwidth}
         \centering
         \includegraphics[width=\textwidth]{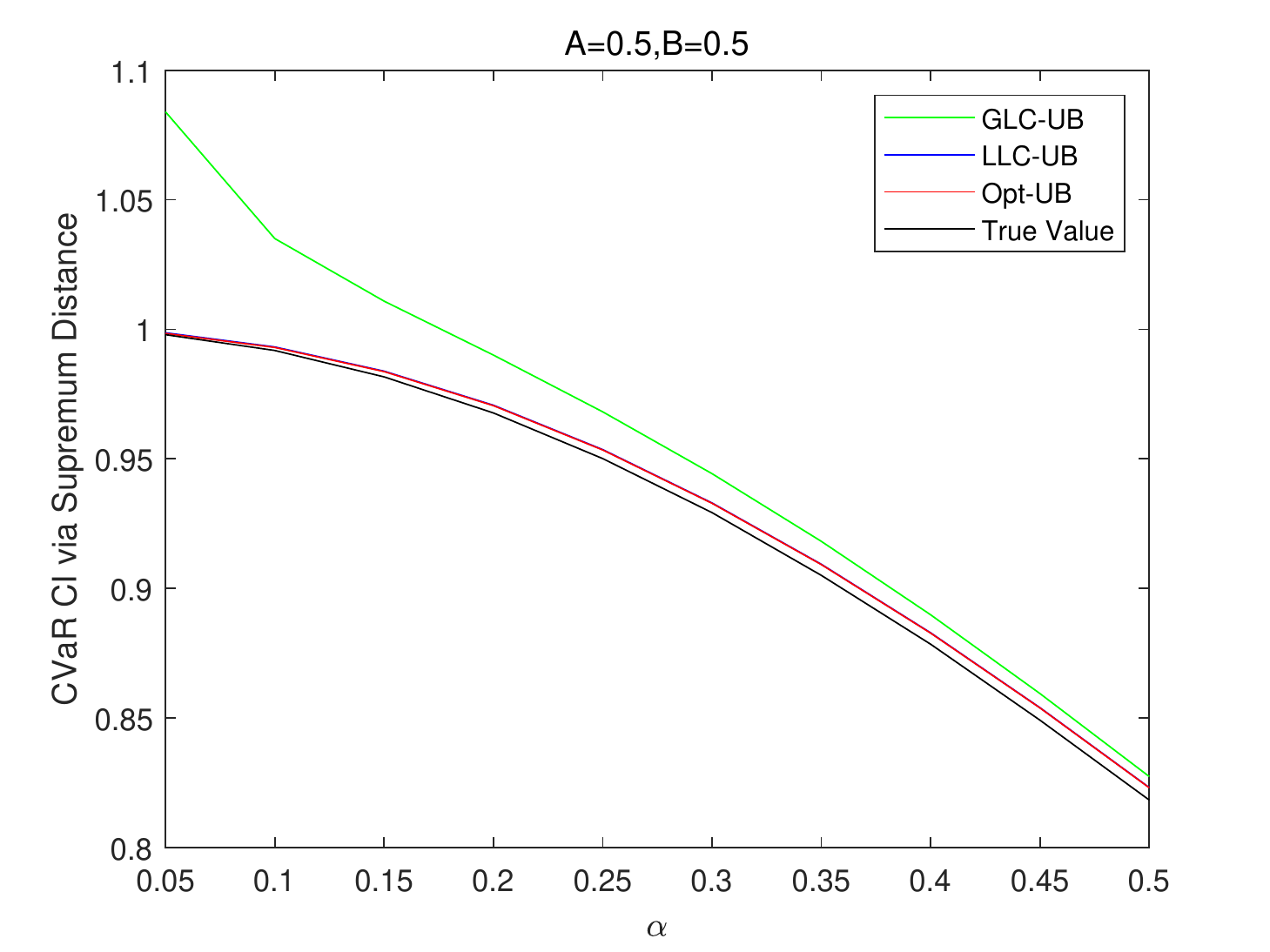}
     \end{subfigure}
     \begin{subfigure}[b]{0.19\textwidth}
         \centering
         \includegraphics[width=\textwidth]{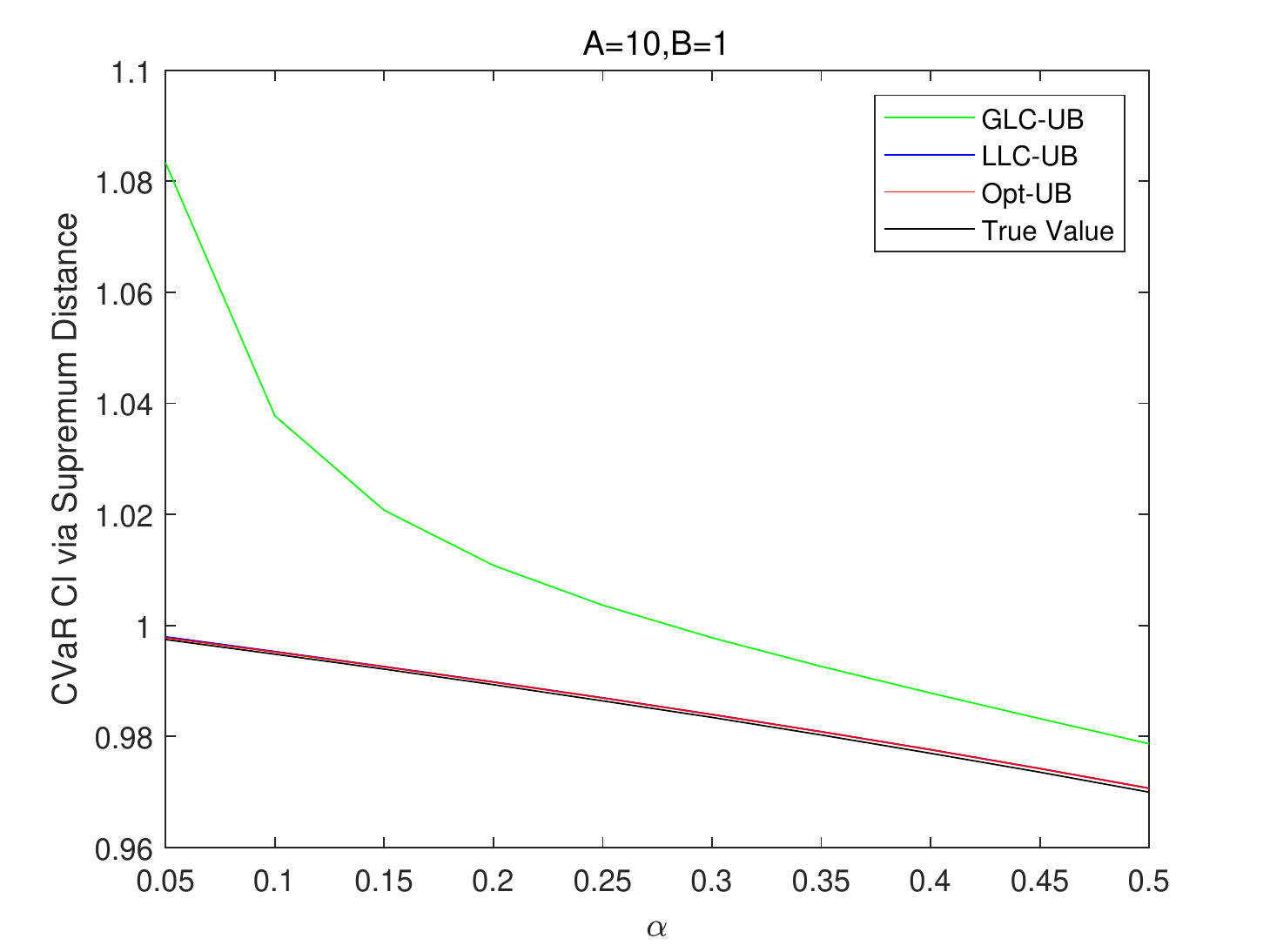}
     \end{subfigure}
        \caption{CVaR UCB with varying $\alpha$}
        \label{fig:cvar_ub_al}
        \vspace{-0ex}
\end{figure*}

\begin{figure*}[ht]
     \centering
     \begin{subfigure}[b]{0.19\textwidth}
         \centering
         \includegraphics[width=\textwidth]{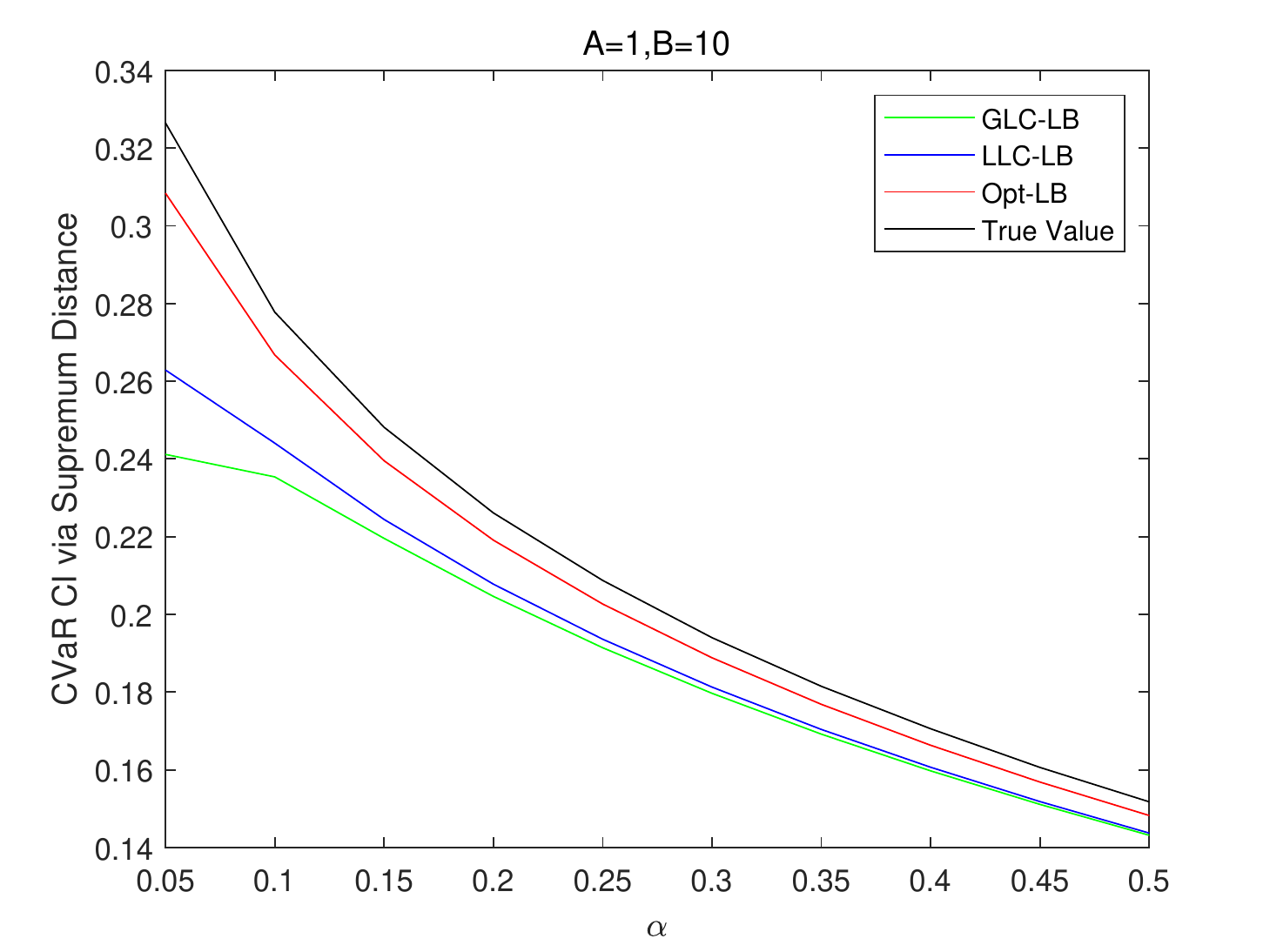}
     \end{subfigure}
     \hfill
     \begin{subfigure}[b]{0.19\textwidth}
         \centering
         \includegraphics[width=\textwidth]{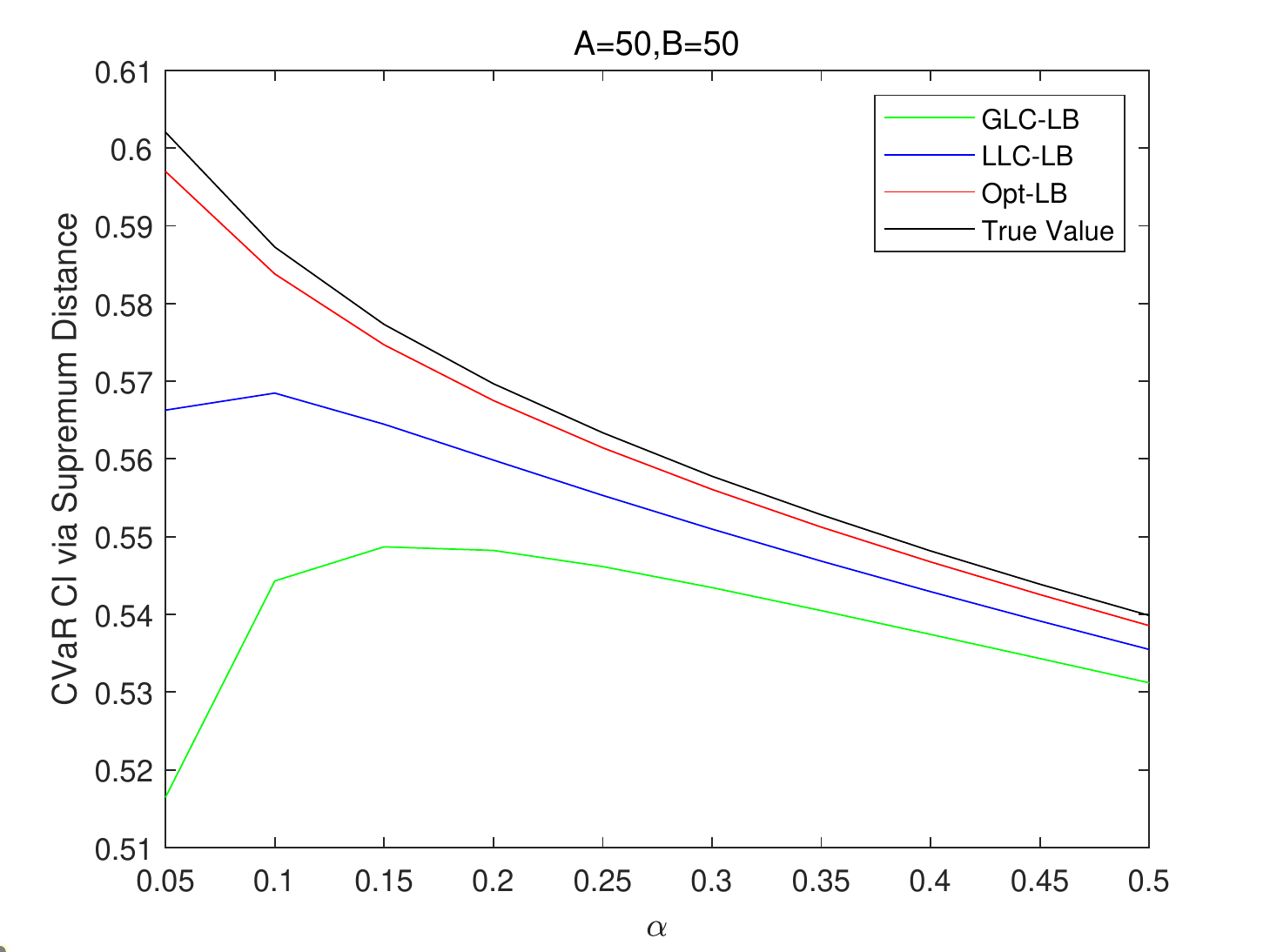}
     \end{subfigure}
     \begin{subfigure}[b]{0.19\textwidth}
         \centering
         \includegraphics[width=\textwidth]{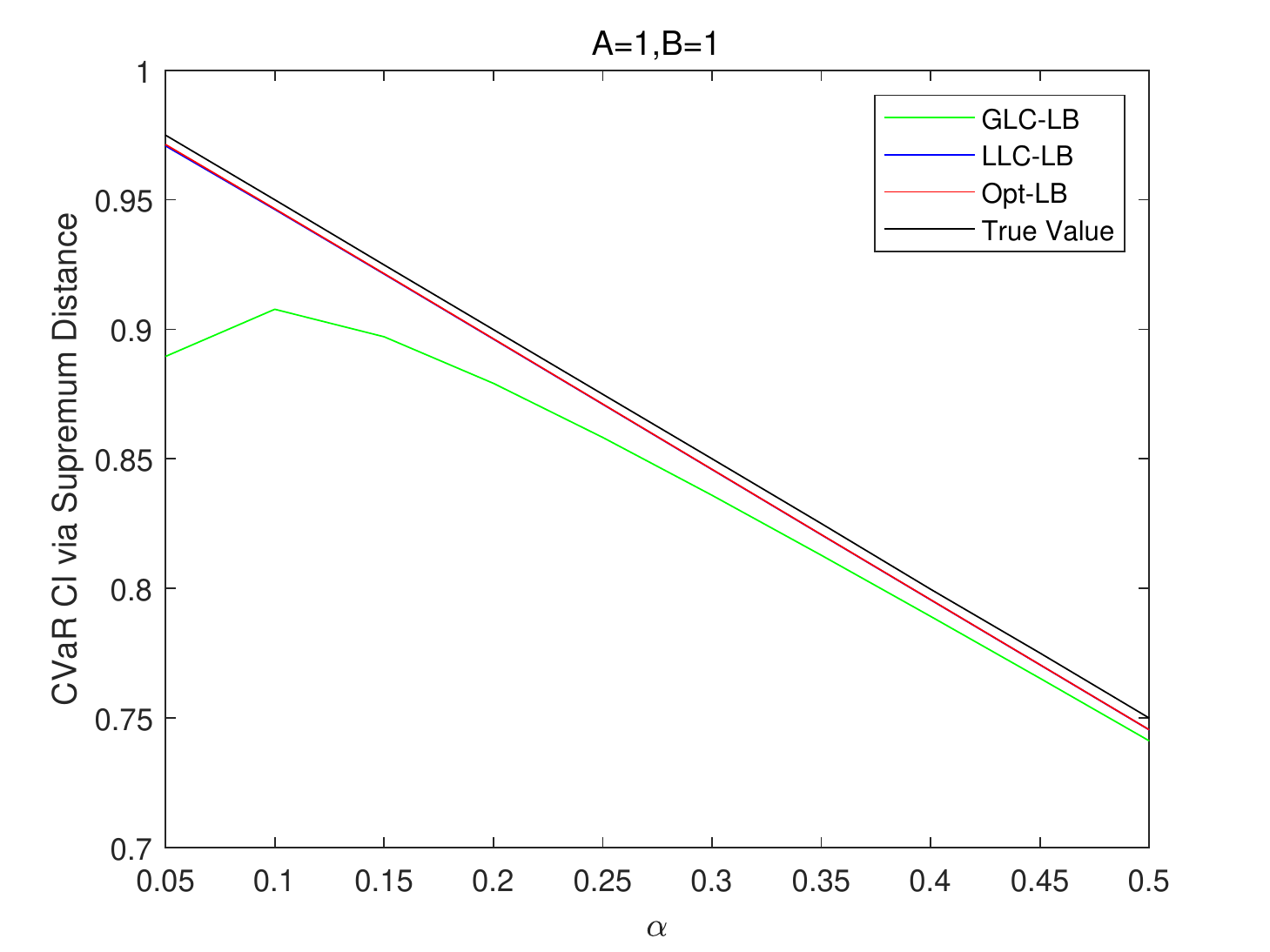}
     \end{subfigure}
     \hfill
     \begin{subfigure}[b]{0.19\textwidth}
         \centering
         \includegraphics[width=\textwidth]{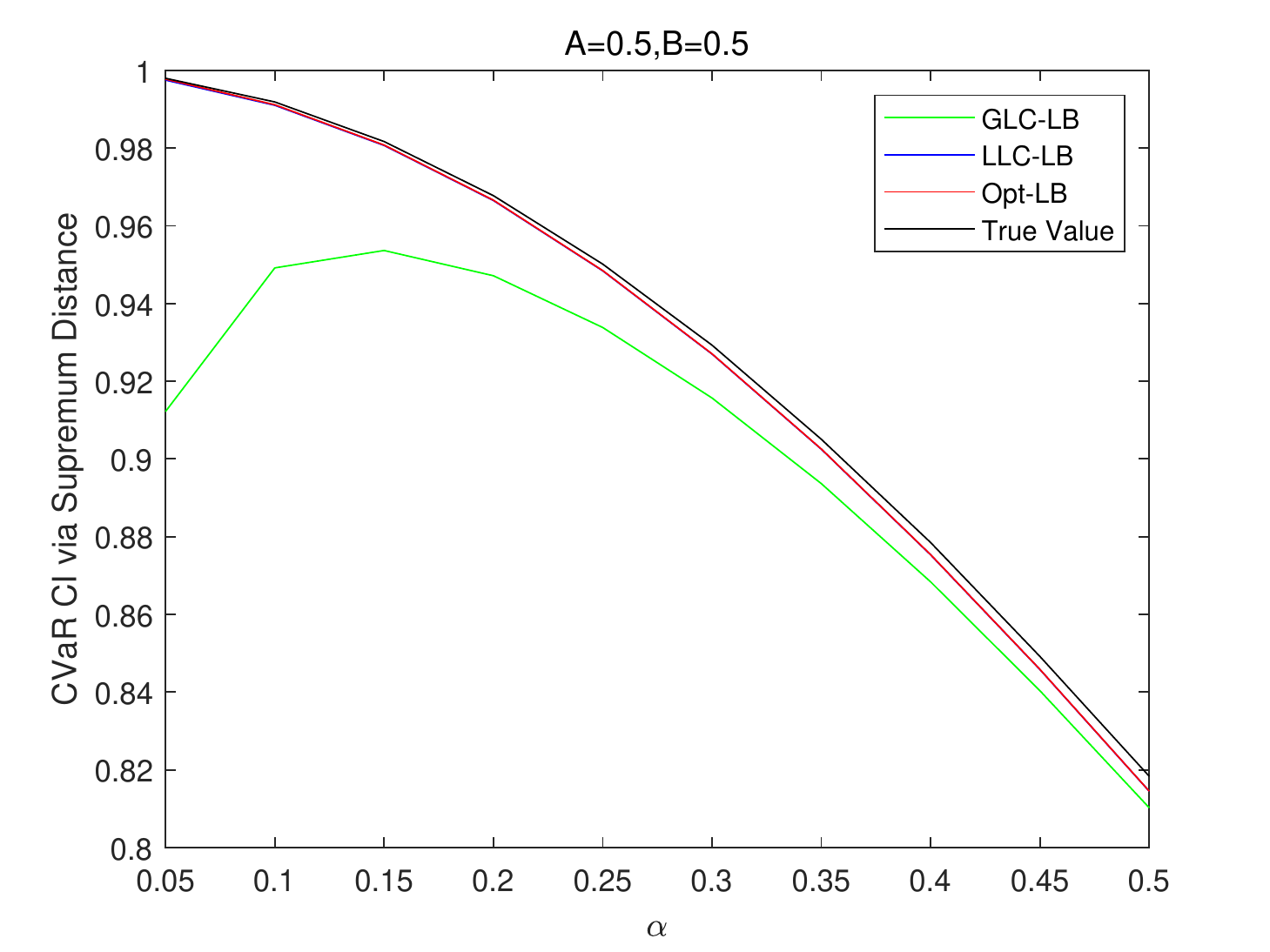}
     \end{subfigure}
     \begin{subfigure}[b]{0.19\textwidth}
         \centering
         \includegraphics[width=\textwidth]{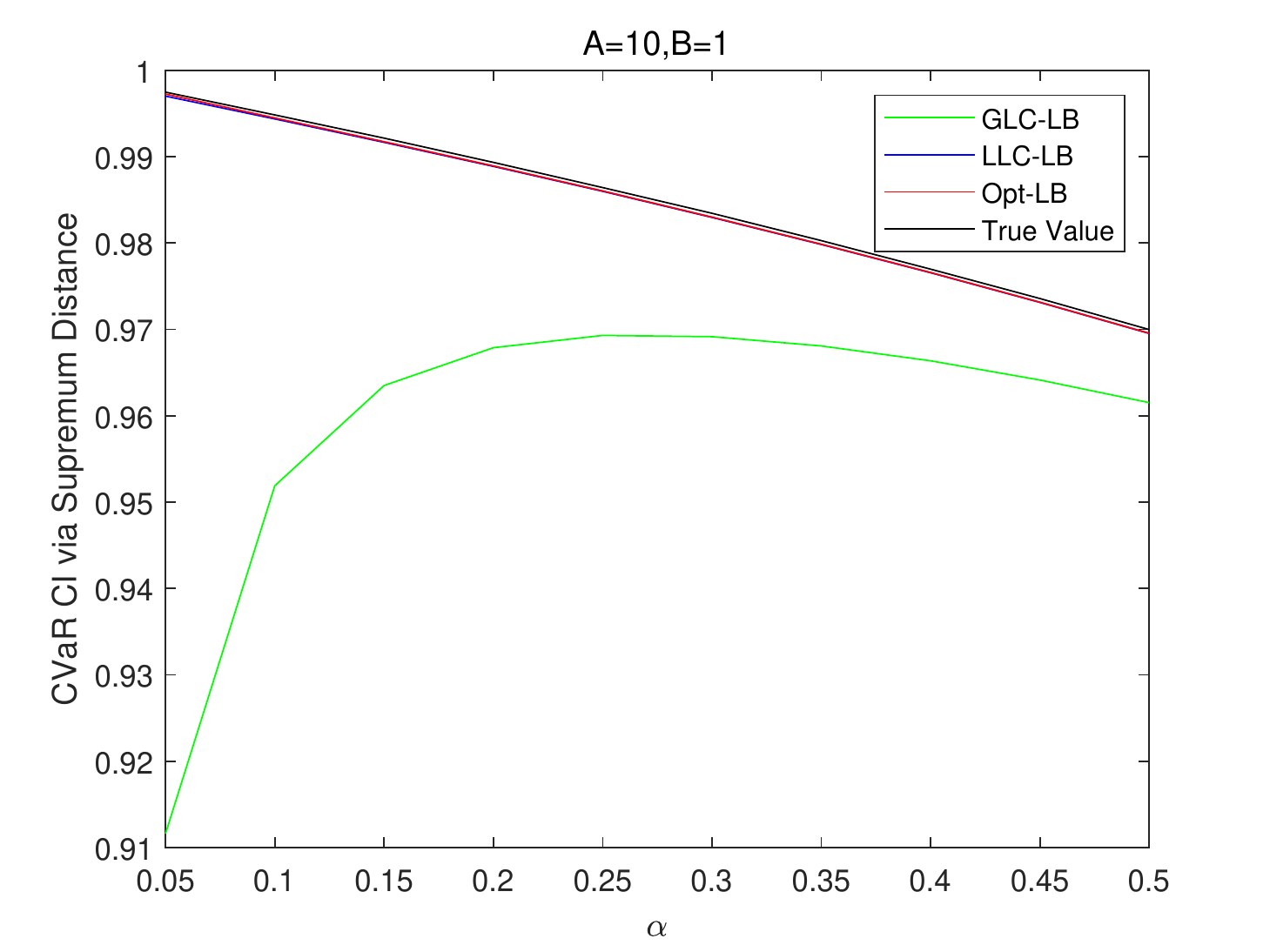}
     \end{subfigure}
        \caption{CVaR LCB with varying $\alpha$}
        \label{fig:cvar_lb_al}
\end{figure*}

\begin{figure*}[ht]
     \centering
     \begin{subfigure}[b]{0.33\textwidth}
         \centering
         \includegraphics[width=\textwidth]{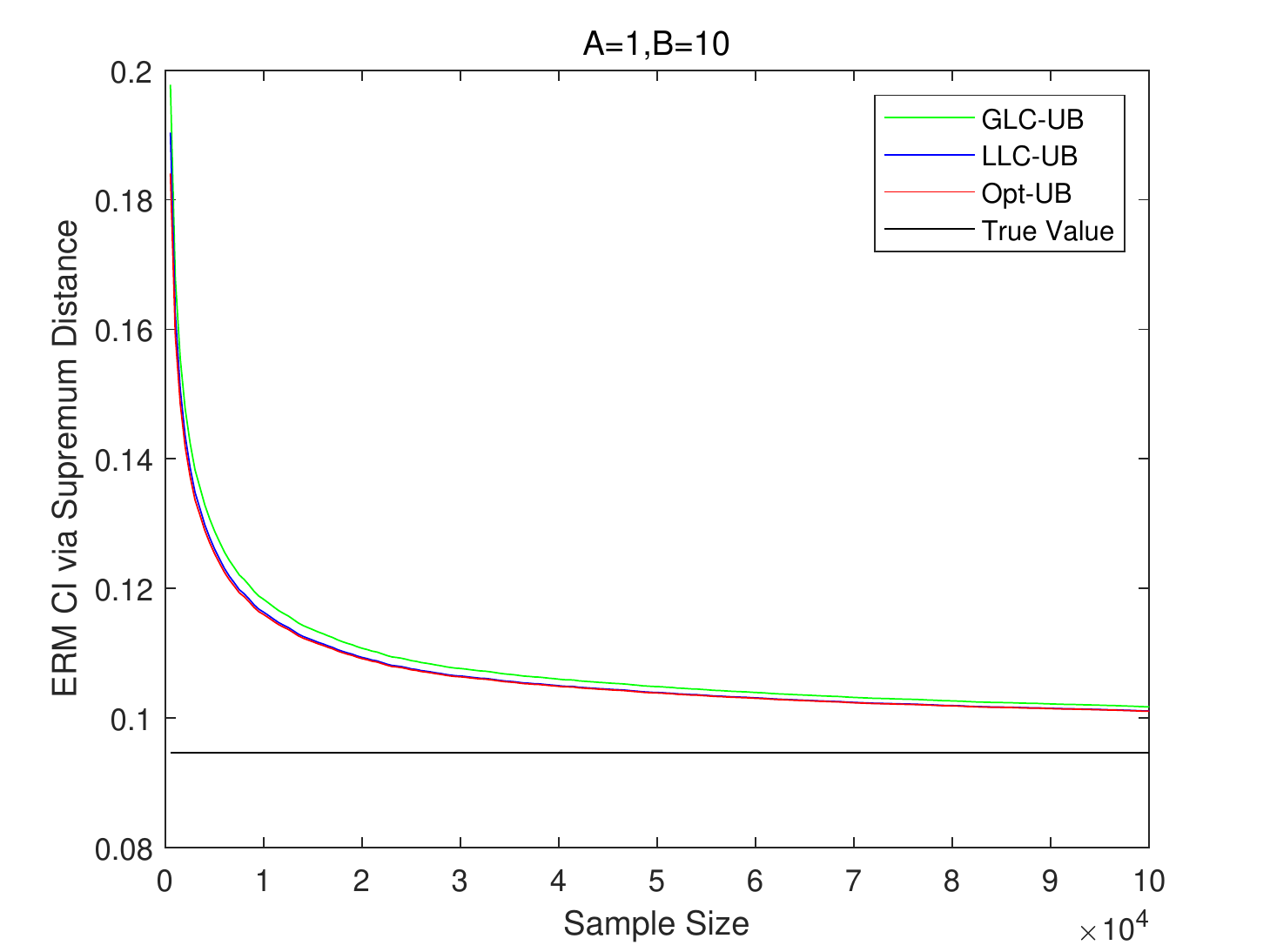}
     \end{subfigure}
     \hfill
     \begin{subfigure}[b]{0.33\textwidth}
         \centering
         \includegraphics[width=\textwidth]{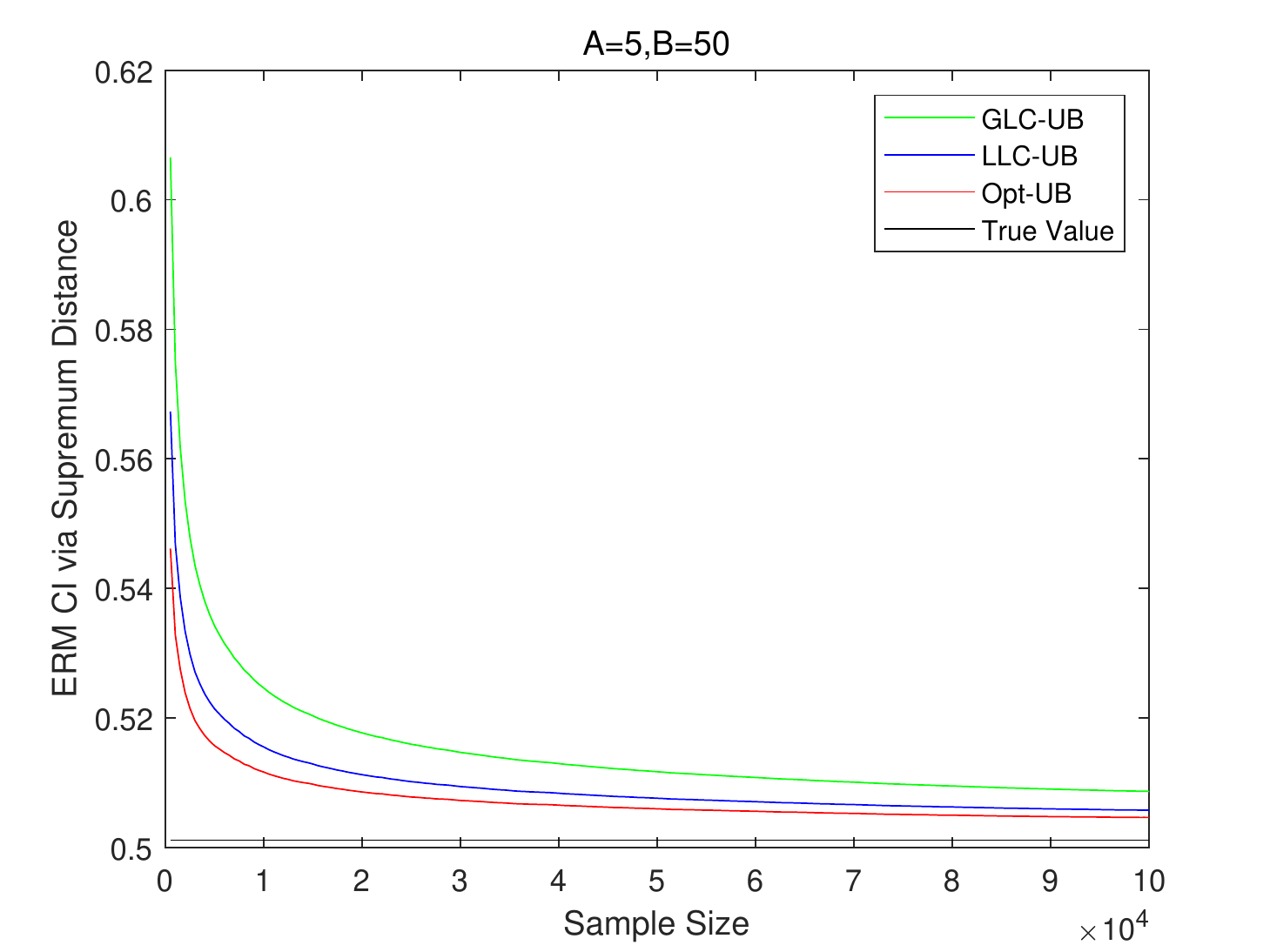}
     \end{subfigure}
     \begin{subfigure}[b]{0.33\textwidth}
         \centering
         \includegraphics[width=\textwidth]{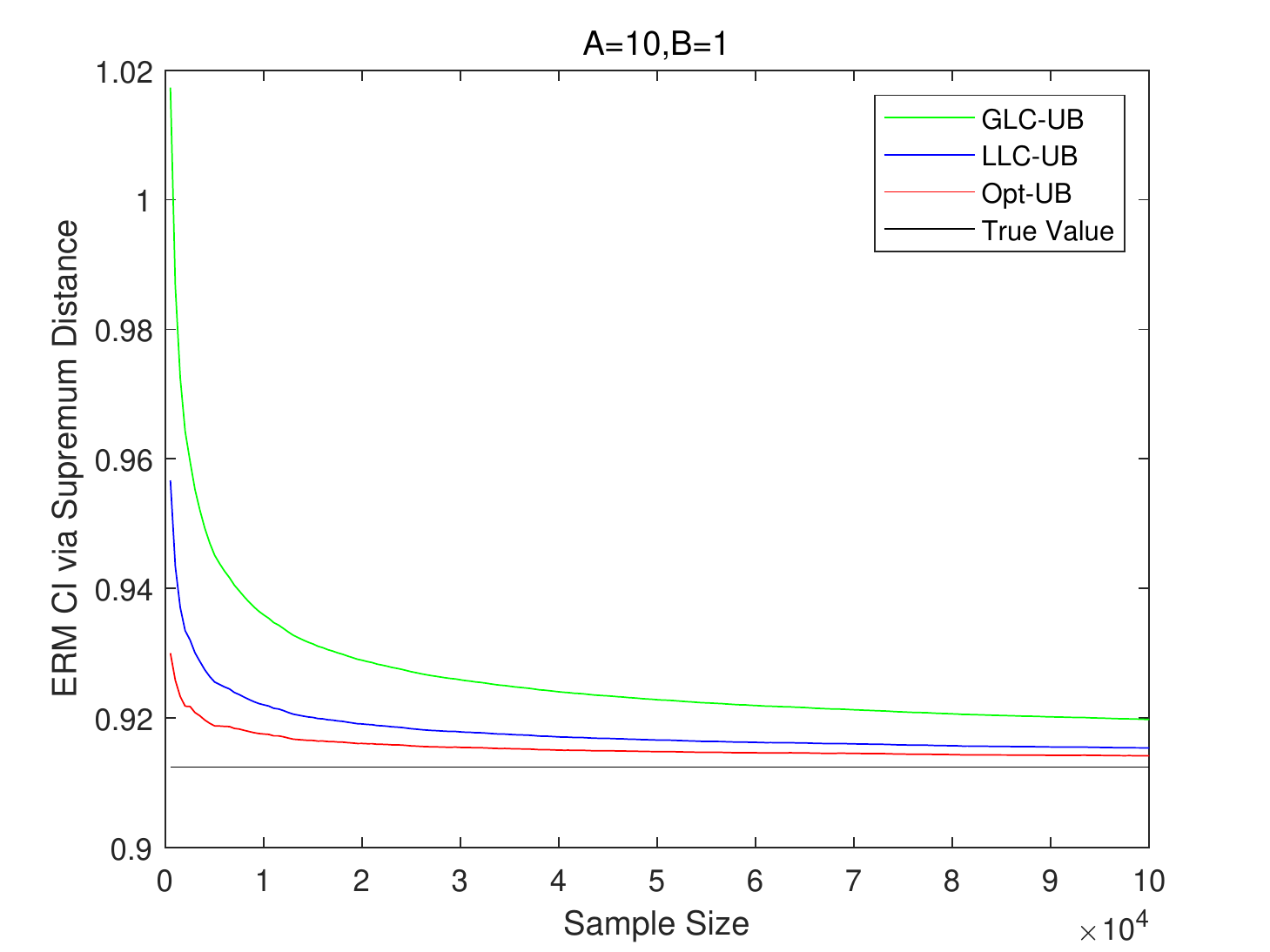}
     \end{subfigure}
     \begin{subfigure}[b]{0.33\textwidth}
         \centering
         \includegraphics[width=\textwidth]{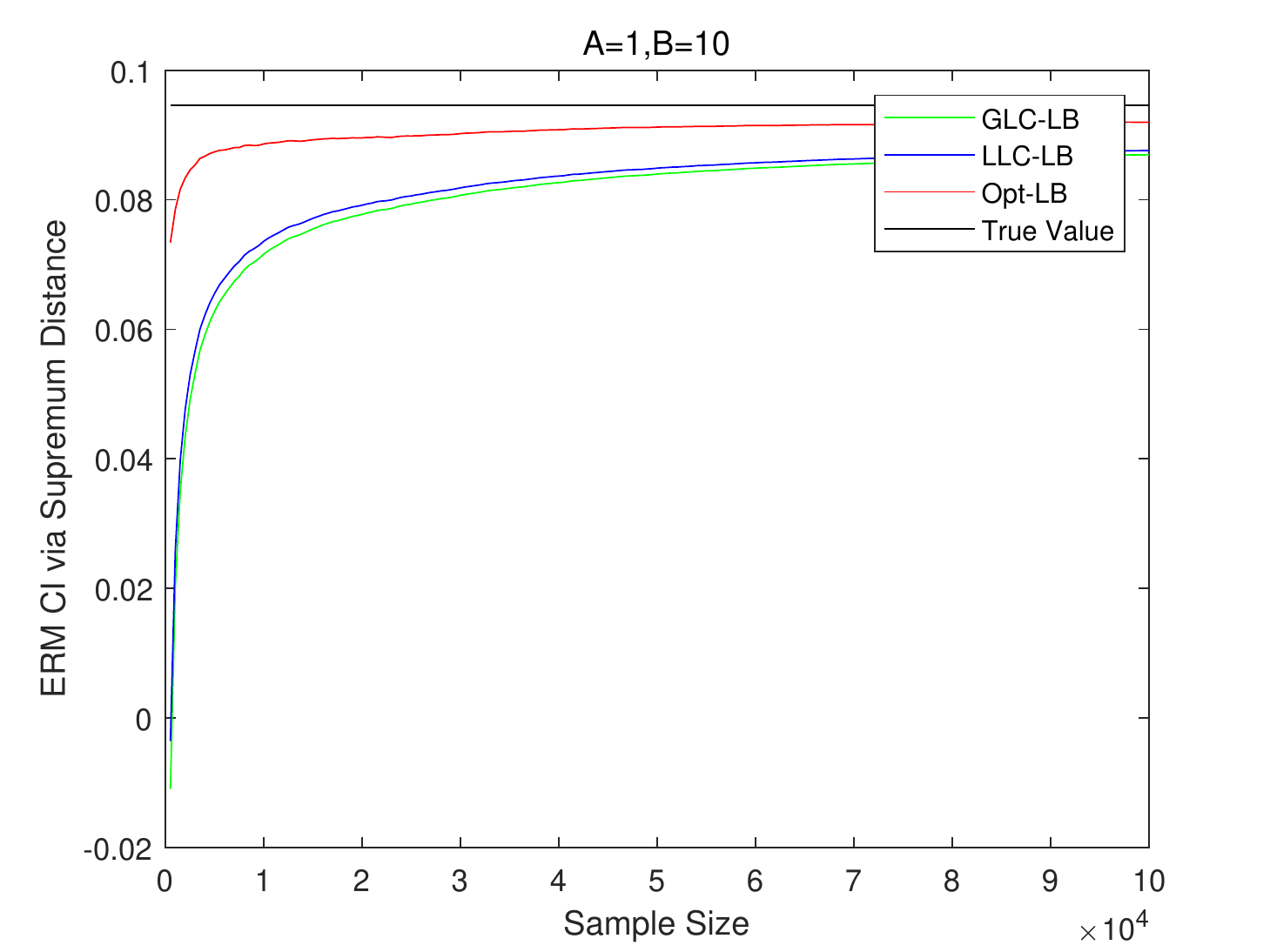}
     \end{subfigure}
     \hfill
     \begin{subfigure}[b]{0.33\textwidth}
         \centering
         \includegraphics[width=\textwidth]{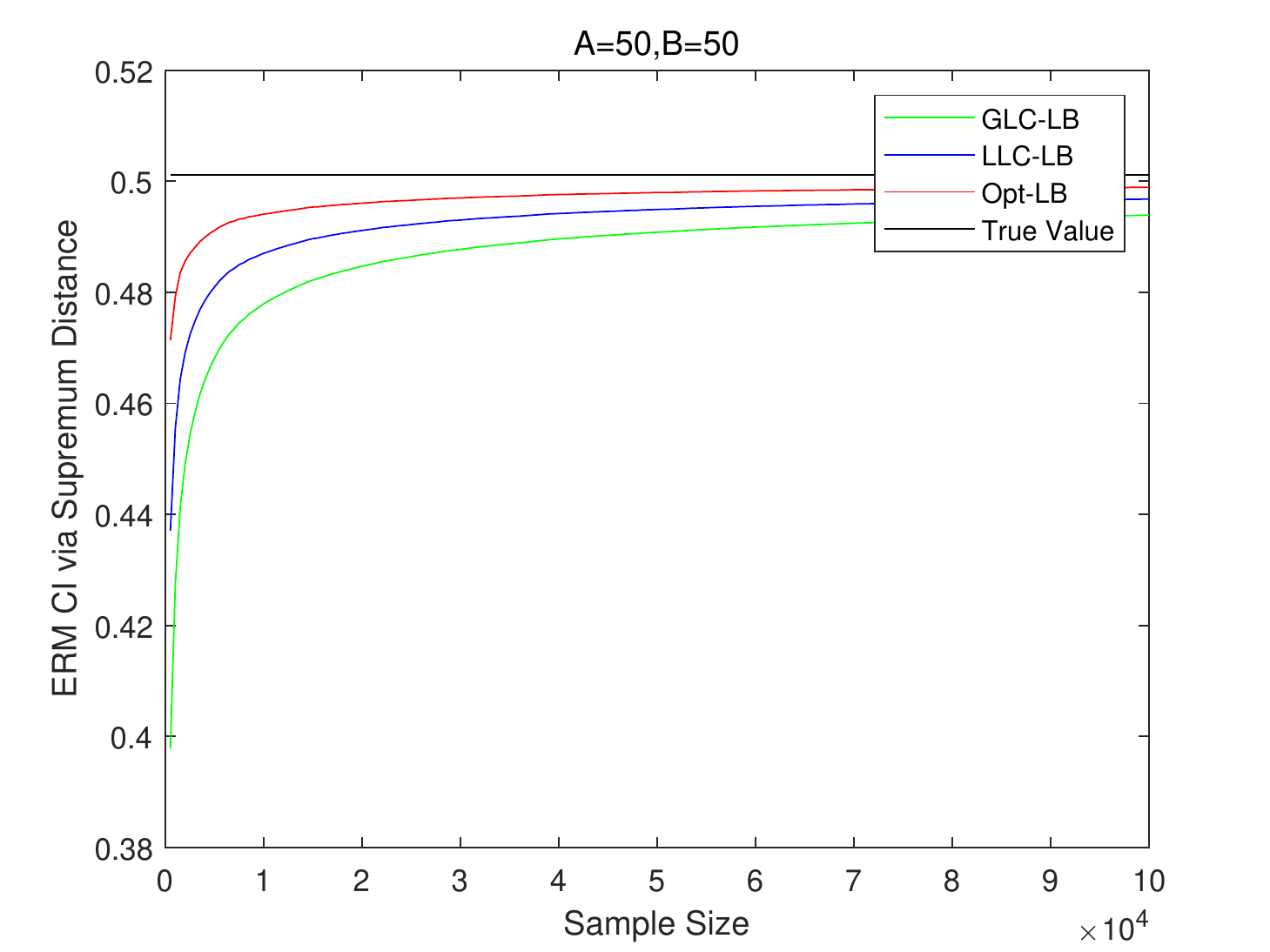}
     \end{subfigure}
     \begin{subfigure}[b]{0.33\textwidth}
         \centering
         \includegraphics[width=\textwidth]{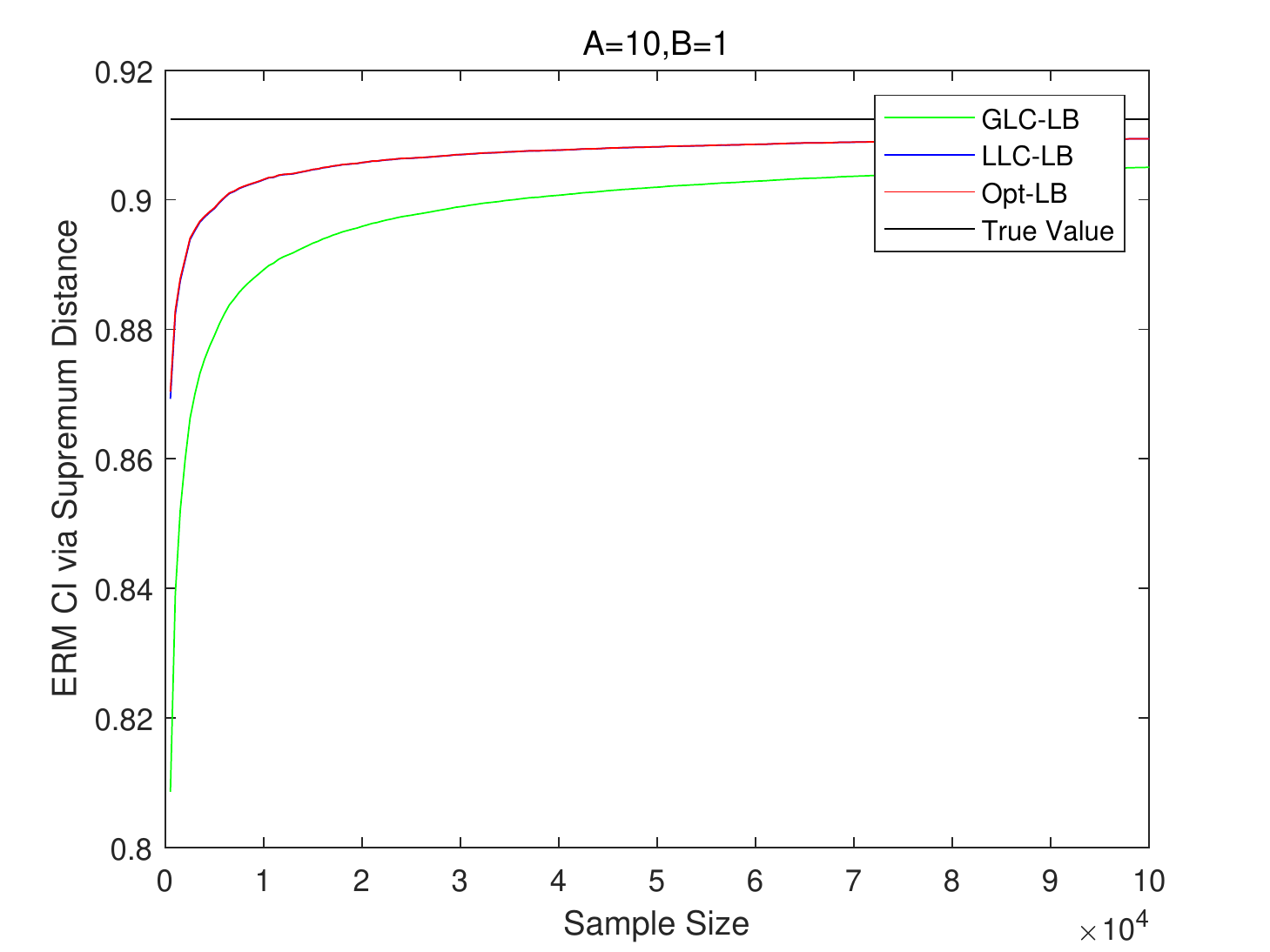}
     \end{subfigure}
        \caption{ERM CI with varying sample size}
        \label{fig:erm_n_sup}
        \vspace{-0ex}
\end{figure*}

\begin{figure*}[ht]
     \centering
     \begin{subfigure}[b]{0.33\textwidth}
         \centering
         \includegraphics[width=\textwidth]{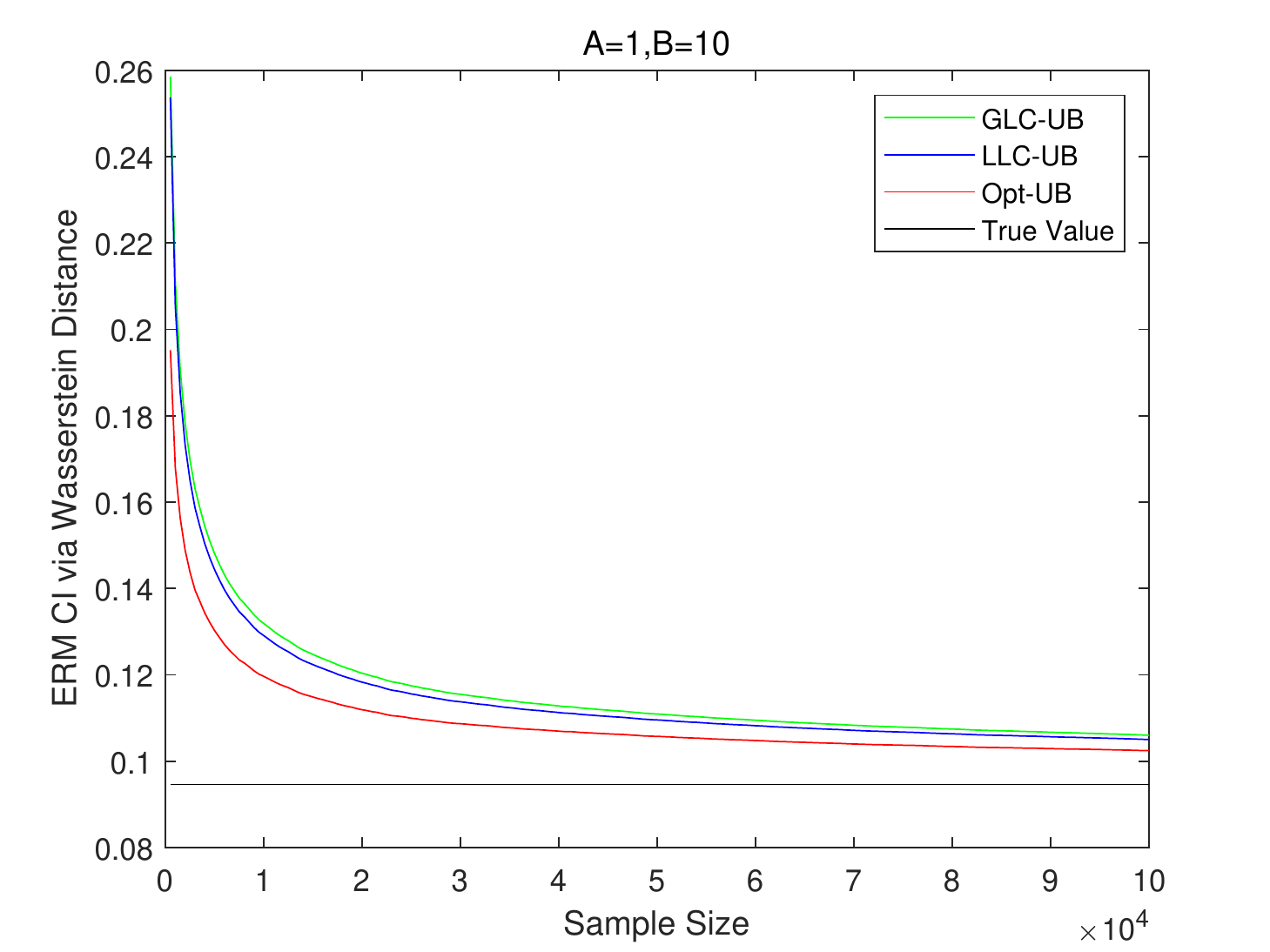}
     \end{subfigure}
     \hfill
     \begin{subfigure}[b]{0.33\textwidth}
         \centering
         \includegraphics[width=\textwidth]{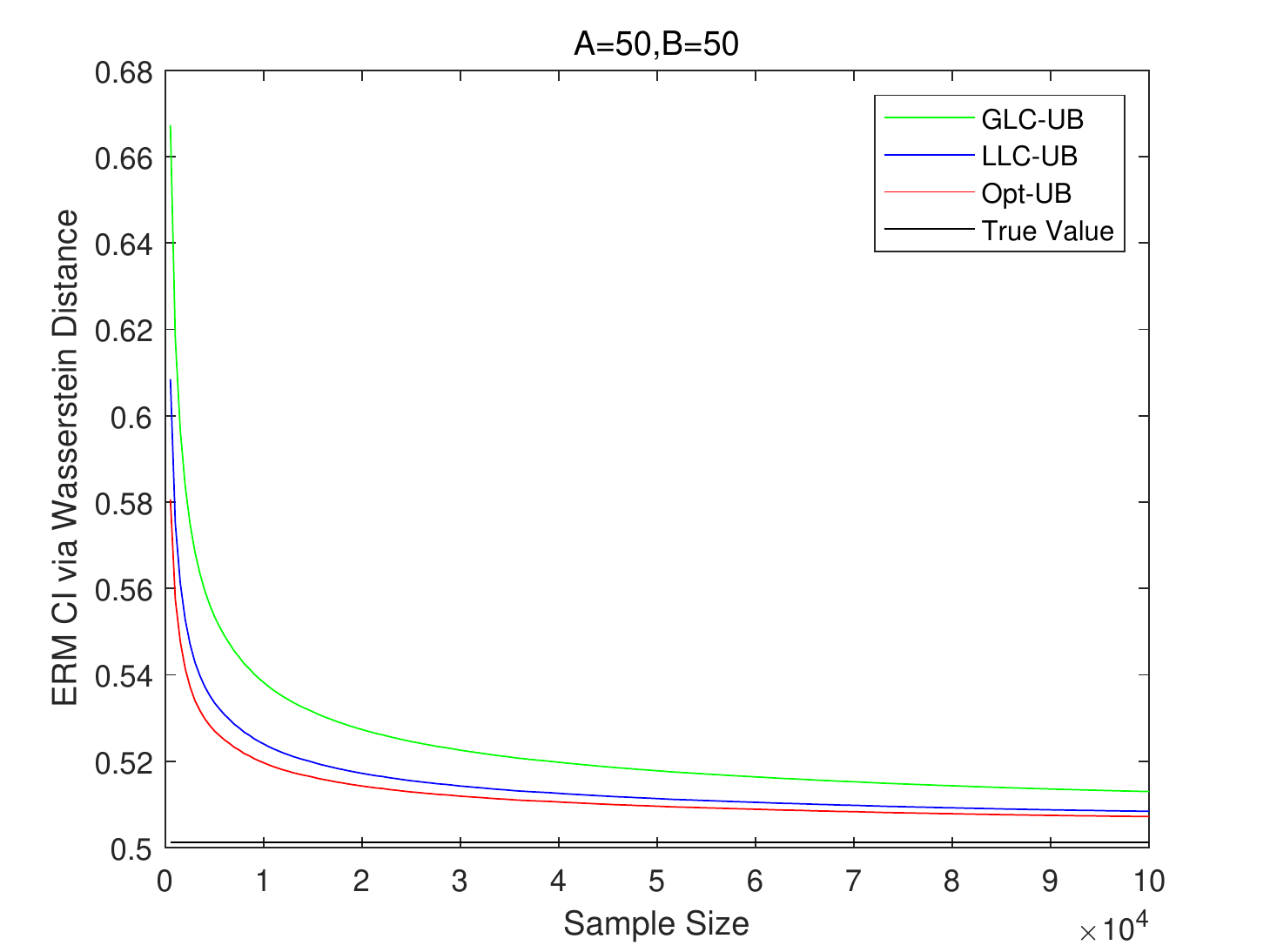}
     \end{subfigure}
     \begin{subfigure}[b]{0.33\textwidth}
         \centering
         \includegraphics[width=\textwidth]{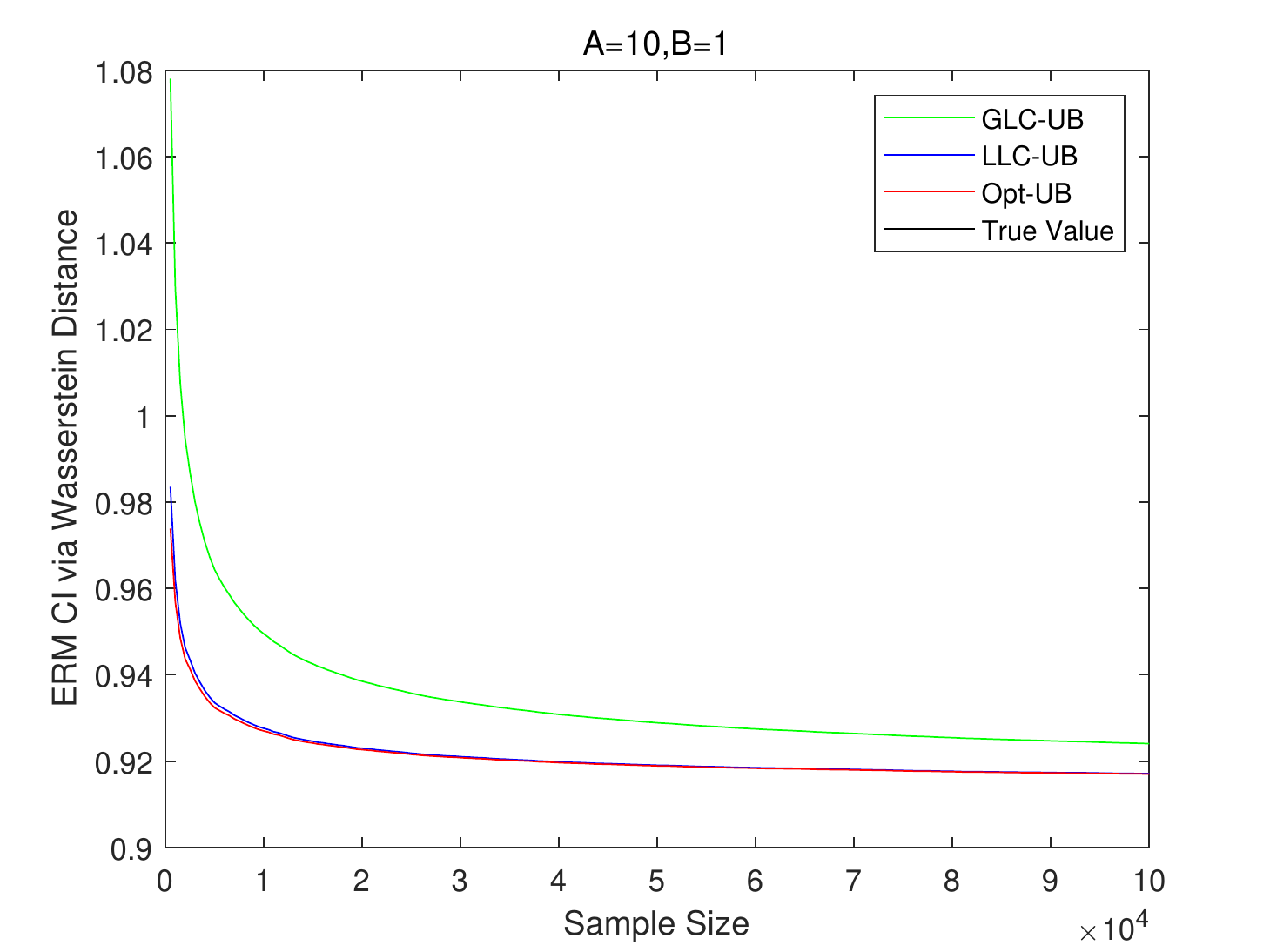}
     \end{subfigure}
     \begin{subfigure}[b]{0.33\textwidth}
         \centering
         \includegraphics[width=\textwidth]{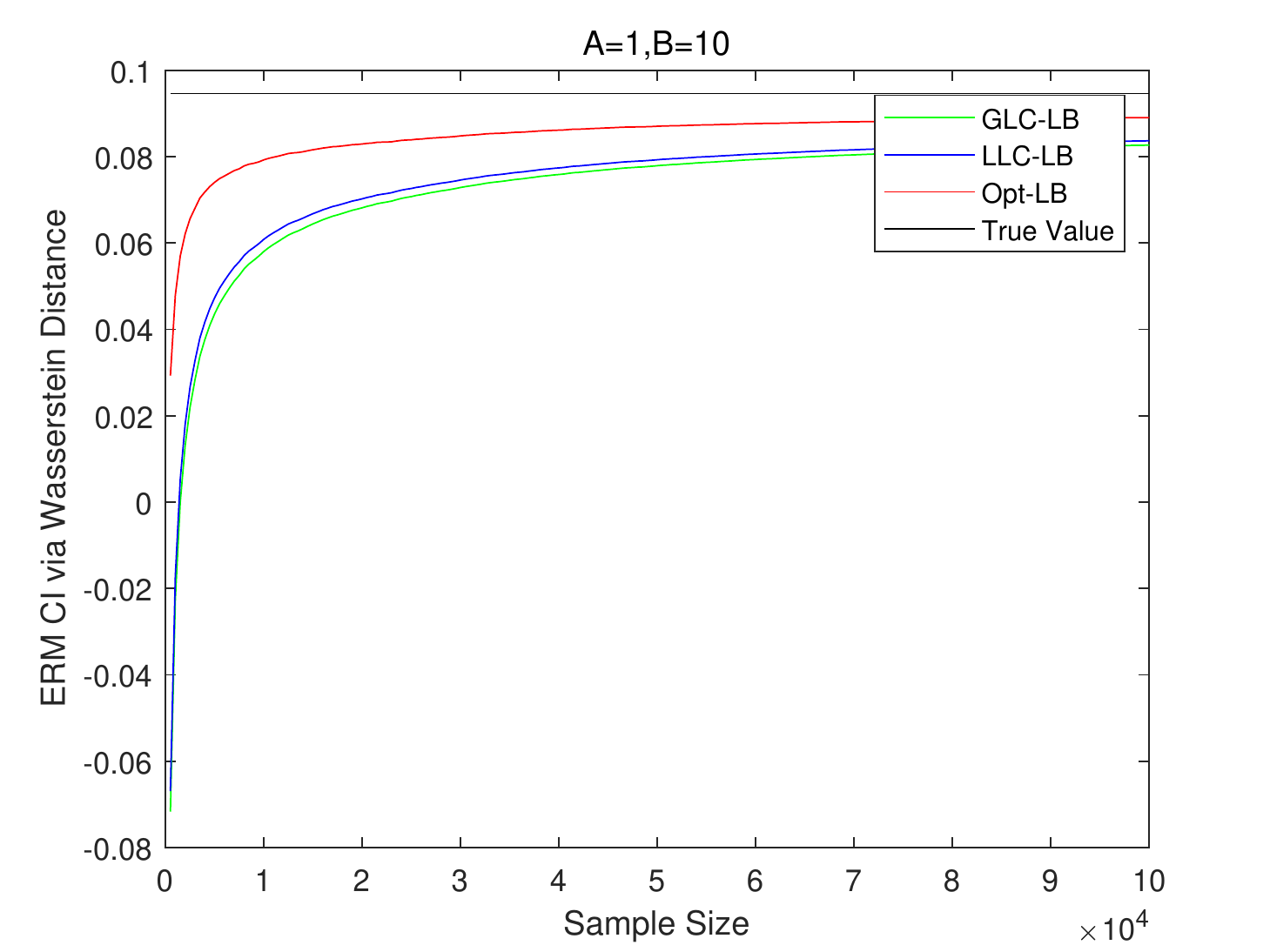}
     \end{subfigure}
     \hfill
     \begin{subfigure}[b]{0.33\textwidth}
         \centering
         \includegraphics[width=\textwidth]{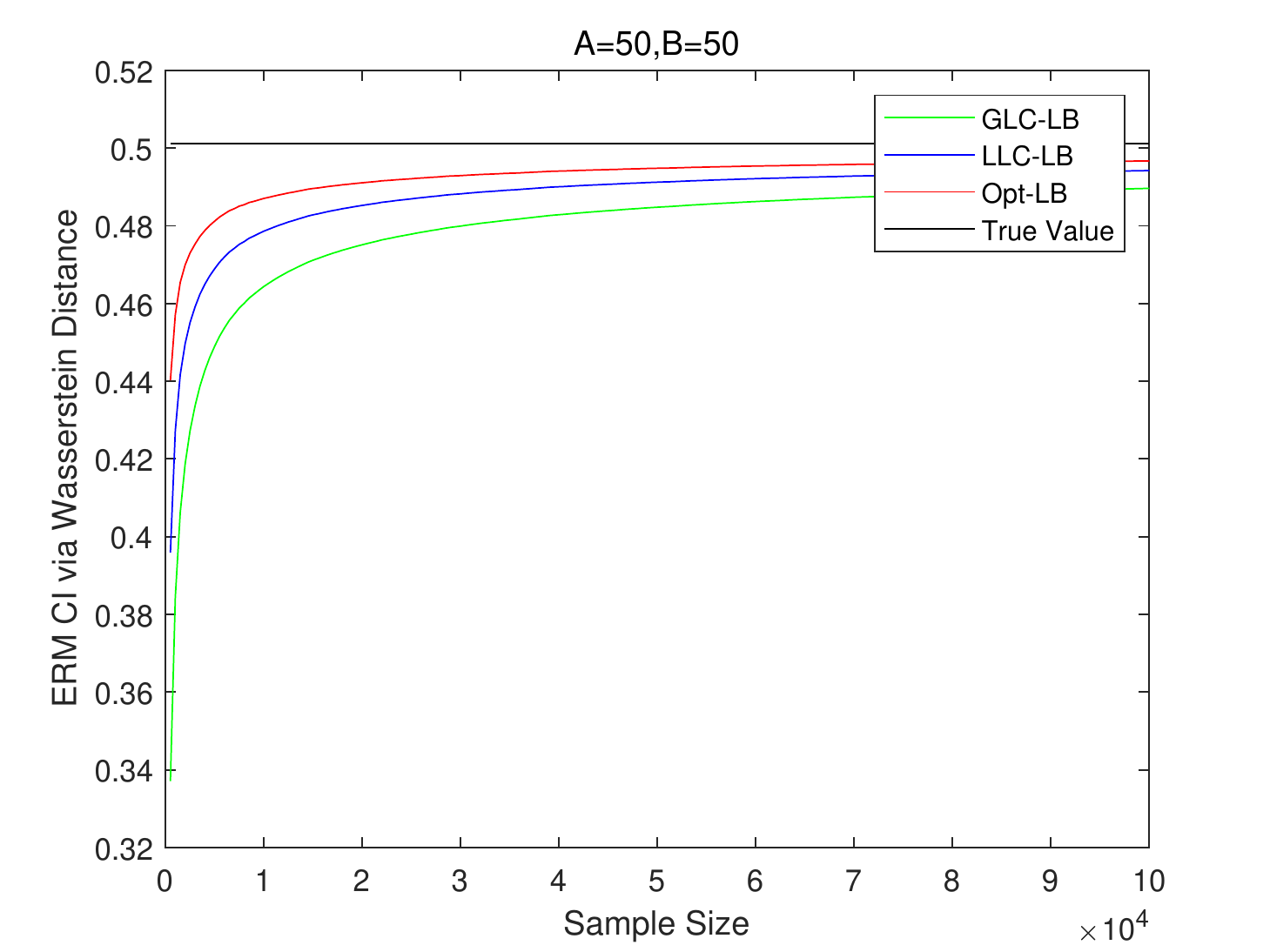}
     \end{subfigure}
     \begin{subfigure}[b]{0.33\textwidth}
         \centering
         \includegraphics[width=\textwidth]{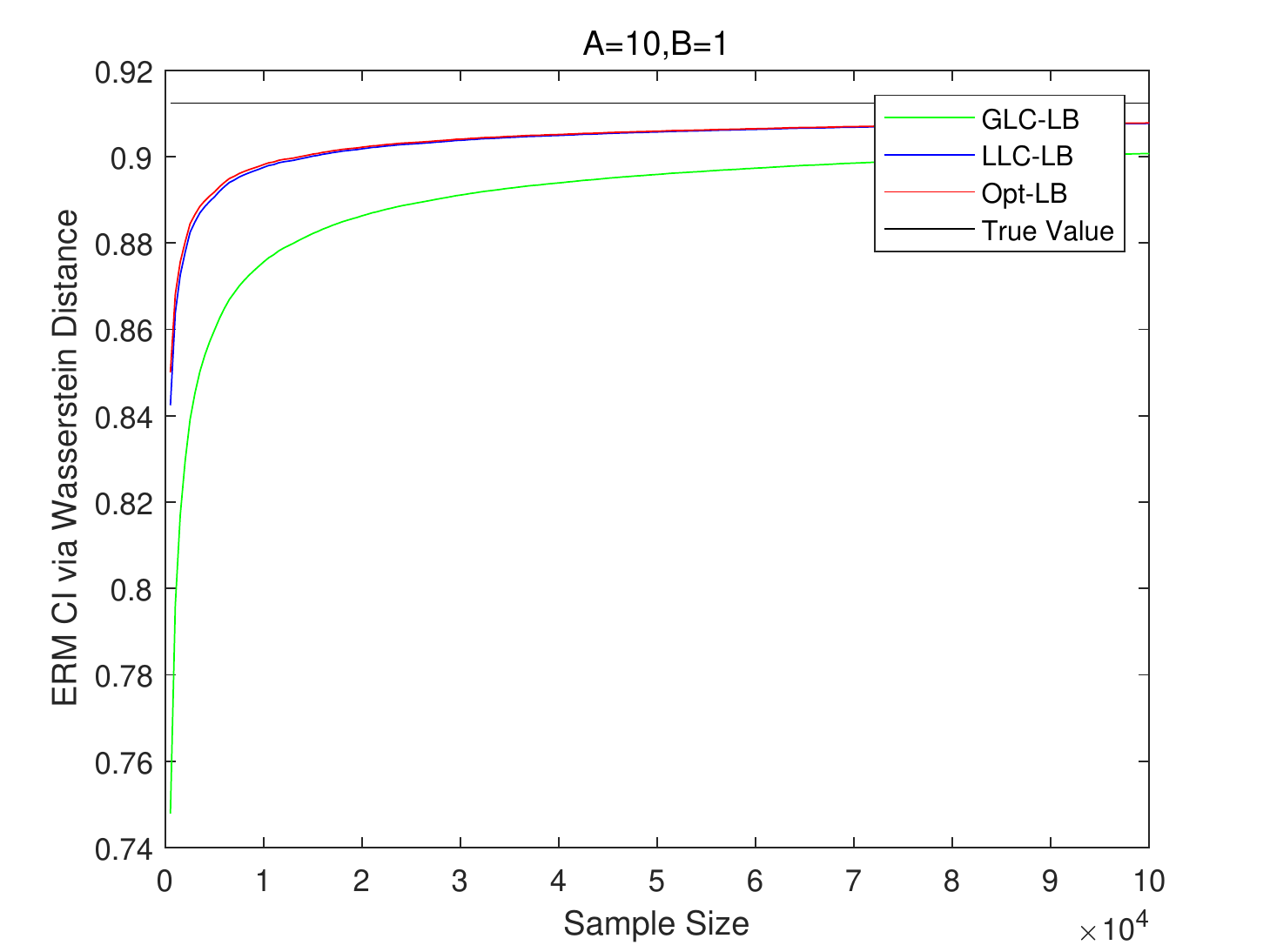}
     \end{subfigure}
        \caption{ERM CI with varying sample size}
        \label{fig:erm_n_was}
        \vspace{-0ex}
\end{figure*}

\begin{figure*}[ht]
     \centering
     \begin{subfigure}[b]{0.33\textwidth}
         \centering
         \includegraphics[width=\textwidth]{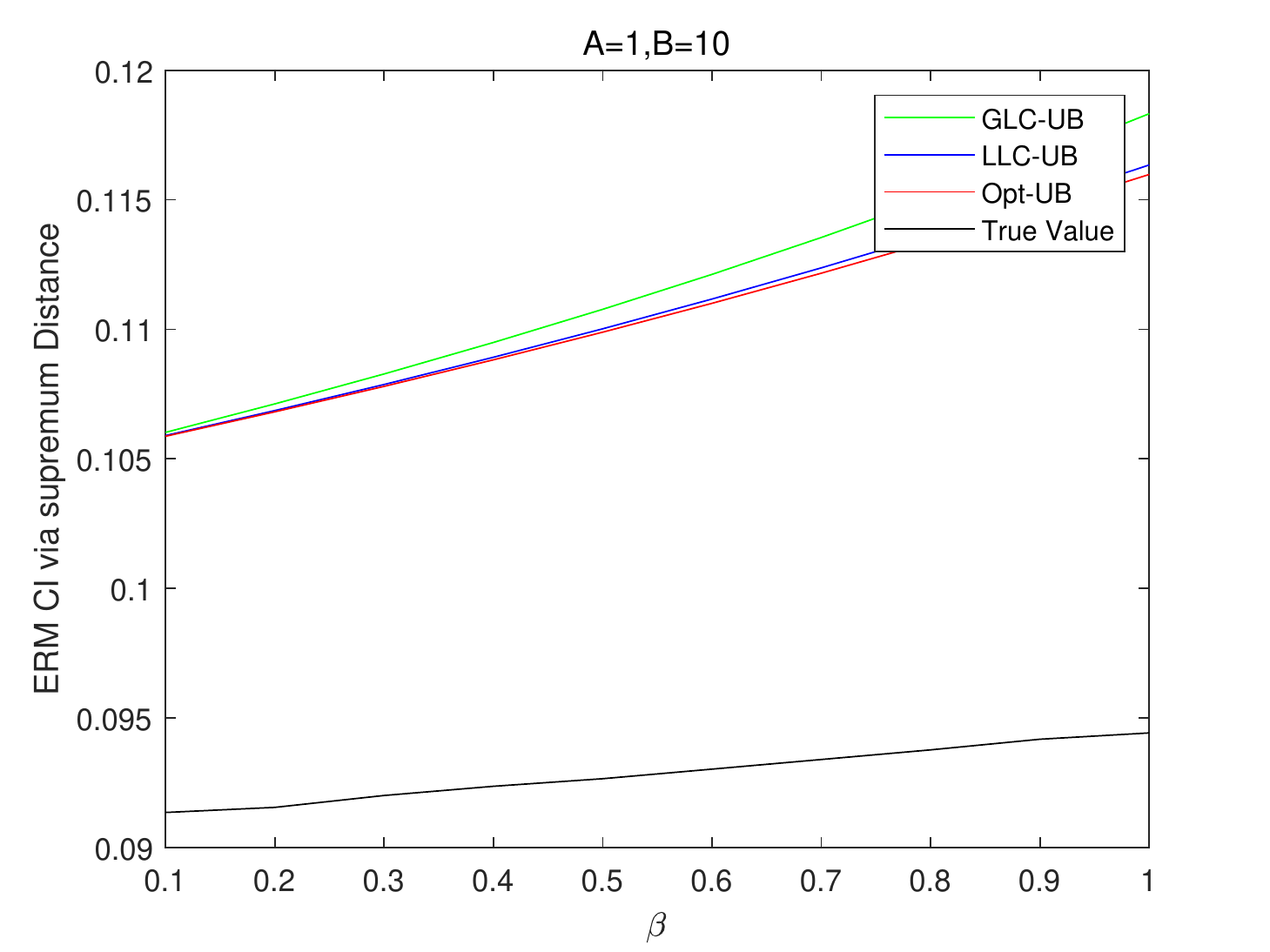}
     \end{subfigure}
     \hfill
     \begin{subfigure}[b]{0.33\textwidth}
         \centering
         \includegraphics[width=\textwidth]{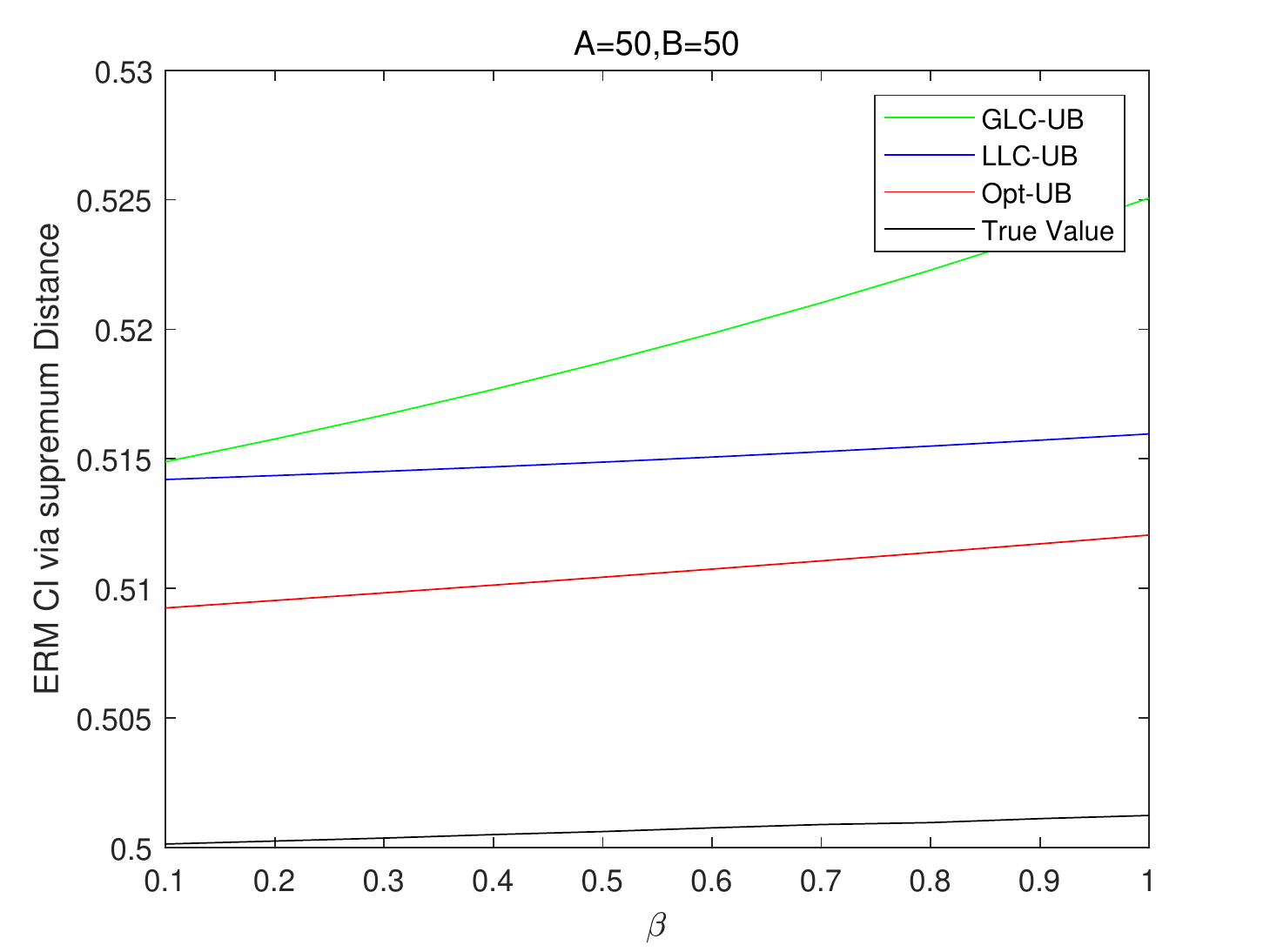}
     \end{subfigure}
     \begin{subfigure}[b]{0.33\textwidth}
         \centering
         \includegraphics[width=\textwidth]{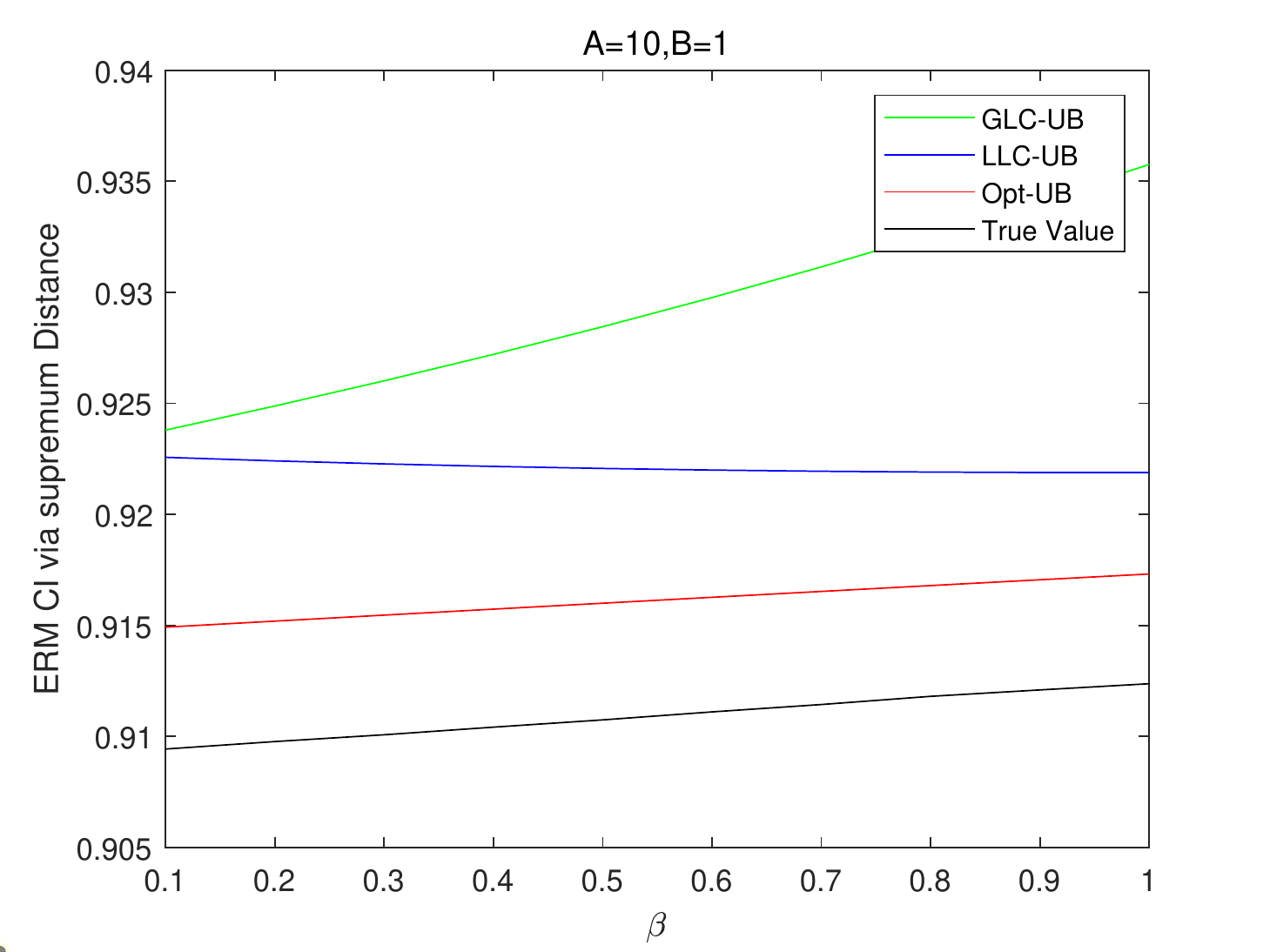}
     \end{subfigure}
     \begin{subfigure}[b]{0.33\textwidth}
         \centering
         \includegraphics[width=\textwidth]{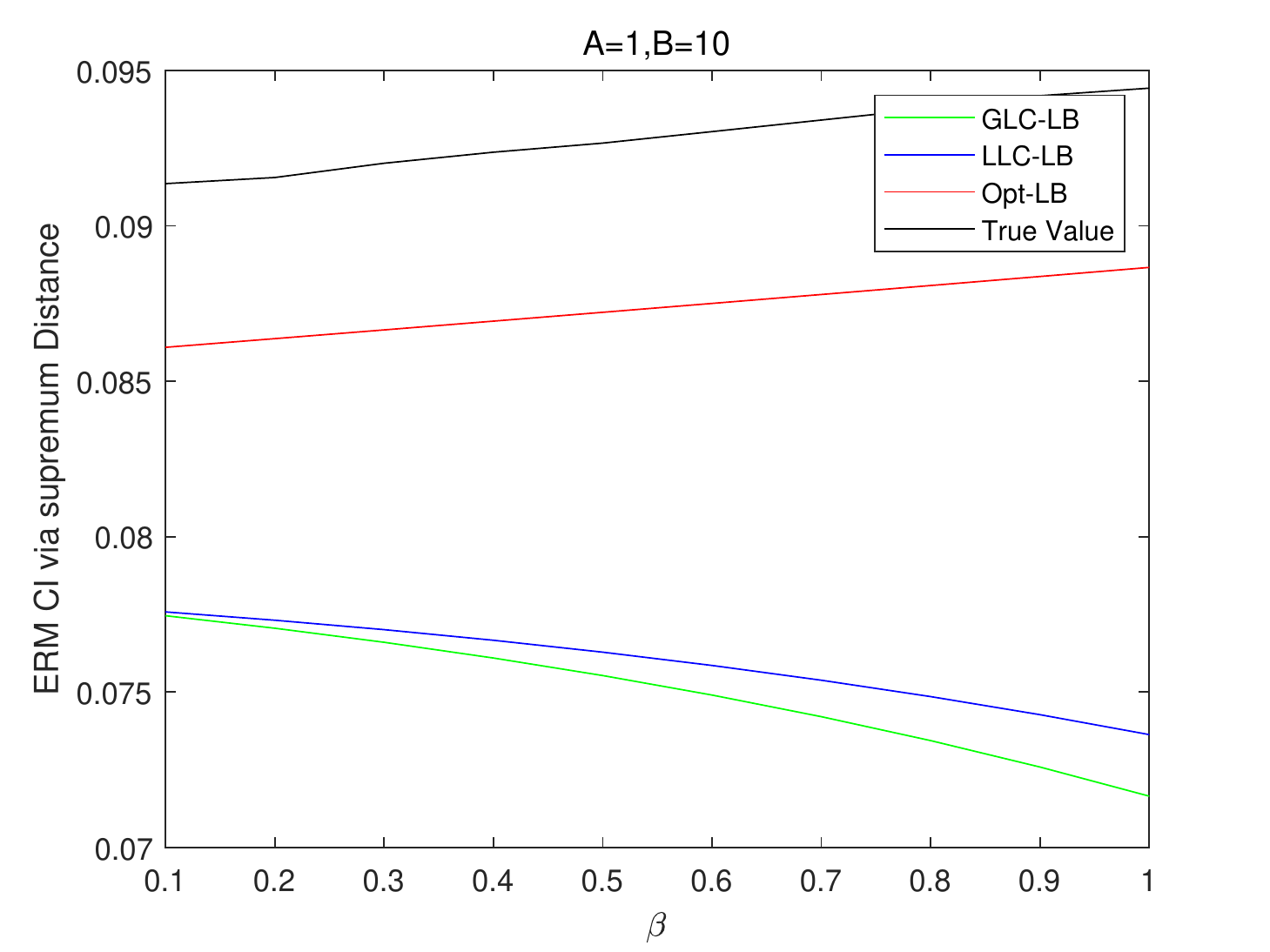}
     \end{subfigure}
     \hfill
     \begin{subfigure}[b]{0.33\textwidth}
         \centering
         \includegraphics[width=\textwidth]{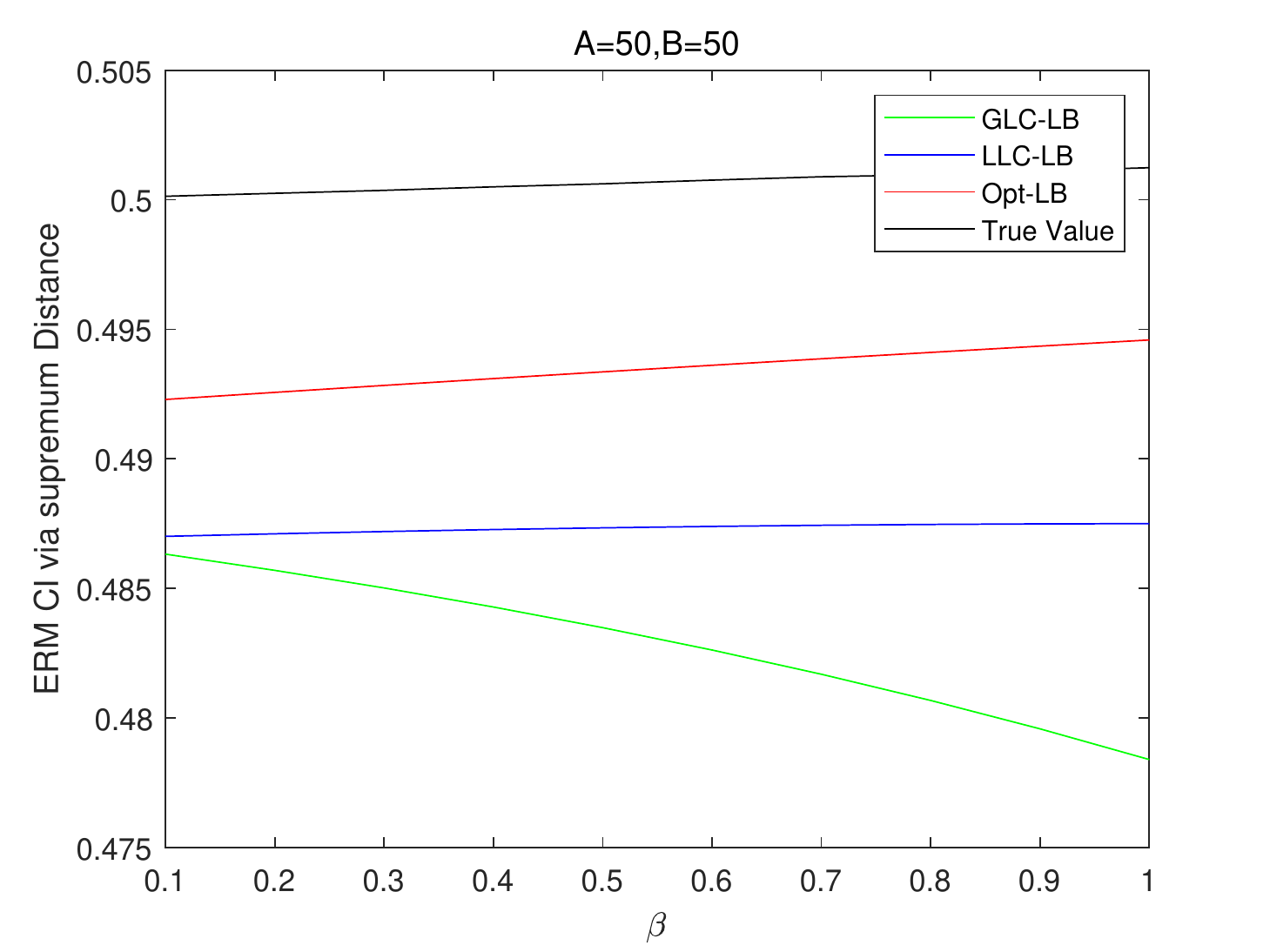}
     \end{subfigure}
     \begin{subfigure}[b]{0.33\textwidth}
         \centering
         \includegraphics[width=\textwidth]{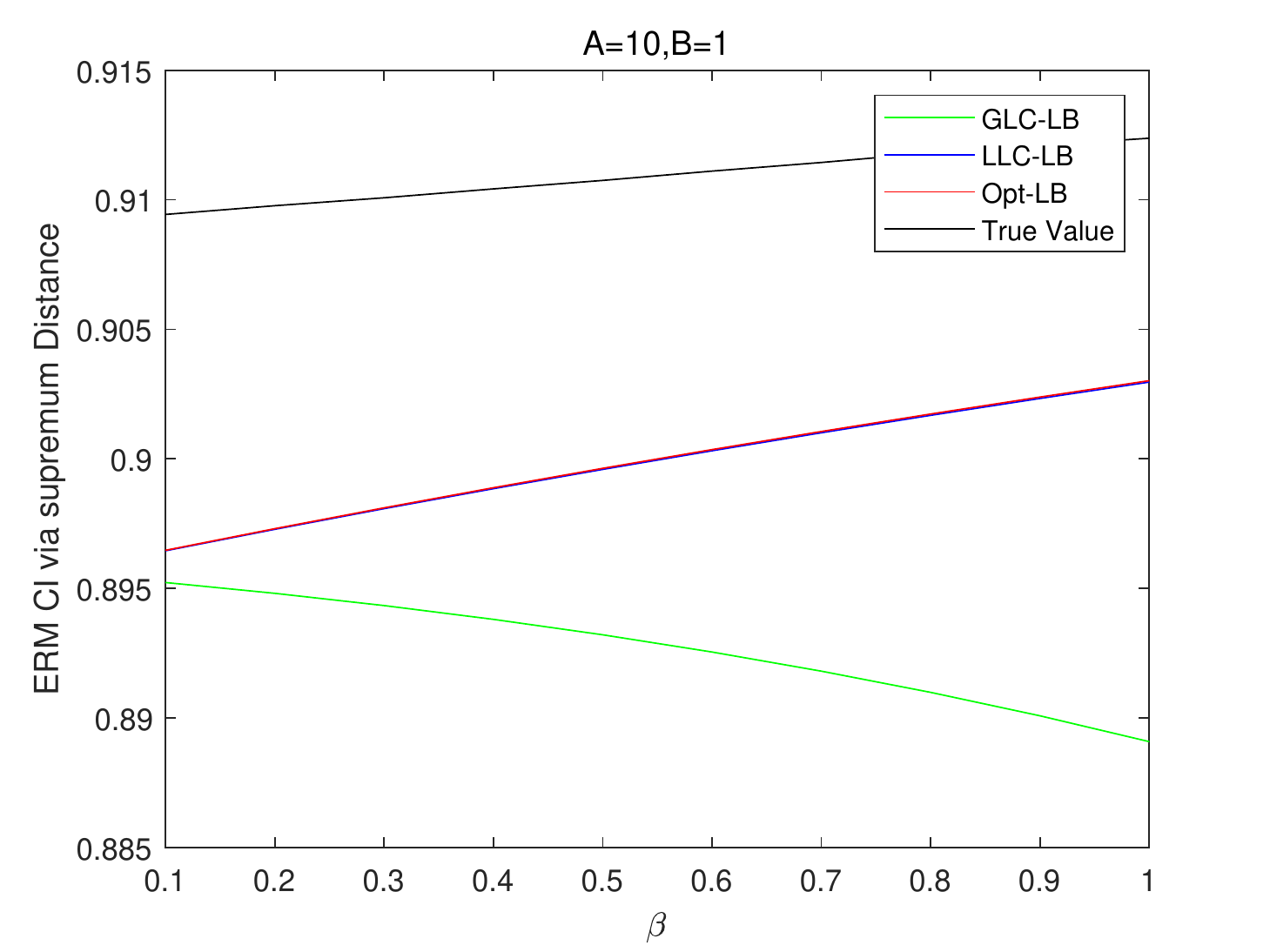}
     \end{subfigure}
        \caption{ERM CI with varying $\beta$}
        \label{fig:erm_be_sup}
        \vspace{-0ex}
\end{figure*}

\begin{figure*}[ht]
     \centering
     \begin{subfigure}[b]{0.33\textwidth}
         \centering
         \includegraphics[width=\textwidth]{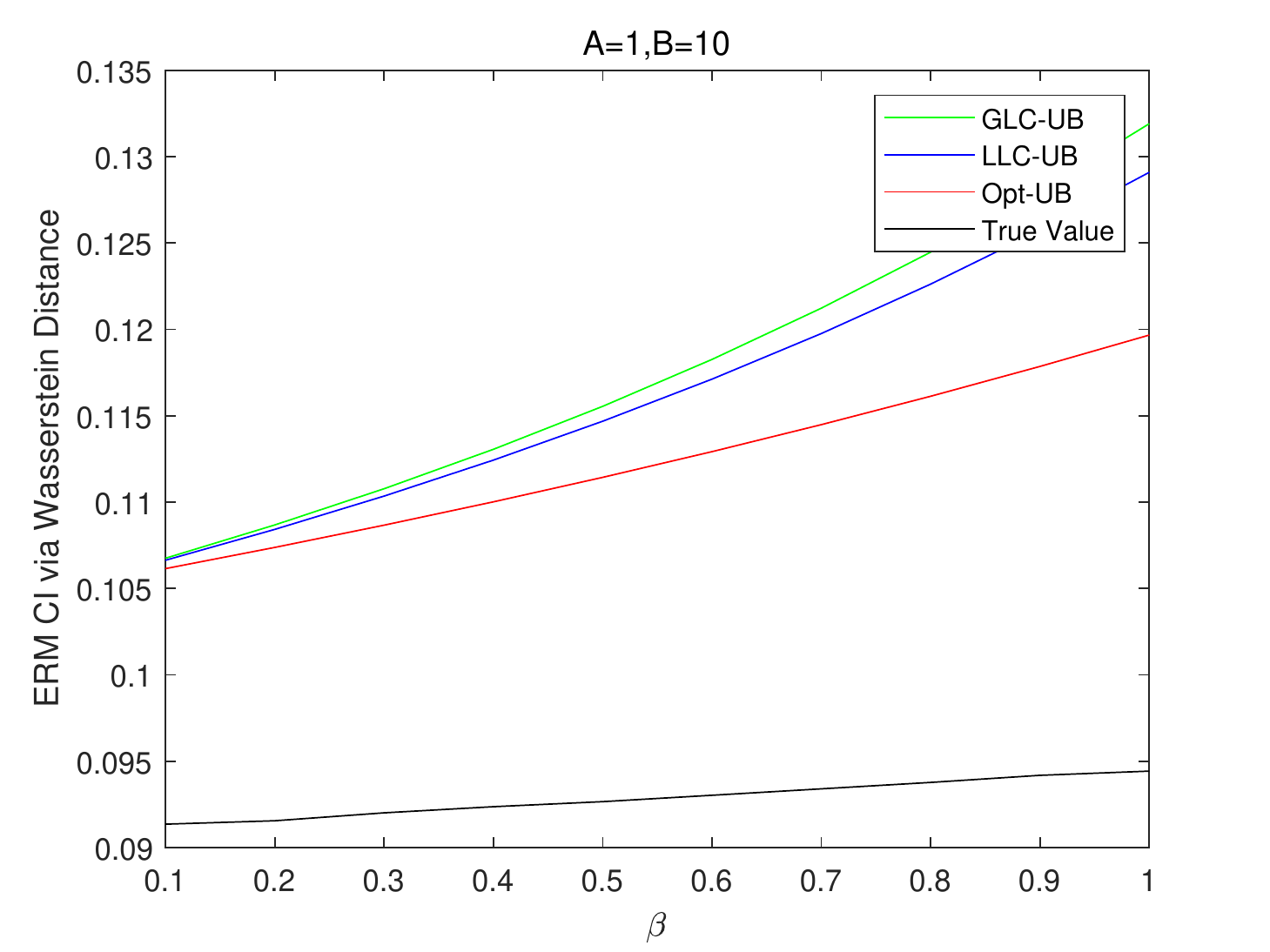}
     \end{subfigure}
     \hfill
     \begin{subfigure}[b]{0.33\textwidth}
         \centering
         \includegraphics[width=\textwidth]{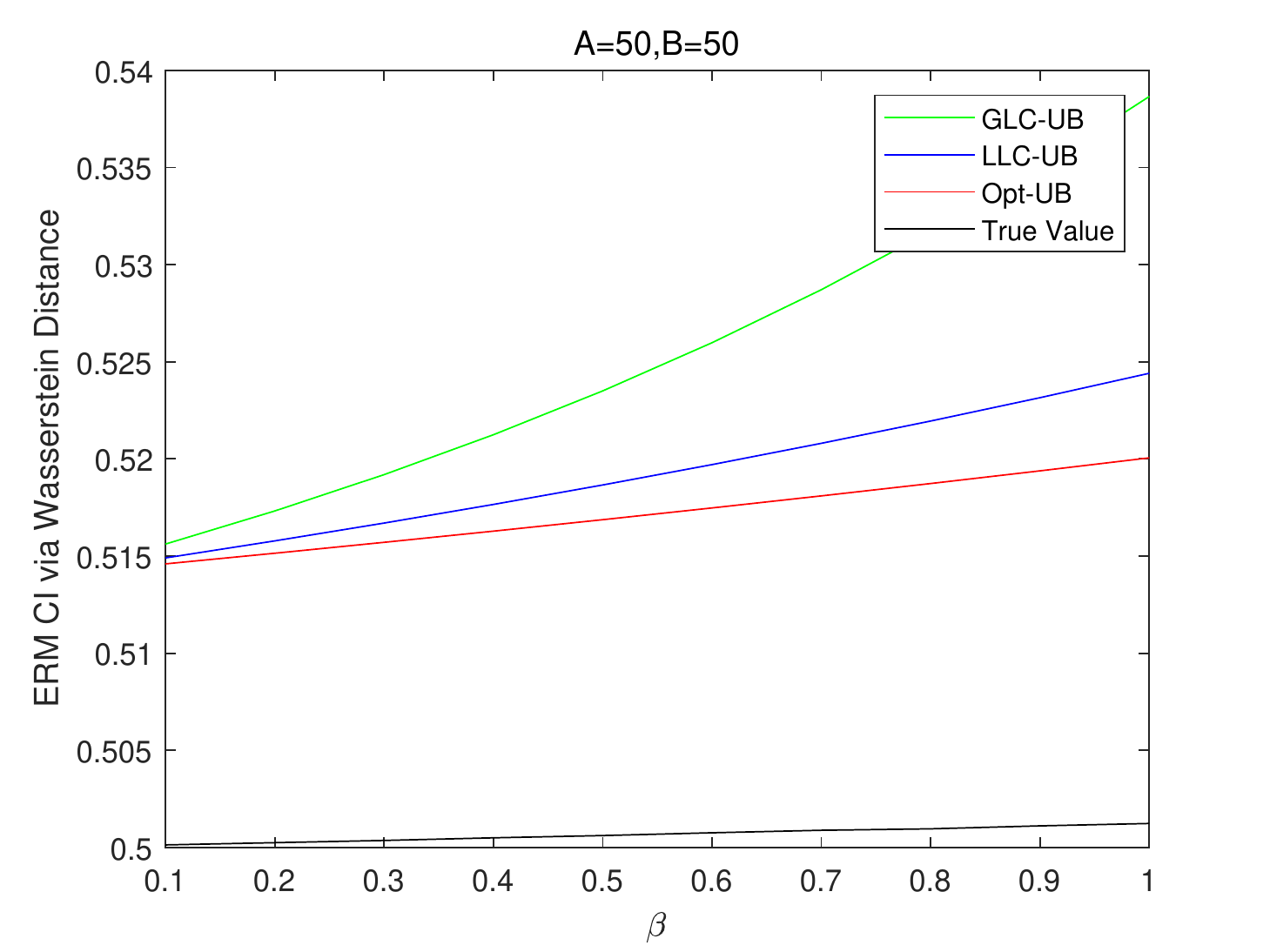}
     \end{subfigure}
     \begin{subfigure}[b]{0.33\textwidth}
         \centering
         \includegraphics[width=\textwidth]{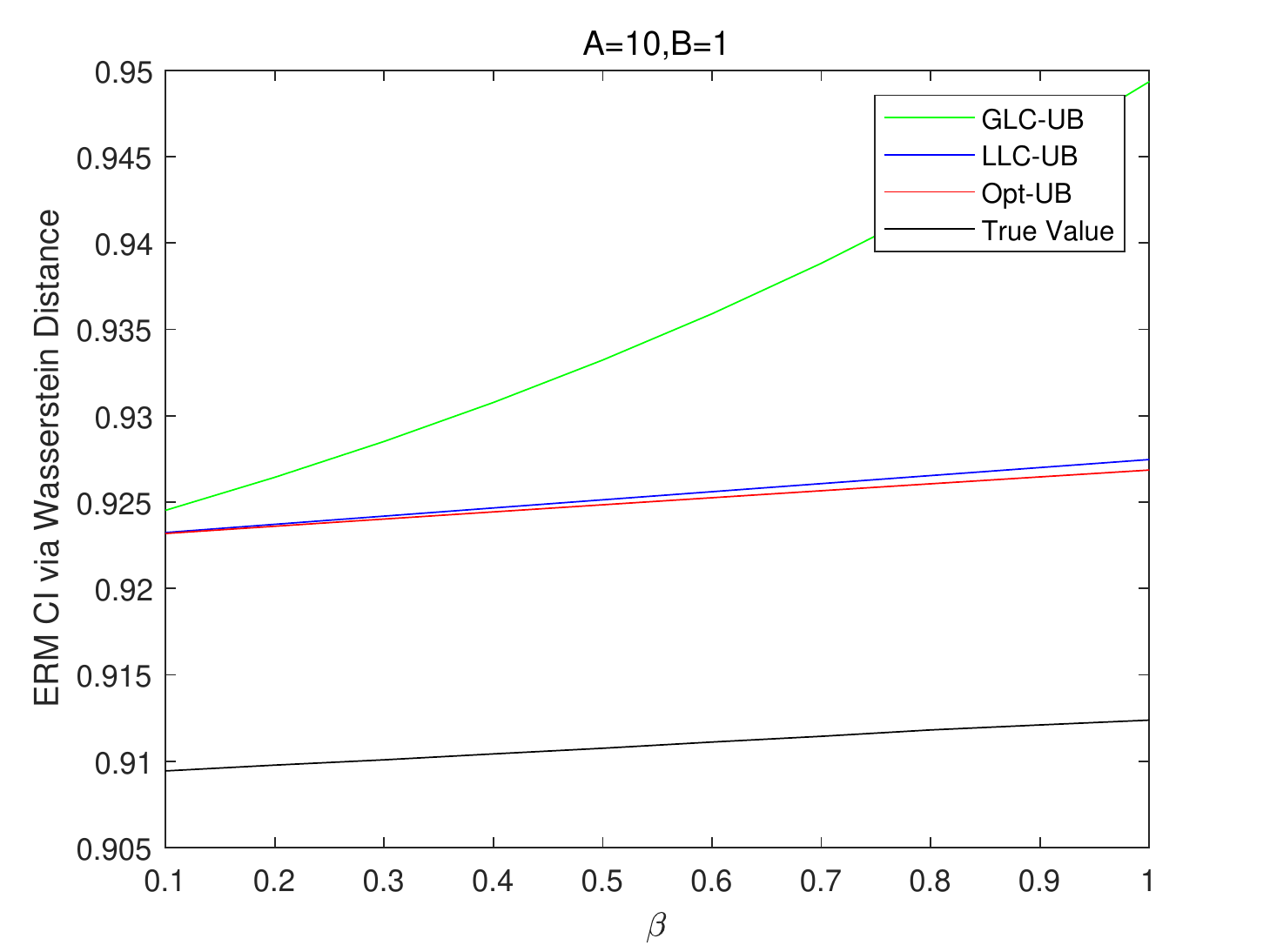}
     \end{subfigure}
     \begin{subfigure}[b]{0.33\textwidth}
         \centering
         \includegraphics[width=\textwidth]{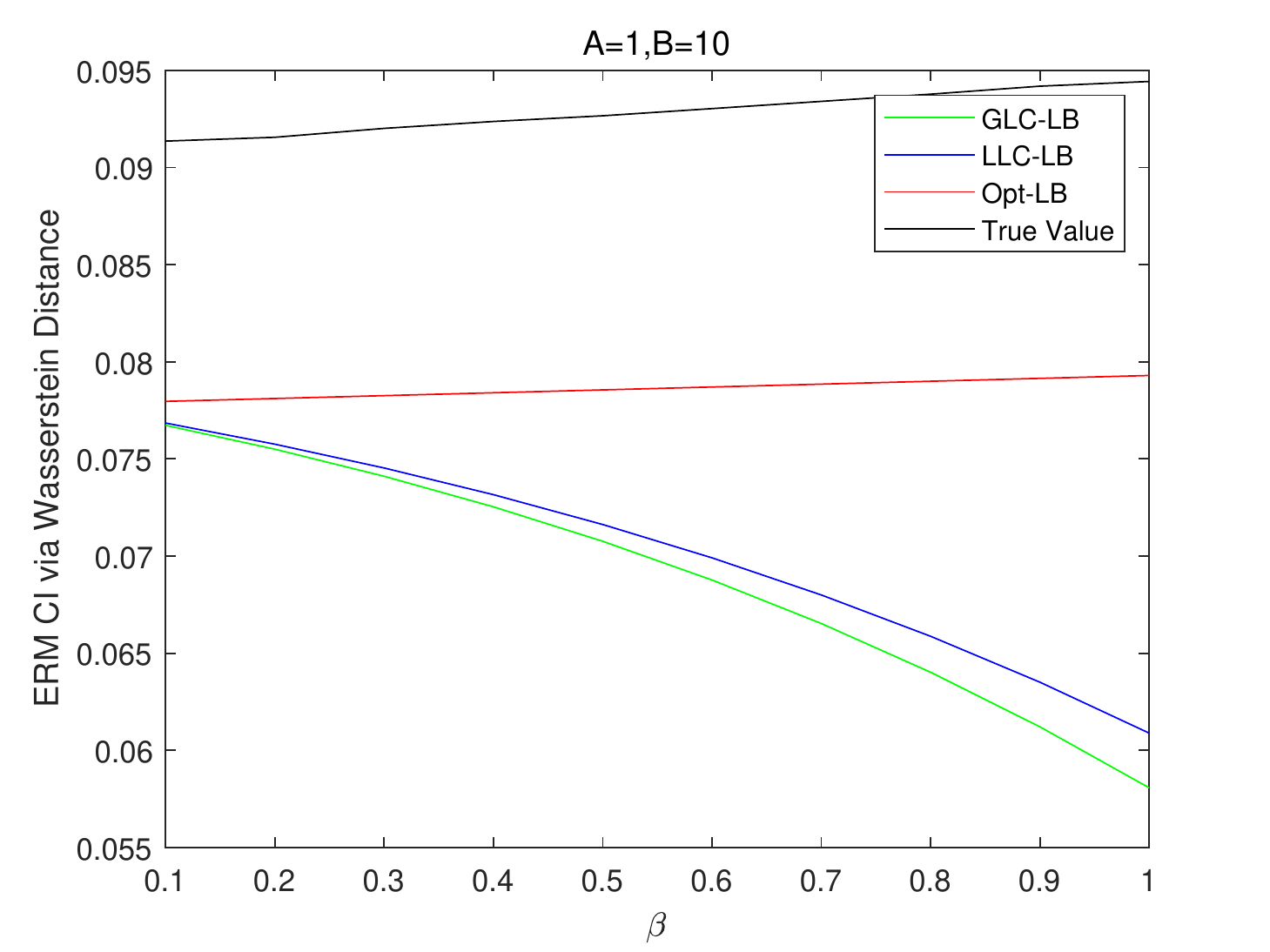}
     \end{subfigure}
     \hfill
     \begin{subfigure}[b]{0.33\textwidth}
         \centering
         \includegraphics[width=\textwidth]{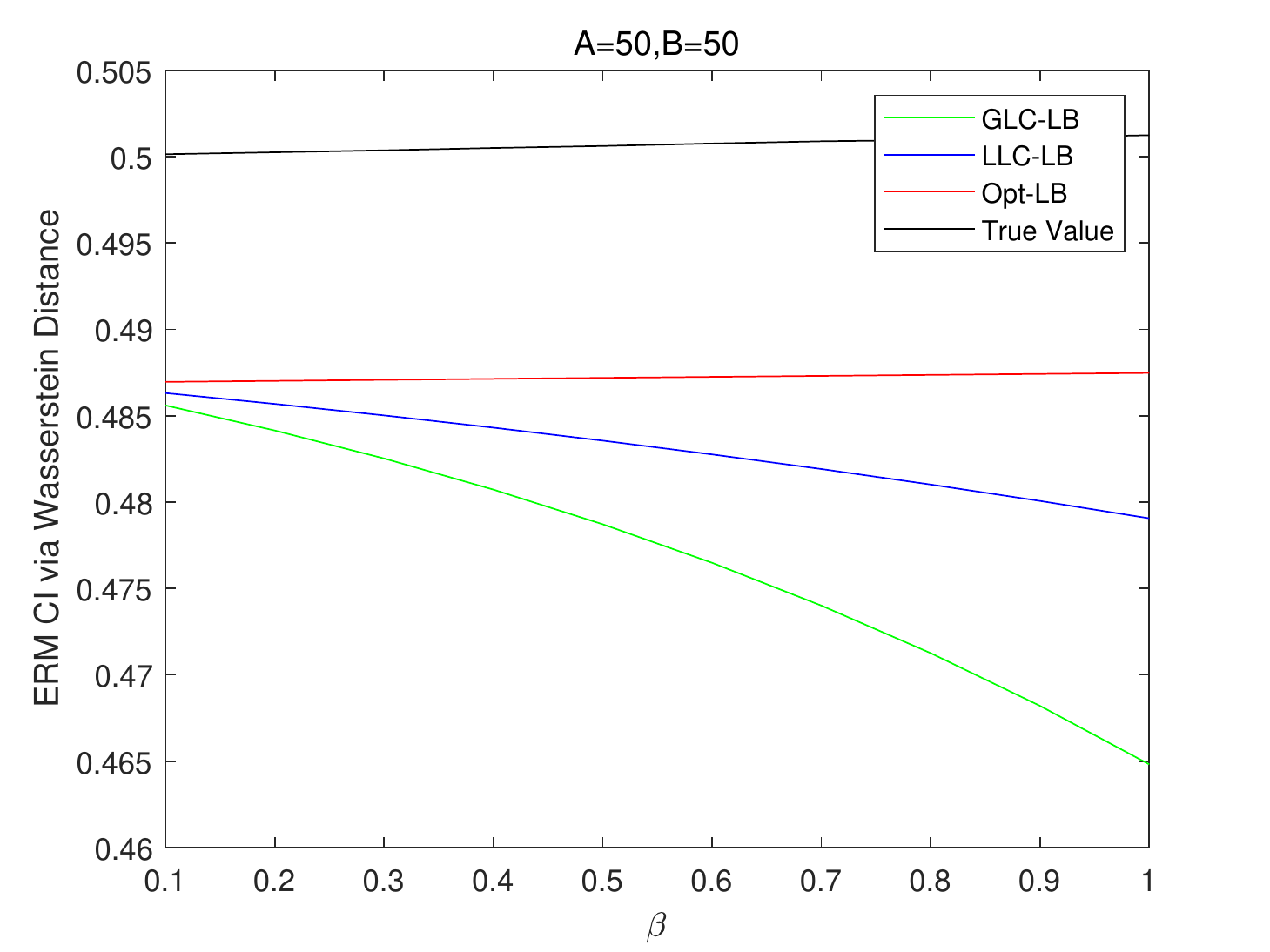}
     \end{subfigure}
     \begin{subfigure}[b]{0.33\textwidth}
         \centering
         \includegraphics[width=\textwidth]{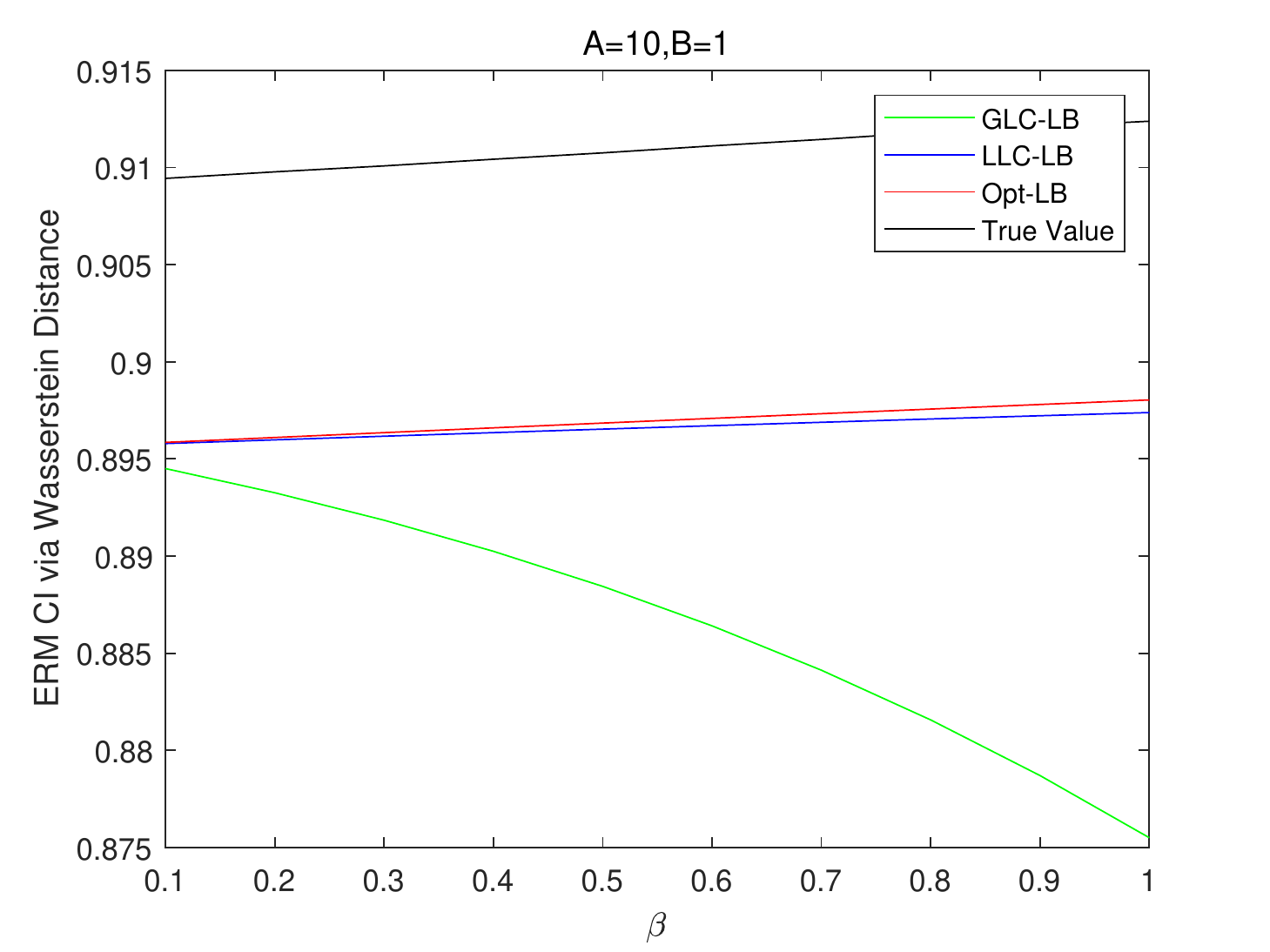}
     \end{subfigure}
        \caption{ERM CI with varying $\beta$}
        \label{fig:erm_be_was}
        \vspace{-0ex}
\end{figure*}
\subsection{CVaR Bandit}
We adopt the same bandit instances as \citet{tamkin2019distributionally}. The parameters of these distributions are given in Table 1 in \citet{tamkin2019distributionally}. The left, middle, and right part of Figure 14 plots the results for easy bandit instance with $\alpha=0.25$, hard bandit instance with $\alpha=0.25$, and hard bandit instance with $\alpha=0.05$. As expected, \texttt{CVaR-UCB} consistently outperforms \texttt{LLC-UCB} and \texttt{GLC-UCB}.
\begin{figure*}[ht]
     \centering
     \begin{subfigure}[b]{0.33\textwidth}
         \centering
         \includegraphics[width=\textwidth]{bandit_easy_0.25.pdf}
     \end{subfigure}
     \hfill
     \begin{subfigure}[b]{0.33\textwidth}
         \centering
         \includegraphics[width=\textwidth]{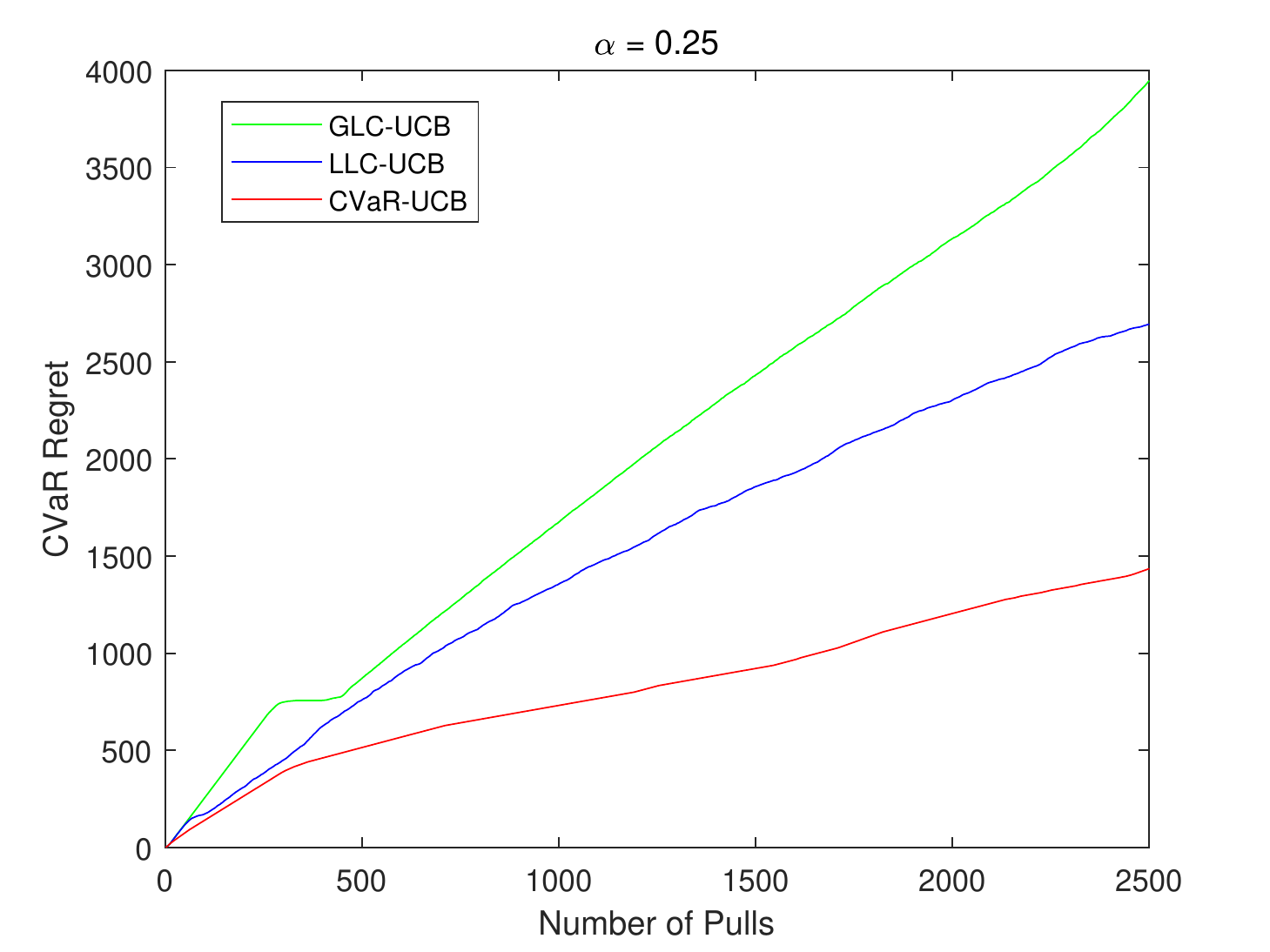}
     \end{subfigure}
     \begin{subfigure}[b]{0.33\textwidth}
         \centering
         \includegraphics[width=\textwidth]{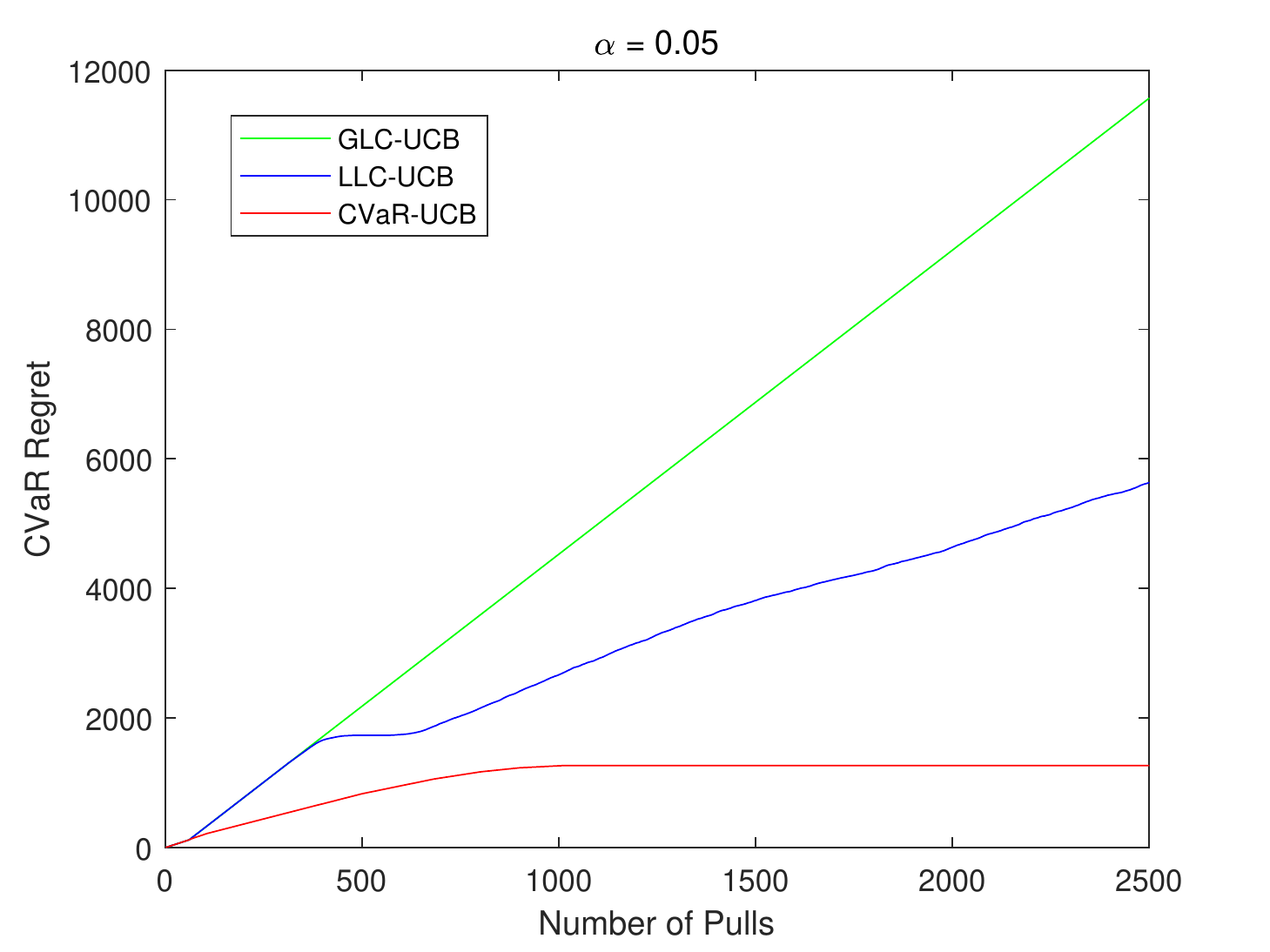}
     \end{subfigure}
        \caption{CVaR bandit}
        \label{fig:cvar_bandit}
        \vspace{-0ex}
\end{figure*}

\end{document}